\documentclass{article} 
\usepackage{colm2024_conference}

\usepackage{microtype}
\usepackage[hidelinks]{hyperref}
\usepackage{url}
\usepackage{booktabs}
\usepackage{graphicx}
\usepackage{cancel}
\usepackage{xcolor} 
\usepackage{colortbl} 
\usepackage{amsmath}
\usepackage{algorithm}
\usepackage[most,skins,theorems]{tcolorbox}
\usepackage{amssymb}
\usepackage{mathtools}
\usepackage{amsthm}
\usepackage{enumitem}
\usepackage{pifont}
\usepackage{listings}
\usepackage{fontawesome}

\usepackage{tcolorbox}
\tcbuselibrary{breakable, skins}
\usepackage{wrapfig}
\usepackage{caption}
\usepackage{mdframed}
\usepackage{tikz}
\usepackage{pgfplots}
\usepackage{subcaption}
\usepackage{booktabs}
\pgfplotsset{compat=1.18}
\usetikzlibrary{arrows.meta,positioning}
\usepackage{amsthm} 

\newtheorem{theorem}{Theorem}[section]
\newtheorem{lemma}[theorem]{Lemma}
\newtheorem{corollary}[theorem]{Corollary}

\newtheorem{proposition}[theorem]{Proposition}
\theoremstyle{definition}
\newtheorem{definition}[theorem]{Definition}

\theoremstyle{remark}
\newtheorem{remark}[theorem]{Remark}
\usepackage{tikz}

\newcommand{\lstbg}[3][0pt]{{\fboxsep#1\colorbox{#2}{\strut #3}}}

\lstdefinelanguage{diff}{
  basicstyle=\ttfamily\small,
  morecomment=[f][\lstbg{red!20}]-,
  morecomment=[f][\lstbg{green!20}]+,
}
\lstdefinelanguage{diffpython}{
  language=diff,
  morekeywords={def, if, else, for, while, return, import, from, as, class, with, try, except, finally, raise, lambda, and, or, not, in, is, None, True, False},
  morecomment=[l]{\#},
  morestring=[b]",
  morestring=[b]',
}

\usepackage[most,skins,theorems]{tcolorbox}
\tcbset{
  aibox/.style={
    width=\linewidth,
    top=7pt,
    bottom=2pt,
    colback=cogni_magenta!5!white,
    colframe=black,
    colbacktitle=black,
    enhanced,
    center,
    attach boxed title to top left={yshift=-0.1in,xshift=0.15in},
    boxed title style={boxrule=0pt,colframe=white,},
  }
}
\newtcolorbox{AIbox}[2][]{aibox,title={#2},#1}

\usepackage[most,skins,theorems]{tcolorbox}
\tcbset{
  cognibox/.style={
    width=\linewidth,
    top=7pt,
    bottom=2pt,
    colback=cogni_orange!5!white,
    colframe=black,
    colbacktitle=black,
    coltitle=white,
    enhanced,
    center,
    attach boxed title to top left={yshift=-0.1in,xshift=0.15in},
    boxed title style={boxrule=0pt,colframe=white,},
  }
}
\newtcolorbox{Cognibox}[2][]{cognibox,title=#2,#1}

\usepackage[most]{tcolorbox}

\definecolor{takeawayFrame}{HTML}{2D6B67} 
\definecolor{takeawayTitle}{HTML}{1E5A57} 
\definecolor{takeawayBack}{HTML}{F7FBFA}  

\tcbset{
  takeaway/.style={
    enhanced,
    breakable,
    colback=takeawayBack,
    colframe=takeawayFrame,
    boxrule=0.2pt,
    arc=5pt,
    left=14pt, right=14pt, top=10pt, bottom=10pt,
    coltitle=white,
    colbacktitle=takeawayTitle,
    attach boxed title to top left={xshift=10pt,yshift=-3mm},
    boxed title style={
      enhanced,
      boxrule=0pt,       
      arc=5pt,
      left=6pt, right=6pt,
      top=1pt, bottom=1pt,
      fontupper=\scriptsize\bfseries\scshape, 
    },
  }
}

\newtcolorbox{Takeaway}[1][]{takeaway,title={TAKEAWAY},#1}

\newtcolorbox{TakeawayBox}[2][]{%
  takeaway,
  title={\fontsize{9}{9}\selectfont\bfseries\scshape #2}, 
  #1
}

\tcbset{
  takeawayplain/.style={
    enhanced,
    breakable,
    colback=takeawayBack,
    colframe=takeawayFrame,
    boxrule=0.2pt,
    arc=5pt,
    left=14pt, right=14pt, top=10pt, bottom=10pt,
  }
}

\newtcolorbox{TakeawayPlain}[1][]{takeawayplain,#1}

\ExplSyntaxOn
\NewDocumentCommand{\DropCapTitle}{m}
 {
  \tl_set:Nx \l_tmpa_tl { \tl_trim_spaces:n {#1} }

  \tl_set:Nx \l_tmpb_tl { \tl_head:N \l_tmpa_tl } 
  \tl_set:Nx \l_tmpc_tl { \tl_tail:N \l_tmpa_tl } 

  {\fontsize{11}{11}\selectfont\bfseries\scshape \l_tmpb_tl}%
  {\fontsize{10}{10}\selectfont\mdseries\scshape \l_tmpc_tl}%
 }
\ExplSyntaxOff

\newtcolorbox{TakeawayDrop}[2][]{%
  takeaway,
  title={\DropCapTitle{#2}},
  #1
}

\usepackage{yfonts} 

\newcommand{\TitleInitial}[1]{{\fontsize{12}{12}\selectfont \textswab{#1}}}

\newcommand{\TitleRest}[1]{{\fontsize{9}{10}\selectfont \mdseries\scshape #1}}

\ExplSyntaxOn
\NewDocumentCommand{\DropCapTitleOldBook}{m}
 {
  \tl_set:Nx \l_tmpa_tl { \tl_trim_spaces:n {#1} }
  \tl_set:Nx \l_tmpb_tl { \tl_head:N \l_tmpa_tl } 
  \tl_set:Nx \l_tmpc_tl { \tl_tail:N \l_tmpa_tl } 
  \TitleInitial{\l_tmpb_tl}\kern0.5pt\TitleRest{\l_tmpc_tl}
 }
\ExplSyntaxOff

\newtcolorbox{TakeawayDropGothic}[2][]{%
  takeaway,
  title={\DropCapTitleOldBook{#2}},
  #1
}

\usepackage[most]{tcolorbox}

\definecolor{cardBack}{HTML}{F4F5F7}
\definecolor{cardFrame}{HTML}{E6E8EC}

\tcbset{
  papercard/.style={
    enhanced,
    breakable,
    colback=cardBack,
    colframe=cardFrame,   
    boxrule=0.1pt,
    arc=8pt,
    left=18pt,right=18pt,top=4pt,bottom=2pt,
    drop shadow={black!7},
  }
}

\newtcolorbox{PaperCard}[1][]{papercard,#1}

\usepackage[most]{tcolorbox}

\definecolor{thmFrame}{HTML}{CCCCFB}   
\definecolor{thmTitle}{HTML}{CCCCFB}   

\tcbset{
  thmbox/.style={
    enhanced,
    breakable,
    colback=white,
    colframe=thmFrame,
    boxrule=0.9pt,
    sharp corners,
    left=12pt, right=12pt, top=10pt, bottom=10pt,
  }
}

\newtcolorbox{ThmBox}[2][]{thmbox,title={#2},#1}

\newtcolorbox{ThmBoxCapital}[2][]{thmbox,title={\DropCapTitle{#2}},
  #1}

\newtcolorbox{ThmBoxGothic}[2][]{thmbox,title={\DropCapTitleOldBook{#2}},
  #1}

\usepackage{xcolor}
\usepackage{tabularray}
\UseTblrLibrary{siunitx} 

\definecolor{Navy}{HTML}{0F172A}
\definecolor{Slate}{HTML}{64748B}
\definecolor{Emerald}{HTML}{10B981} 
\definecolor{Rose}{HTML}{F43F5E}    
\definecolor{AxisColor}{HTML}{CBD5E1} 

\usepackage[most,skins,theorems]{tcolorbox} 

\definecolor{emerald}{HTML}{10B981}   
\definecolor{slate900}{HTML}{0F172A}  
\definecolor{cogni_orange}{HTML}{FF5101}  
\definecolor{cogni_magenta}{HTML}{FF2165}  
\definecolor{brandA}{RGB}{31,41,55}   

\definecolor{brandB}{HSB}{180,0.6,0.8}

\tcbset{
  observation/.style={
    width=\linewidth,
    top=7pt,
    bottom=2pt,
    colback=emerald!5!white,  
    colframe=slate900,
    colbacktitle=slate900,
    coltitle=white,
    enhanced,
    center,
    attach boxed title to top left={yshift=-0.1in,xshift=0.15in},
    boxed title style={boxrule=0pt,colframe=white},
  }
}

\newtcolorbox{Observation}[2][]{observation,title=#2,#1}

\definecolor{rliableolive}{HTML}{BBCC33}
\definecolor{rliableblue}{HTML}{77AADD}
\definecolor{rliablered}{HTML}{EE8866}
\definecolor{SDEblue}{RGB}{28 58 88}
\definecolor{cc1}{rgb}{1.0, 0.44, 0.37}
\definecolor{cc2}{rgb}{0.0, 0.2, 0.6}
\definecolor{cc3}{RGB}{255, 191, 0}
\definecolor{cc4}{RGB}{0, 128, 128}

\theoremstyle{definition}

\usepackage{cleveref}
\crefname{section}{Sec.}{Sec.}
\crefname{theorem}{Theorem}{Theorems}
\crefname{corollary}{Corollary}{Corollaries}
\crefname{lemma}{Lemma}{Lemmas}
\crefname{equation}{Eq.}{Eq.}
\crefname{proposition}{Proposition}{Propositions}
\crefname{claim}{Claim}{Claims}
\crefname{remark}{Remark}{Remarks}
\crefname{observation}{Observation}{Observations}
\crefname{assumption}{Assumption}{Assumptions}
\crefname{template}{Template}{Template}
\crefname{definition}{Definition}{Definitions}
\crefname{appendix}{App.}{Apps.}
\crefname{algorithm}{Algorithm}{Algorithms}
\crefname{figure}{Fig.}{Fig.}
\crefname{table}{Table}{Tables}
\crefname{property}{Property}{Properties}
\crefname{line}{Line}{Lines}

\definecolor{table-blue}{RGB}{173, 216, 230}
\definecolor{darkblue}{rgb}{0, 0, 0.5}
\hypersetup{colorlinks=true, citecolor=darkblue, linkcolor=darkblue, urlcolor=magenta}

\newcommand{\TPR}{\mathrm{TPR}}
\newcommand{\FPR}{\mathrm{FPR}}
\newcommand{\Var}{\mathrm{Var}}
\newcommand{\p}{\mathbf{p}}
\newcommand{\y}{y}
\newcommand{\q}{\mathbf{q}}
\newcommand{\A}{\mathbf{A}}

\usepackage{amsmath,amssymb,amsthm,mathtools}
\usepackage{bm}
\usepackage{enumitem}
\usepackage[hidelinks]{hyperref}
\usepackage[margin=1in]{geometry}

\newcommand{\E}{\mathbb{E}}
\newcommand{\J}{\mathfrak{J}}

\newcommand{\unif}{u}

\newcommand{\vv}{v}
\newcommand{\ww}{w}

\DeclareMathOperator{\logit}{logit}

\usepackage{amsmath,amssymb}
\usepackage{float} 
\usepackage{algorithm}
\usepackage[noend]{algpseudocode}

\algrenewcommand\algorithmiccomment[1]{\hfill$\triangleright$~#1}
\algrenewcommand\algorithmicrequire{\textbf{Inputs:}}
\algrenewcommand\algorithmicensure{\textbf{Output:}}

\usepackage{amsmath,amssymb,amsthm}

\theoremstyle{remark}

\usepackage[table,dvipsnames]{xcolor} 

\definecolor{table-blue}{HTML}{2F6FEF}   
\definecolor{accent-violet}{HTML}{B14FFF}
\colorlet{delta-pos}{green!55!black}
\colorlet{delta-neg}{red!65!black}
\usepackage{booktabs}
\usepackage{siunitx}
\usepackage{pgfplots,pgfplotstable}
\pgfplotsset{compat=1.18}
\usepackage{booktabs}
\usepackage[table]{xcolor}
\usepackage{siunitx}
\usepackage{xfp} 

\definecolor{accbase}{RGB}{65,105,225} 
\definecolor{good}{RGB}{46,160,67}     
\definecolor{bad}{RGB}{220,53,69}      

\usepackage{amsmath,amsfonts,bm}

\def\eqref#1{equation~\ref{#1}}

\def\1{\bm{1}}

\def\vv{{\bm{v}}}

\DeclareMathAlphabet{\mathsfit}{\encodingdefault}{\sfdefault}{m}{sl}
\SetMathAlphabet{\mathsfit}{bold}{\encodingdefault}{\sfdefault}{bx}{n}

\definecolor{cogni_orange}{HTML}{FF5101}
\definecolor{cogni_magenta}{HTML}{FF2165}

\colorlet{abstractTop}{cogni_orange!4!white}
\colorlet{abstractBottom}{cogni_magenta!1!white}

\newtcolorbox{abstractshade}{
  enhanced,
  boxrule=0pt,
  arc=10pt,
  left=5pt,right=5pt,top=0pt,bottom=0pt,
  before skip=8pt, after skip=10pt,
  interior style={top color=abstractTop, bottom color=abstractBottom},
}

\makeatletter
\let\origabstract\abstract
\let\endorigabstract\endabstract
\renewenvironment{abstract}{%
  \begin{abstractshade}%
  \origabstract%
}{%
  \endorigabstract%
  \end{abstractshade}%
}
\makeatother

\DeclareRobustCommand{\CognichipAffil}{%
  \raisebox{-0.15\height}{\includegraphics[height=2.5ex]{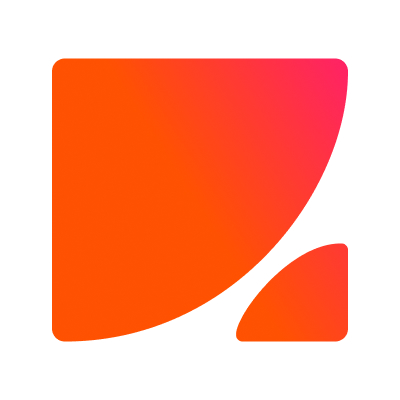}}%
  \hspace{0.4em}Cognichip AI%
}

\title{Rate or Fate? \newline RLV$^\varepsilon$R: Reinforcement Learning with Verifiable Noisy Rewards}

\author{%
\parbox{\textwidth}{\centering
\textbf{Ali Rad$^{1}$\footnotemark[1]\quad
Khashayar Filom$^{1}$\quad
Darioush Keivan$^{1}$}\\
\textbf{Peyman Mohajerin Esfahani$^{2}$\quad
Ehsan Kamalinejad$^{1}$}\\[0.6em]
$^{1}$\CognichipAffil\quad
$^{2}$University of Toronto
}%
}

\colmfinalcopy 

\begin{document}
\maketitle
\footnotetext[1]{Corresponding author: \texttt{ali@cognichip.ai}}
\vspace{-0.7cm}
\vspace{-0.45cm}
\begin{center}
\textbf{\faGithub~\url{https://github.com/cognichip/Noisy-RL}}
\end{center}
\vspace{-0.2cm}
\begin{center}
     \today
\end{center}
\vspace{-0.3cm}

\begin{abstract}
Reinforcement learning with verifiable rewards (RLVR) is a simple but powerful paradigm for training LLMs: sample a completion, verify it, and update. In practice, however, \textbf{the verifier is almost never clean}--unit tests probe only limited corner cases; human and synthetic labels are imperfect; and LLM judges (e.g., RLAIF) are noisy and can be exploited--and this problem worsens on harder domains (especially coding) where tests are sparse and increasingly model-generated. We ask a pragmatic question: does the verification noise merely slow down the learning (\emph{rate}), or can it flip the outcome (\emph{fate})?

To address this, we develop an analytically tractable \textbf{multi-armed bandit view of RLVR dynamics}, instantiated with GRPO and validated in controlled experiments. Modeling \textbf{false positives} and \textbf{false negatives} and grouping completions into recurring reasoning modes yields a replicator-style (\textbf{natural-selection}) flow on the probability simplex. The dynamics decouples into within-correct-mode competition and a one-dimensional evolution for the mass on incorrect modes, whose drift is determined solely by Youden's index $J=\mathrm{TPR}-\mathrm{FPR}$. This yields a sharp \textbf{phase transition}: when $J>0$, the incorrect mass is driven toward extinction (learning); when $J=0$, the process is neutral; and when $J<0$, incorrect modes amplify until they dominate (anti-learning and collapse). In the learning regime $J>0$, noise primarily rescales convergence time (``rate, not fate''). Experiments on verifiable programming tasks under synthetic noise reproduce the predicted $J=0$ boundary. Beyond noise, the framework offers a general lens for analyzing RLVR stability, convergence, and algorithmic interventions.
\end{abstract}





\vspace{-0.6cm}
\begin{figure}[h]
    \centering
    \includegraphics[width=1\linewidth]{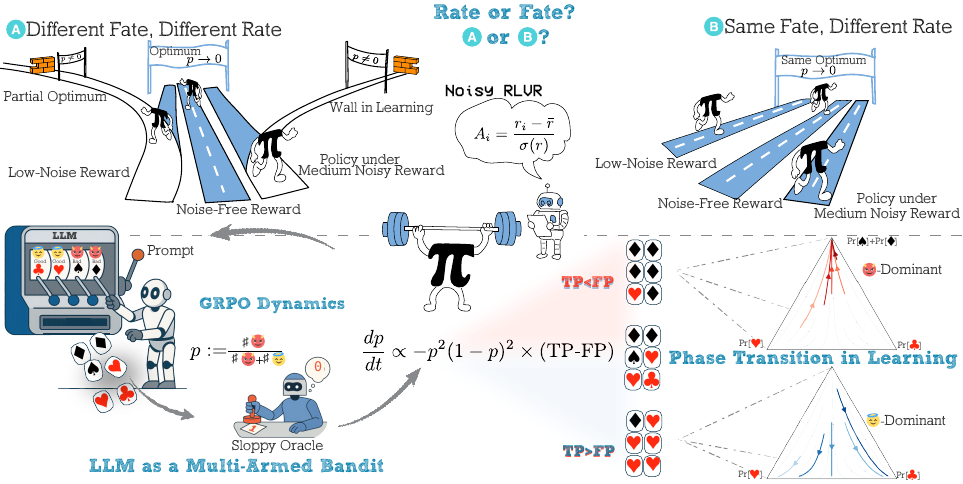}
    \label{fig:fig1}
\end{figure}

\newpage

\vspace{-0.2cm}

\vspace{-0.1cm}

\begin{figure}
    \centering
    \includegraphics[width=1.0\linewidth]{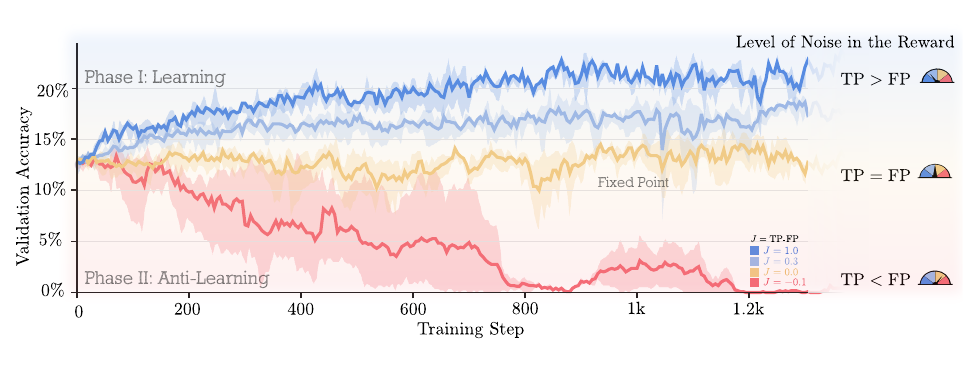}
\caption{\textbf{Phase transition under a sloppy oracle.}
We vary a single noise knob $J$, defined as
$J = 1 - \delta_{\mathrm{FP}} - \delta_{\mathrm{FN}} = \mathrm{TPR} - \mathrm{FPR}$,
which summarizes the net effect of false-positive and false-negative reward corruption induced by a noisy oracle.
As $J$ crosses zero, we observe a sharp transition from \emph{Phase I} (learning; $J>0$, accuracy improves with training and the bad-solution mass $p$ for a prompt converges to $0$) to \emph{Phase II} (unlearning; $J<0$, accuracy systematically degrades and $p$ converges to $1$), with $J=0$ as the critical boundary.
Curves report \emph{validation} performance during GRPO training on a Python code-generation task (two epochs) starting from Qwen-2.5-3B; solid lines denote the run-averaged $\mathbb{E}[\mathrm{pass@1}]$ and shaded bands indicate $95\%$ confidence intervals across runs.
Although our mean-field ODE characterizes the \emph{training-time} dynamics of $p$, the same sign-of-$J$ transition is mirrored on held-out validation prompts in this setup, suggesting the drift is not merely a training-set artifact.
}
    \label{fig:phase-transition-experimental-results}
\end{figure}

\section{Introduction}\label{sec:intro}

\paragraph{Reinforcement Learning and LLMs.} Recent breakthroughs in the reasoning capabilities of large language models (LLMs) through Reinforcement Learning (RL)--particularly Reinforcement Learning with Verifiable Rewards (RLVR) and group-normalized algorithms such as  Group Relative Policy Optimization (GRPO) (\cite{rlvr,su2025rlvrbridge})--have greatly expanded the frontier of model intelligence (\cite{shao2024deepseekmath}). These methods provide further hope for the idea that true creativity and intelligence can emerge through self-play, interaction with an environment, and reward-based feedback.

\paragraph{Group-Normalized RL and RLVR.} Group-normalized  approaches in RL, like GRPO, eliminate the need for an explicit reward model (critic or PRM,  \cite{ppo,lightman2023let}) in verifiable domains such as mathematics and code generation, and demonstrate that even a few rollouts per prompt are often sufficient to approximate the advantage of each generated sequence (\cite{rlvr,su2025rlvrbridge}). However, the cornerstone of these algorithms remains the sequence-level reward. This naturally leads to several key questions:

\begin{PaperCard}
\begin{quote}
\emph{How sensitive is RL training to the quality of grand truth labels and rewards? Is the performance robust? Or does performance converge to a fraction of the noise-free label as simply as "you get what you pay for /garbage in, garbage out"?}
\end{quote}
\end{PaperCard}


 Approaches such as LLM-as-Judge and RLAIF (\cite{Bai2022ConstitutionalAI,Rlaif_vs_rlhf})  have shown some promise to replacing human preference labeling with AI feedback (synthetic preferences).
Alternative methods attempt to remove the need for ground-truth supervision altogether, relying instead on label-free or self-rewarding mechanisms. However, these too are vulnerable to the same sources of noise, such as false positives and false negatives. For example, methods based on majority voting (\cite{zuo2025ttrl}), using the consistency of reasoning  traces (\cite{zhang2025consistent}) or using it's own model internal feedback (RLIF), like self-certainty--which leverage the log-probability confidence of generated sequences--aim to approximate reliable reward signals without external supervision (e.g., \cite{intuitor, deep_think_with_confidence}). 
\begin{TakeawayDrop}{Motivation}
\textbf{Learning without ground truth?} Reinforcement learning is fundamentally driven by reward and environment feedback, and is therefore highly sensitive to their quality. Can an agent still learn and self-improve when the feedback is noisy and no direct ground-truth signal is available?\newline To make this question precise, We begin by analyzing \textbf{RL dynamics under noisy feedback}. 
\end{TakeawayDrop}

\begin{figure}
    \centering
    \includegraphics[width=1\linewidth]{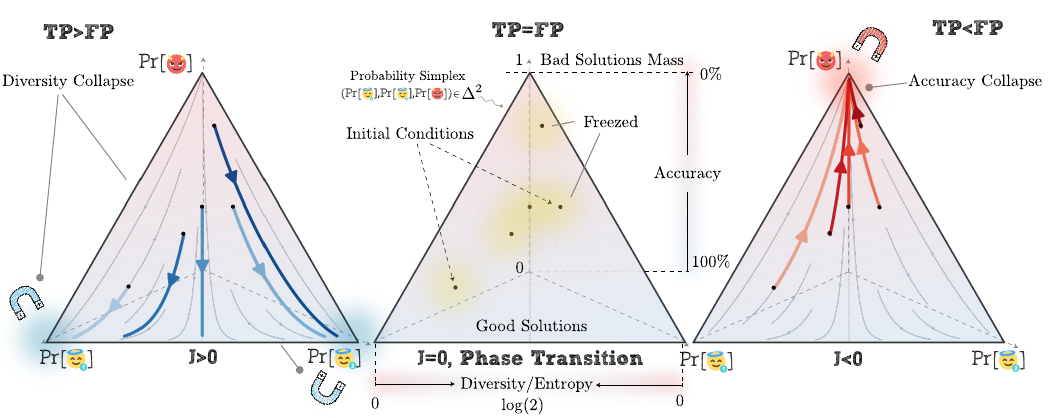}
    \caption{
    \textbf{The Winner Takes It All.}
The GRPO mean-field dynamics exhibit a striking structural property: in the absence of any symmetry among initial good arm masses, one arm ultimately dominates.
When the oracle is in the \emph{learning} regime ($J = \mathrm{TPR} - \mathrm{FPR} > 0$), the total bad-arm mass $p(t)$ converges to $0$, and among the good arms the algorithm converges to the one with the highest initial probability  (see Appendix \ref{sec:inner-good} for a rigorous theoretical analysis of this phenomenon.).
At the \emph{critical} boundary ($J=0$), the probability vector remains fixed throughout training, yielding a continuum of equilibrium points.
In the \emph{unlearning} regime ($J<0$), the flow reverses and the bad arm becomes the almost-sure winner.
This “winner-take-all” behavior is an inherent feature of the GRPO-style replicator dynamics.}
    \label{fig:placeholder}
\end{figure}

\paragraph{Noisy Reward.} Both major sources of supervision--human annotations and synthetically generated examples--are inherently prone to \emph{label noise}, which directly induces \emph{reward noise} in reinforcement learning settings. In addition, real-world automatic checkers are often imperfect: unit tests may be incomplete, edge cases may be missed, and multiple correct solutions may go unrecognized, etc. This problem becomes even more acute when LLM models produce the labels or rewards (\cite{grattafiori2024llama}), since AI-generated supervision can introduce systematic biases rather than purely random noise.
\paragraph{Noisy Reward for Coding Tasks.} These challenges are particularly severe in RLVR for coding, where verification relies on incomplete test suites and admits many semantically correct implementations, unlike short-answer math problems (e.g., AIME) or multiple-choice benchmarks with a fixed answer key (e.g., MMLU; \cite{hendryckstest2021}). Although current generative models perform impressively on many programming tasks, as we push toward increasingly difficult problems we should expect test coverage and fidelity to deteriorate. In the extreme, pass/fail outcomes can become essentially uncorrelated with true functional correctness, so that unit-test based verification approaches (and may even fall below) chance-level reliability. For this reason, we center our experiments on Python programming tasks, where imperfect verification is the norm and the resulting noise regime is both realistic and practically important.


In practice, these issues manifest primarily as \textbf{False Positives (FP)}--incorrect solutions that receive positive reward-and \textbf{False Negatives (FN)}--correct solutions that receive zero or negative reward. Together, these developments highlight an increasing dependence on potentially noisy supervision signals in RL training pipelines. This naturally raises a central question:\newline

\begin{PaperCard}
\begin{quote}
\emph{ Is there a noise threshold beyond which learning becomes unreliable and may even collapse, especially under AI-generated supervision?}

\end{quote}
\end{PaperCard}

To address this question, we start by considering the per-prompt sequence-level graders, and the reward is a noisy binary signal \(r\in\{0,1\}\). We model reward/label noise by two rates of flipping the true label with  false–negative and false–positive rates
\begin{equation}\label{eq:def-deltas}
\delta_{\mathrm{FN}}=\Pr(r=0\mid \text{good}),\qquad
\delta_{\mathrm{FP}}=\Pr(r=1\mid \text{bad}).
\end{equation}
Notice that these error rates in general can be time dependent, i.e $\delta_{\mathrm{FP},\mathrm{FN}}=\delta_{\mathrm{FP},\mathrm{FN}}(t)$.
Based on these two levels of noise, we can define a scalar 
\begin{equation}\label{eq:def-J}
\boldsymbol{J \;\coloneqq\; 1-\delta_{\mathrm{FN}}-\delta_{\mathrm{FP}}
\;=\;\mathrm{TPR}-\mathrm{FPR}}\footnote{Here, \(\mathrm{\bold{TPR}}:=\text{\textbf{T}rue \textbf{P}ositive \textbf{R}ate}=1-\delta_{\mathrm{FN}}\) and \(\mathrm{\bold{FPR}}:=\text{\textbf{F}alse \textbf{P}ositive \textbf{R}ate}=\delta_{\mathrm{FP}}\). For brevity, we also use TP and FP to denote the \emph{rates} 
$\mathrm{TPR}$ and $\mathrm{FPR}$. 
Thus TP/FP should always be read as TPR/FPR.}\in[-1,1] .
\end{equation}
which summarizes the \emph{net discriminative power} of the checker. In statistical decision theory this is known as \emph{\textbf{Youden’s index}}~\cite{youden1950index}. Concretely,
\begin{itemize}[leftmargin=13.5em]
\item \(J=1\): Perfect rewarder (\(\mathrm{TPR}=1,\ \mathrm{FPR}=0\));
\item \(J=0\): Chance-level (uninformative) rewarder (\(\mathrm{TPR}=\mathrm{FPR}\));
\item \(J<0\): Inverted or anti-informative rewarder.
\end{itemize}
Geometrically, \(J\) equals the \emph{vertical distance} between the ROC curve of the checker and the diagonal ``random'' line; thus it measures how far the verifier’s decisions deviate from random guessing. 

Recently, controlled experiments such as \cite{chen2025acereason} have shown that
imbalances between false-positive and false-negative rewards can induce severe
mode collapse during RL training. Related phenomena have also been documented
in mechanisms based on self-certainty, majority voting, or entropy
regularization (\cite{zhang2025co}), where systematic reward skew leads the model
to collapse onto a narrow subset of outputs. Methods that rely on LLM-as-Judge
further suffer from reward hacking, revealing structural vulnerabilities in
model-generated supervision. Complementary theoretical analyses of
\emph{noisy verifiers} in RLVR have begun to emerge (\cite{cai2025noisyrlvr}),
but despite this progress, a fundamental question remains:






\begin{TakeawayDrop}{This paper asks}
\begin{itemize}[leftmargin=0em]
\item \textbf{How much reward sloppiness can RLVR tolerate?}  
Concretely, under what levels of label/reward noise does GRPO/REINFORCE continue to improve accuracy, and when do characteristic failure modes (e.g., collapse or reward inversion) begin to emerge?

\item \textbf{What are the learning dynamics of RLVR under noisy rewards?}  
How do the magnitude and structure of noise affect asymptotic accuracy? In the infinite-training limit, can learning under noisy rewards reach the same performance as with noise-free signals?

\item \textbf{When and how do these dynamics break down-is there a phase transition?}  
Does the system exhibit a critical threshold beyond which learning abruptly fails or reverses direction, analogous to a phase transition in physical systems?

\item \textbf{Can we predict the \emph{rate} of accuracy improvement analytically?}  
Can we derive a closed-form drift (or mean-field ODE) for an accuracy-related state variable (e.g., total bad-mode mass \(p(t)\)) that yields time-to-accuracy predictions and cleanly separates \emph{rate} effects from \emph{fate} (the limiting performance)?
\end{itemize}
\end{TakeawayDrop}

To address these questions, we introduce the framework \textbf{RLV}$^\varepsilon$\textbf{R}: \textbf{R}einforcement \textbf{L}earning with \textbf{V}erifiable \textbf{Noisy} \textbf{R}ewards, to study the effects of noise through the lens of analytical insight.

\section{RLVR with Sloppy Rewards}

 Group-normalized reinforcement learning with verifiable rewards (RLVR) is an effective method to improve the policy optimization via simple, iterative loop. The goal is to maximize the expected reward by contrasting the performance of multiple completions generated from the same prompt. A single iteration proceeds as follows:

\begin{enumerate}[leftmargin=*, label=\textbf{\arabic*.}]
    \item \textbf{Sampling.} 
    For a given prompt $x$, we draw a cohort of $G$ independent completions from the current policy:
    \[
        y_1,\ldots,y_G \sim \pi_\theta(\cdot\mid x).
    \]
    Equivalently, sample indices $I_g \overset{\mathrm{iid}}{\sim}\mathrm{Categorical}(\mathbf p)$ with $\mathbf p=\mathrm{softmax}(\boldsymbol z)$, and let $y_g$ be the completion associated with $I_g$.
    \item \textbf{Scoring.} 
    We assign a raw reward $r_g := r(x,y_g)$ to each completion using a programmatic rule or a learned reward model.
    
    \item \textbf{Group Normalization.} 
    To isolate the relative quality of each completion, we compute per-sample advantages by standardizing rewards within the group:
    \[
      \widehat{A}_g \;=\; \frac{r_g - \bar r}{\sigma_r + \varepsilon},
    \]
    where $\bar r$ and $\sigma_r$ are the empirical mean and standard deviation of the rewards in the current batch.
    
    \item \textbf{Update.} 
    We perform a policy-gradient step to reinforce high-advantage completions (deferring PPO-style clipping or KL penalties to \S.\ref{subsec:mean-field-bad-mass}). The core update is:
    \begin{equation}\label{eq:objective-function}
      \Delta \theta
      \;=\;
      \eta\,\frac{1}{G}\sum_{g=1}^G \widehat{A}_g\,\nabla_{\theta}\log \pi_\theta(y_g\mid x).
    \end{equation}
\end{enumerate}

    
    
    

\begin{figure}
    \centering
    \includegraphics[width=1\linewidth]{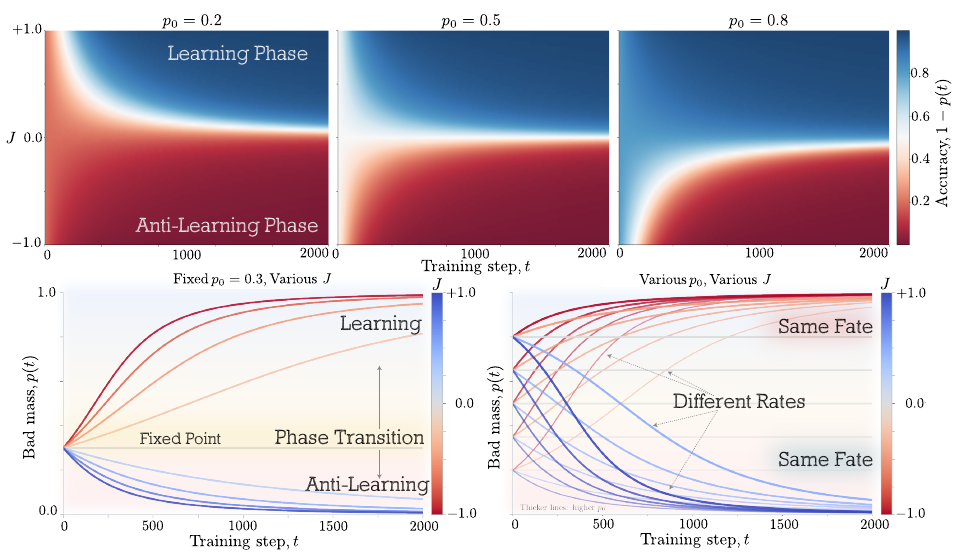}
    \caption{\textbf{ Rate, not Fate.} Our mean-field multi-armed bandit analysis predicts that, whenever learning succeeds ($J>0$), training with noisy rewards asymptotically converges to the same final accuracy as training on clean rewards; the noise level controls only the convergence speed, not the eventual performance. 
    The same phenomena happen for $J<0$. Our mean field model, along with experimental observation in Fig.~\ref{fig:phase-transition-experimental-results}  supports the ``convergence rate is different, fate is the same'' prediction.}
    \label{fig:placeholder}
\end{figure}
\subsection{Learning Dynamics}
To understand the essence of this update mechanism, let's consider the simple binary outcome setup 
 where the LLM generates either a ``good'' or ``bad'' solution. (In the \S \ref{sec:simplex-dynamics}, we will generalize the setup to the most general case.) Let $p$ be the probability of generating a bad solution, controlled by a logit $z$ such that $p=\sigma(z)=1/(1+e^{-z})\equiv\pi({\text{Bad}})$.

\paragraph{Good vs. Bad.}
For a per-sample normalized advantage $\widehat{A}$, the update direction is determined by the correlation between the normalized advantage $\widehat{A}$ and the gradient of the log-probability such that the expected logit update is
\[
\Delta z\ \propto\ \mathbb{E}\!\big[\widehat{A}\,\nabla_z\log\pi(a)\big].
\]
 In a binary setting, the score function simplifies to $\nabla_z\log\pi(\text{bad}) = 1-p$ and $\nabla_z\log\pi(\text{good}) = -p$.

By defining the conditional expected advantages $f(\text{bad}) = \mathbb{E}[\widehat{A} \mid \text{bad}]$ and $f(\text{good}) = \mathbb{E}[\widehat{A} \mid \text{good}]$,
we can compute the full expectation over actions:
\[
\mathbb{E}[\widehat{A}\,\nabla_z\log\pi(a)]
= p\,f(\mathrm{bad})\,(1-p) + (1-p)\,f(\mathrm{good})\,(-p)
= p(1-p)\bigl(f(\mathrm{bad})-f(\mathrm{good})\bigr).
\]

Passing to continuous time using $\dot p=\frac{dp}{dz}\,\dot z=p(1-p)\,\dot z$ together with \eqref{eq:objective-function}, we observe:

\begin{TakeawayDrop}{The Law of  Motion}
Group normalization up-weights better-than-average samples and down-weights worse-than-average ones. In the simplest two-class reduction (``good'' vs.\ ``bad''), this induces a \emph{replicator-skeleton} for the bad mass $p(t)=\pi_\theta(\mathrm{bad}\mid x)$:

\begin{equation}\label{eq-ODE-simple}
\dot p(t)
= -\,\eta\,\bigl[p(t)\bigl(1-p(t)\bigr)\bigr]^2 \,\Bigl(f(\mathrm{good})-f(\mathrm{bad})\Bigr),\qquad \textbf{GRPO Dynamics}
\end{equation}
where $f(\mathrm{class}):=\E[\widehat{A}\mid \mathrm{class}]$. Thus if $f(\mathrm{good})>f(\mathrm{bad})$, then $\dot p<0$ and accuracy rises.
\end{TakeawayDrop}

Despite its simplicity, \eqref{eq-ODE-simple} provides profound insight into the behavior of RL algorithms under group normalization. As long as good solutions yield higher normalized scores than bad ones, that is, when $f(\mathrm{good})>f(\mathrm{bad})$, the system drains probability mass, $p(t)$, from the bad state.
 
\paragraph{Interpretation.}
More generally, the dynamics of the form
\[
\dot p_i(t)=p_i(t)\Bigl(f_i\bigl(p(t)\bigr)-\bar f\bigl(p(t)\bigr)\Bigr),
\qquad
\bar f(p)=\sum_j p_j\, f_j(p),\qquad \textbf{Replicator Dynamics}
\]
are called the \emph{\textbf{replicator dynamics}}. Here, $f$ is the \emph{\textbf{fitness}} function, $p_i\ge 0$ and $\sum_i p_i=1$, and each type (or strategy) $i$ is rewarded or penalized according to how its fitness compares with the population average: types with above‑average fitness grow in frequency, while those with below‑average fitness decline (\cite{cressman2003evolutionary}).

\begin{PaperCard}
\begin{quote}
\emph{GRPO is like a \textbf{natural selection}. The one with \textbf{fitness} above the average will survive eventually.} 
\end{quote}
\end{PaperCard}

\subsection{Noisy Rewards}

To understand how reward noise shapes GRPO dynamics in \eqref{eq-ODE-simple}, we consider the binary-reward
setting $r\in\{0,1\}$ in which the observed signal is corrupted by the
false-positive and false-negative rates introduced in
\eqref{eq:def-deltas}. As summarized by Youden’s index $J$
in~\eqref{eq:def-J}, such a ``sloppy'' grader departs from the ground truth. Under this noise model, the effective success probability $q(p)$ and the reward
variance $\sigma(p)^2$ become functions of the bad-mass $p$:
\[
q(p) := \mathbb{E}[r]
      = (1-\delta_{\mathrm{FN}})
        - (1-\delta_{\mathrm{FP}}-\delta_{\mathrm{FN}})\,p,
\qquad
\sigma(p)^2 := \mathrm{Var}[r]
              = q(p)\bigl(1-q(p)\bigr).
\]

\paragraph{Signal vs. Noise.}
How does this corruption affect the learning dynamics? Remarkably, under group normalization, the advantage gap takes a purely geometric form (see Appendix~\ref{appendix:Group-Normalized RL with Noisy Reward}):
\begin{equation}\label{eq:delta_r_noisy}
\mathbb{E}[\widehat{A}\mid \text{good}] - \mathbb{E}[\widehat{A}\mid \text{bad}] \;=\; \frac{1-\delta_{\text{FN}}-\delta_{\text{FP}}}{\sigma(\delta_{\text{FN}},\delta_{\text{FP}},p)}\equiv\; \frac{J}{\sigma(p)}.
\end{equation}

\paragraph{The Crucial Role of Youden's Index.}
Substituting this back into the equation describing our dynamics, we see that $J$ acts as a signed coefficient of friction for learning.
\begin{itemize}[leftmargin=*]
    \item \textbf{If $J > 0$ (Signal):} The grader is better than random chance. The update pushes mass toward the correct solution, though the speed is throttled by the noise level $J$. 
    \item \textbf{If $J < 0$ (Anti-Signal):} The grader is systematically misleading. Learning actively reverses, optimizing for the wrong objective.
    \item \textbf{If $J = 0$ (Noise):} The gradient vanishes. The advantage gap collapses to zero, leaving the policy drifting in neutral territory regardless of the sample size.
\end{itemize}

\section{Phase Transition}


Let's continue to work with the binary outcome setting of the previous section.
Substituting the noisy reward gap from \eqref{eq:delta_r_noisy} into
the two-state GRPO dynamics~\eqref{eq-ODE-simple}, we obtain the following
scalar dynamics for the bad mass $p = \Pr(\text{bad})$:
\begin{equation}\label{eq:ODE-simple-binary}
    \dot{p} 
    \;=\; 
    -\,\eta\,\frac{J}{\sigma(p)}\,p^{2}(1-p)^{2},
    \qquad 
    J = \mathrm{TPR} - \mathrm{FPR}.
\end{equation}
Here $J$ is Youden’s index (the effective signal polarity), and $\sigma(p)>0$
is the group-normalized reward standard deviation at bad-mass level $p$.



\begin{TakeawayDrop}{Three Learning Regimes.}
    The sign of $J$ completely determines learning behavior: (Assuming $p_0\neq0,1\equiv$ degenerate cases)
\begin{equation*}
J > 0
\quad \Rightarrow \quad 
\dot{p} < 0 
\quad \Rightarrow \quad 
\textbf{Learning}: \text{ bad mass shrinks, accuracy rises to $1$}
\end{equation*}
\begin{equation*}
J = 0
\quad \Rightarrow \quad 
\dot{p} = 0 
\quad \Rightarrow \quad 
\textbf{Neutral}: \text{no systematic improvement, pure drift}
\end{equation*}
\begin{equation*}
J < 0
\quad \Rightarrow \quad 
\dot{p} > 0 
\quad \Rightarrow \quad 
\textbf{Anti-learning}: \text{bad mass grows, accuracy decays to $0$}
\end{equation*}
\end{TakeawayDrop}


This reveals a \textbf{sharp phase transition} at the critical boundary $\mathrm{TPR} = \mathrm{FPR}$, where the reward signal crosses from informative to misleading. 

Notice that in the noise-free setting (a perfect oracle),
\eqref{eq:ODE-simple-binary} simplifies to
\begin{equation}\label{eq-ODE-perfect-oracle}
    \dot{p}
    \;=\;
    -\,\eta\,p^{3/2}(1-p)^{3/2},\qquad \textbf{GRPO Dynamics with Perfect Oracle}
\end{equation}
which represents the baseline GRPO dynamics in the absence of reward noise.

\medskip
\noindent\emph{Note: In practice, $\mathrm{TPR}$ and $\mathrm{FPR}$ may drift as policy and rewarder co-evolve. without losing any generality, the equation above captures the instantaneous learning direction at each moment.}

\subsection{Bifurcation at the Critical Point}

\paragraph{Fixed points and stability.} 
For any $J\neq 0$, the system has two boundary equilibria: $p^\star = 0$ (all good) and $p^\star = 1$ (all bad), where $p^2(1-p)^2$ vanishes. Which equilibrium attracts depends entirely on the sign of $J$:
\begin{itemize}[leftmargin=*,itemsep=0.3em]
    \item \textbf{When $J>0$ (learning regime):} $p=0$ is globally attracting on $(0,1)$ while $p=1$ repels. Starting from any nontrivial mixture, accuracy $1-p(t)$ climbs monotonically toward perfection.
    
    \item \textbf{When $J<0$ (anti-learning regime):} The basin structure inverts completely. Now $p=1$ attracts and $p=0$ repels, causing RLVR updates to systematically degrade performance until accuracy collapses to zero.
    
    \item \textbf{At the knife edge $J=0$:} The entire interval $[0,1]$ becomes a continuum of neutrally stable fixed points---directional information vanishes entirely.
\end{itemize}

\medskip
\noindent\textbf{Crossing $J=0$ induces a fundamental qualitative change} in learning dynamics, delineating the boundary between reward signals that guide learning and those that actively mislead it. The ODE is identically zero in this case.

\paragraph{Special case: $J=1$.}
In the noise-free regime ($J=1$), where rewards are perfectly reliable,
the closed-form solution of~\eqref{eq-ODE-perfect-oracle} gives an explicit
trajectory for the bad-arm probability as a function of gradient steps:


\begin{equation}\label{eq:p(t)}
p(t)
=
\begin{cases}
\displaystyle
\frac{1}{2}
\;+\;
\frac{1}{2}\,
\frac{\varphi(p_0)-\frac{\eta}{2}t}{\sqrt{\,4+\bigl(\varphi(p_0)-\frac{\eta}{2}t\bigr)^2\,}}, 
& \text{if } p_0\neq 0,1, \qquad
\varphi(p)\;:=\;\frac{2p-1}{\sqrt{p(1-p)}}.\\[0.75em]
0,1, & \text{if } p_0 = 0,1 \,\, (\equiv \text{LLM has/hasn't absolute clue about the prompt})
\end{cases}
\end{equation}

In particular, if $p_0 \neq 0,1$, then the late–time asymptotic is
\[
p(t)\;\sim\;\frac{4}{\eta^2 t^2}\rightarrow 0,
\qquad t\to\infty.
\]
Thus, under perfectly reliable labels, the bad mass decays with a universal $t^{-2}$ tail: the accuracy $1-p(t)$ races toward $1$ at a polynomial rate determined entirely by the learning rate $\eta$.
 
%


\paragraph{RLVR limitation.}
Equation~\eqref{eq:p(t)} exposes an \emph{support} barrier: the boundary states $p_0\in\{0,1\}$ are absorbing, so modes with zero initial mass cannot be created by RLVR. In particular, the analysis assumes $1-p_0>0$ so that correct solutions are sampled occasionally and the gradient has something to amplify; if instead $1-p_0=0$ (the prompt lies beyond the base model’s capability), the dynamics are trapped at the degenerate equilibrium $p(t)\equiv 1$ and learning never takes off. Thus RLVR can sharpen and reweight reasoning paths already present in the base model, but it cannot reliably expand capability beyond its initial support, consistent with findings that RLVR mainly boosts pass@$1$ while large-$k$ coverage can shrink (\cite{yue2025does}).

\paragraph{Asymptotics: which tail, and when.}
The late-time behavior depends critically on whether the reward variance $\sigma(p)$ vanishes at the attracting equilibrium. Define the boundary variances. Let $\sigma_0=\sqrt{(1-\delta_{\mathrm{FN}})\,\delta_{\mathrm{FN}}}$ and
$\sigma_1:=\sqrt{\delta_{\mathrm{FP}}(1-\delta_{\mathrm{FP}})}$.
The late-time decay is governed by whether the reward variance \(\sigma(p)\) \emph{vanishes} at the attracting vertex.

\paragraph{Case (i): $J>0$ with attractor at $p=0$.}
Two distinct regimes emerge:
\begin{itemize}[leftmargin=1.2em,itemsep=2pt]
\item If \textbf{Nondegenerate noise} ($\delta_{\mathrm{FN}}>0$), then $\sigma(p)\to\sigma_0>0$ and late-time behavior ($\time \rightarrow 0$) of the dynamics will be 
\[
\dot p \;=\; -\eta\,\frac{J}{\sigma_0}\,p^2 \;+\;o(p^2)
\quad\Rightarrow\quad
p(t)\;\sim\;\frac{\sigma_0}{\eta J}\,\frac{1}{t}.
\]
\item If \textbf{Variance-degenerate case} ($\delta_{\mathrm{FN}}=0$), then
$\sigma(p)=\sqrt{q(1-q)}\sim\sqrt{Jp}$ and
\[
\dot p \;=\; -\eta\sqrt{J}\,p^{3/2} \;+\;o\!\bigl(p^{3/2}\bigr)
\quad\Rightarrow\quad
p(t)\;\sim\;\frac{4}{\eta^2 J}\,\frac{1}{t^2}.
\]
\end{itemize}

\paragraph{Case (ii): $J<0$ with attractor at $p=1$.}
Defining the good mass $u(t):=1-p(t)$, we find $q(1)=\delta_{\mathrm{FP}}$ and $\sigma(1)=\sigma_1$. Since $\sigma_1>0$ whenever $\delta_{\mathrm{FP}}\in(0,1)$,
\[
\dot{u} \;=\; -\,\frac{\eta|J|}{\sigma_1}\,u^2 \;+\;o(u^2)
\quad\Rightarrow\quad
u(t)=1-p(t)\;\sim\;\frac{\sigma_1}{\eta|J|}\,\frac{1}{t}.
\]
Note that no $t^{-2}$ regime appears here: setting $\delta_{\mathrm{FP}}=0$ would force $J\ge 0$, precluding this case.

Therefore, the learning direction is determined by $\operatorname{sign}(J)$, while the asymptotic tail follows a universal pattern:
\[
\text{error} \sim 
\begin{cases}
\mathcal{O}(t^{-1}) & \text{when reward variance is nonzero at the attractor,} \\
\mathcal{O}(t^{-2}) & \text{in the variance-degenerate case } (\delta_{\mathrm{FN}}=0, J>0).
\end{cases}
\]
This provides a crisp analytical characterization of test-time error $p(t)$ under noisy reward signals.

\subsection{Rate, Not Fate}

As long as $J>0$, the one–dimensional ODE for 
the bad mass has a single basin of attraction. In particular, both the noisy and the noise-free dynamics converge to the same limiting state (up to the critical knife-edge where the sign of $J$ flips). In this sense, reward noise does not change the \emph{fate} of training: the basin is the same, and so is the final performance. What changes is only the \emph{speed} at which we flow toward that basin.

More precisely, comparing the noisy and perfectly reliable dynamics gives the simple time–rescaling
\begin{equation}
    \frac{\dot p_{\mathrm{noisy}}}{\dot p_{\mathrm{perfect}}}
    \;\propto\;
    \frac{1}{J}
\end{equation}
Thus, for example, when $J=0.5$, the noisy system needs roughly twice compute steps to trace out the same trajectory in $p$. In other words, additional compute can compensate for imperfect data label/reward signal.

In the next sections, we generalize this analysis to the multi-solution setting and incorporate additional components such as importance sampling, the clipping ratio, and a possible KL penalty term, yielding a formulation that fully aligns with practical PPO/GRPO-style algorithms.

\subsection{Maximal Learnability at Intermediate Bad Mass}\label{subsec:p-half}

In the noiseless regime $J = 1$, our scalar dynamics for the bad mass $p$ satisfy
\[
|\Delta p| \;\propto\; \bigl[p(1-p)\bigr]^{3/2}.
\]
The prefactor $p(1-p)$ is maximized at
\[
p^\star \;=\; \tfrac{1}{2},
\]
so the largest single-step reduction in bad-arm mass occurs when the current bad probability is neither too small nor too large, but instead sits at the ``intermediate'' value $p \approx 1/2$. As $p \to 0$ or $p \to 1$, the factor $p(1-p)$ vanishes, and the dynamics slow down: once a prompt is either almost always solved correctly or almost always answered incorrectly, additional GRPO steps make only marginal progress on that prompt.

\paragraph{Connection to prior ``\(p(1-p)\)'' learnability observations.}
The emergence of an intermediate-difficulty optimum mirrors patterns observed in several seemingly distinct analyses.
\citet{bae2025online} derive progress bounds proportional to \(p(1-p)\) and empirically find that prompts with \(p(x)\approx 0.5\)
are most learnable. \citet{foster2025learning} relate learnability to reward variance; for Bernoulli rewards,
\(\Var(r)=q(1-q)\) is maximized at \(q=\tfrac12\).
Our mean-field GRPO dynamics provides a complementary \emph{dynamical} explanation: the same non-saturation phenomenon that yields high
information content also governs the instantaneous rate at which bad mass is eliminated. See Appendix \ref{sec:max-learnability} for more details.


\begin{quote}
    \textbf{Training is most efficient on ``medium-difficulty'' questions, where the model is roughly $50$--$50$ between good and bad solutions. Under asymmetric noise the optimum shifts but remains at an intermediate bad mass.}
\end{quote}

\section{LLM as a Multi-Armed Bandit}\label{sec:llm-bandit}

In contemporary reinforcement learning (RL) training paradigms for LLMs, such as RLVR and related frameworks, the supervision signal is typically provided only after the model has produced an entire response. Because the reward is evaluated at the completion level rather than per token, it is often more appropriate to treat the entire output sequence as a single decision made by the policy. This viewpoint naturally suggests a bandit-style abstraction, where each sequence corresponds to one action (or "arm") and the learning signal is attached to that action as a whole \citep{kreutzer2017bandit, nguyen2017reinforcement, dang2025weight}. Earlier sequence-level policy gradient methods--such as RLOO and related REINFORCE variants \citep{ahmadian2024back}--implicitly operated in this regime, while contemporary approaches like GRPO \citep{shao2024deepseekmath} make this perspective explicit by defining advantages and updates directly over full generations.

Adopting the bandit abstraction provides a clean and principled theoretical foundation for our methodology (see Appendix~\ref{app:MAB} for further details).

\paragraph{From sequences to modes.}
To ground this abstraction, consider a fixed prompt \(x\) (e.g., a math problem). Sampling an LLM at non-zero temperature yields many distinct token sequences, yet these typically collapse into a small number of \emph{reasoning modes}--canonical solution paths, recurring chains of thought, or standard solver templates. This clustering exposes a low-dimensional structure: each mode acts as an ``arm'' in the bandit abstraction, carrying its own probability mass. (Recent work has begun to map the landscape of such reasoning modes \citep{zhou2025landscape} and to leverage their structure for designing new training algorithms \citep{zhang2025consistent}.)

We now formalize this intuition. Fix a prompt \(x\) and let \(y \sim \pi_\omega(\cdot \mid x)\) be a sampled completion. With a length cutoff \(L_{\max}\), the effective support \(\mathcal{Y}_{\le L_{\max}}=\bigcup_{\ell=1}^{L_{\max}}\mathcal{V}^{\ell}\) is finite. We define a coarse-graining map \(\phi:\mathcal{Y}_{\le L_{\max}}\!\to\!\mathcal{H}\) that clusters sequences into \emph{reasoning modes} \(\mathcal{H}=\{h_1,\dots,h_{K+M}\}\). This induces a categorical \emph{mode policy}:
\[
\pi_{\theta}(h\mid x) := \sum_{y:\,\phi(y)=h}\pi^{(L)}_\omega(y\mid x), \qquad \pi^{(L)}_\omega(y\mid x)\;\propto\; \pi_\omega(y\mid x)\,\mathbf{1}\{y\in\mathcal{Y}_{\le L_{\max}}\},
\]
which we parameterize by \emph{effective logits} \(\theta=(\theta_1,\ldots,\theta_{K+M})\) such that \(\pi_\theta(h_i\mid x)=\mathrm{softmax}(\theta)_i\). The key quantity for our analysis is the total probability mass on incorrect modes, or the \emph{bad mass} \(p\):
\[
p \;=\; \Pr[\text{LLM chooses an incorrect mode} \mid x].
\]

Two boundary cases are notable: (i)~\textit{String-matching}, where \(\phi\) is the identity and \(\mathcal{H}=\mathcal{Y}_{\le L_{\max}}\); and (ii)~\textit{Infinite horizon}, where \(L_{\max}\to\infty\). In practice, our mode-level statements depend only on \(\pi_\theta(\cdot\mid x)\) over \(\mathcal{H}\) and remain invariant to the specific choice of a reasonable \(\phi\).

\paragraph{Good vs.\ bad families.}
To further analyze the distribution, we partition the set of modes into good (correct) and bad (incorrect) subsets:
\[
\mathcal{H}=\mathcal{H}^+\cup\mathcal{H}^-, \quad |\mathcal{H}^+|=K, \quad |\mathcal{H}^-|=M.
\]
Let \(\alpha\) and \(p\) denote the aggregate mass of the correct and incorrect families, respectively:
\[
\alpha \;=\;\sum_{h\in\mathcal{H}^+}\pi_\theta(h\mid x), \qquad p \;=\;\sum_{h\in\mathcal{H}^-}\pi_\theta(h\mid x) = 1 - \alpha.
\]
We define the relative distribution within the good modes as \(y \in \Delta^{K-1}\) and within the bad modes as \(z \in \Delta^{M-1}\), where:
\[
y_i := \frac{\pi_\theta(h_i\mid x)}{\alpha} \text{ for } h_i \in \mathcal{H}^+, \qquad z_j := \frac{\pi_\theta(h_j\mid x)}{p} \text{ for } h_j \in \mathcal{H}^-.
\]
Thus, the full distribution over all arms is given by the vector \((\alpha y_1, \dots, \alpha y_K, p z_1, \dots, p z_M)\). This decomposition allows us to analyze the model's performance (\(\alpha\) vs. \(p\)) independently from its internal preference for specific reasoning paths (\(y\) and \(z\)).

\begin{TakeawayDrop}{LLM as Multi-Armed Bandit}
\emph{Treat each prompt as a tiny bandit problem.} The model spreads probability across a few \textbf{reasoning modes} (arms). A sequence-level grader marks modes as correct (1) or incorrect (0). GRPO pushes probability toward above-average arms and away from below-average ones. Under noise-free rewards, the mass on incorrect modes shrinks monotonically.
\end{TakeawayDrop}

\section{Geometric Flow on the Probability Simplex}
\label{sec:simplex-dynamics}

\begin{figure}[t]
    \centering
    \includegraphics[width=1\linewidth]{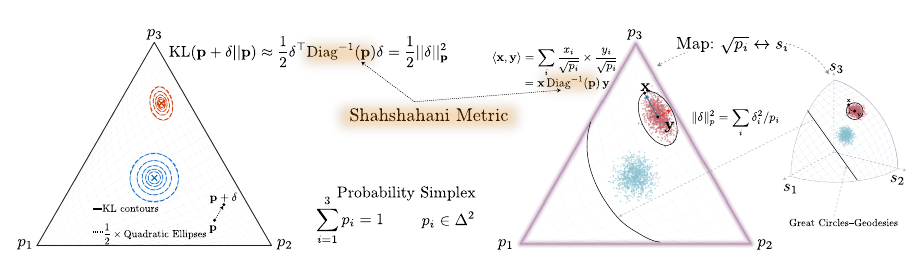}
    \caption{\textbf{Geometry of the Probability Simplex.} The policy $\mathbf{p}$ evolves on the non-Euclidean manifold $\Delta^{K+M-1}$. The softmax Jacobian $\mathfrak{J}(\mathbf{p})$ endows this space with the Shahshahani (Fisher) geometry, projecting updates onto the zero-sum tangent space $T_{\mathbf{p}}$. GRPO induces a mass-conserving replicator flow, $\dot{\mathbf{p}}=\eta \mathfrak{J}(\mathbf{p}) \mathbf{A}$, which dynamically redistributes probability mass based on relative advantage. In the local tangent space, the forward KL divergence manifests as the quadratic form $\frac{1}{2}\delta^\top \mathrm{Diag}(\mathbf{p})^{-1}\delta$, identifying $\mathfrak{J}(\mathbf{p})$ as the inverse Riemannian metric.}
    \label{fig:shah}
\end{figure}

We model the policy as a probability vector $\mathbf{p}$ distributed over $K$ ``good'' arms and a distinct ``bad'' arm comprising $M$ internal modes (representing bad-mass). The state space is the simplex:
\[
    \mathbf{p} = (p_1, \ldots, p_K, \ldots, p_{K+M}) \in \Delta^{K+M-1},
    \quad \text{where} \quad 
    \Delta^{d} = \left\{ \mathbf{x} \in \mathbb{R}_{\ge 0}^{d+1} : \mathbf{1}^\top \mathbf{x} = 1 \right\}.
\]
We denote $p:=p_{\mathrm{b}}$ as the aggregate bad-mass coordinate. Admissible velocity fields are constrained to the tangent space $T_{\mathbf{p}}\Delta^{K+M-1} = \{v \in \mathbb{R}^{K+M} : \mathbf{1}^\top v = 0\}$, ensuring total probability mass is strictly conserved.

The geometry of optimization on this manifold is governed by the softmax Jacobian, $\mathfrak{J}(\mathbf{p}) = \mathrm{Diag}(\mathbf{p}) - \mathbf{p}\mathbf{p}^\top$. This operator acts as the inverse  of the Riemannian metric tensor (associated with the Shahshahani or Fisher-Rao metric), mapping gradients into natural gradients on the tangent space:
\begin{equation}
    \mathfrak{J}(\mathbf{p}) v \;=\; \mathbf{p} \odot (v - \bar{v}), \quad \text{with} \quad \bar{v} := \mathbf{p}^\top v.
\end{equation}

 Under the GRPO objective, the policy update follows a continuous-time natural gradient flow 
\[
\dot{\mathbf{p}}
\;=\;
\eta\,\mathfrak{J}(\mathbf{p})^{2}\,\mathbf{A},
\qquad
\textbf{\emph{GRPO dynamics on }}\Delta^{K+M-1},
\]
 where the vector $\mathbf{A}$ comes from GRPO-style advantage computation. Equivalently, this can be expressed in the ``replicator-flow'' form:
\begin{equation}
\label{eq:pdot_general}
    \dot{\mathbf{p}} 
    \;=\; 
    \eta \mathbf{p} \odot \left[ \mathfrak{J}(\mathbf{p})\mathbf{A} - \langle \mathbf{p}, \mathfrak{J}(\mathbf{p})\mathbf{A} \rangle \mathbf{1} \right].
\end{equation}
This ``replicator dynamics'' form reveals two key properties: (1) \emph{Multiplicativity}, ensuring the faces of the simplex remain invariant; and (2) \emph{Relative Performance}, where mass flows between arms strictly based on their advantage relative to the mean (see Appendix.~\ref{appendix:Group-Normalized RL with Noisy Reward} for full details).

\subsection{Decoupling the Dynamics: Shape vs. Good and Bad Mass}
To disentangle the evolution of the ``good'' policy structure from the decay of the ``bad'' masses, we decompose the state vector. In the simplex interior, we parameterize $\mathbf{p}$ as:
\[
    \mathbf{p} = \bigl((1-p)y, \; pz\bigr),
\]
where $y \in \Delta^{K-1}$ and $z \in \Delta^{M-1}$ are the normalized distributions over good and bad arms respectively, and $p \in (0,1)$ represents the total bad mass.

Applying this coordinate change to \eqref{eq:pdot_general} decouples the system into shape dynamics (internal to $y$ and $z$) and mass dynamics (governing $p$). The block-diagonal evolution equations are:
\begin{TakeawayPlain}
\begin{subequations}
\label{eq:coupled_dynamics}
\begin{align}
    \dot y &= +\kappa(p)\,  y \odot \left( y - \|y\|_2^2  \boldsymbol{1}  \right), \label{eq:good_shape}\\
    \dot z &= -\kappa(p)\, z \odot \left(  z - \|z\|_2^2  \boldsymbol{1} \right), \label{eq:bad_shape}\\
    \dot p &= -\eta \frac{J}{\sigma(p)} [p(1-p)]^2 \left( \|y\|_2^2 + \|z\|_2^2 \right), \label{eq:mass_decay}
\end{align}
\end{subequations}
\end{TakeawayPlain}
where the time-rescaling factor is $\kappa(p) := \eta \frac{J}{\sigma(p)} p(1-p)$.
\begin{remark}
In the noise-free regime ($J=1$) with $\sigma(p)=\sqrt{p(1-p)}$, this factor simplifies to
\[
\kappa(p)=\eta\,\sqrt{p(1-p)},
\qquad
\frac{\eta J}{\sigma(p)}[p(1-p)]^2={\eta}\,{{p(1-p)}}^{\frac{3}{2}}.
\]
This specialization highlights that noise affects the effective time scale and sign (direction) of the dynamics, while the geometric structure of the flows remains unchanged.
\end{remark}
\begin{remark}
In the mass dynamics \eqref{eq:mass_decay}, the denominator is the standard deviation $\sigma(p) = (q(p)(1-q(p)))^{1/2}$, where $q(p) := (1-\delta_{\mathrm{FN}}) - Jp$. Since $J = 1 - \delta_{\mathrm{FP}} - \delta_{\mathrm{FN}}$, we have $\sigma(p) \in (0,1)$ for all $p \in (0,1)$, except in the singular case where $J=0$.
\end{remark}

\paragraph{Interpretation.} The coupled system \eqref{eq:coupled_dynamics} highlights a competition between three distinct geometric forces:

\begin{description}
    \item[Diversity Collapse in Good Arms:] Equation \eqref{eq:good_shape} describes a self-reinforcing flow. The term $\|y\|_2^2$ (the collision probability) acts as a threshold: arms with mass $y_i > \|y\|_2^2$ grow super-linearly, causing the distribution to sharpen and diversity to collapse onto the optimal arms. Let
\(
S^\star \;:=\; \arg\max_{i\in[K]} y_i(0)
\), then for $J>0$,
\[
y_i(t)\to 0\quad (i\notin S^\star), \qquad\textbf{GRPO's Diversity Collapse }
\]
    
    \item[Entropy Increase in Bad Arms:] Conversely, the bad-mass distribution $z$ evolves under a \emph{negative} feedback loop \eqref{eq:bad_shape}. This flow pushes $z$ away from concentration and toward the uniform distribution (maximum entropy) on $\Delta^{M-1}$.
    \[
z(t)\ \longrightarrow\ \frac{1}{M}\mathbf 1,
\]
This ``spreading'' effect slows the decay of $p$ if the bad mass is diffuse (i.e., when $\|z\|_2^2 \approx 1/M$). 

    \item[Bad Mass Evolution:] The total bad mass $p$ decays monotonically (provided $J > 0$) at a rate proportional to $[p(1-p)]^2$, but the rate is modulated by the structural sparsity $\|y\|_2^2 + \|z\|_2^2$. In the late-time limit, as the good arms collapse ($\|y\|_2^2 \to 1$) and bad arms homogenize ($\|z\|_2^2 \to 1/M$), the decay rate stabilizes, driven purely by the system's reward gap $J$
    \[
    \dot p
    = -\eta\frac{J}{\sigma(p)}[p(1-p)]^2\bigl(\|z\|_2^2+\|y\|_2^2\bigr)
    \;\xrightarrow[t\to\infty]{}\;
    \;\approx\;
    -\eta\frac{J}{\sigma(p)}[p(1-p)]^2,
    \]
    Motion is slow near the faces $p \approx 0$ and $p \approx 1$ and faster in the interior. In the learning regime $J>0$, the prefactor $J/\sigma(p)$ is positive, ensuring $p$ is monotonically driven down.

\end{description}

Thus, the simplex dynamics factor cleanly into a shape evolution within the good-arms simplex and a mass evolution that suppresses the bad arms.

\subsection{The right geometry: Shahshahani metric on \texorpdfstring{$\Delta$}{Δ}}

The geometry of the simplex is fundamentally non-Euclidean. Displacing probability mass from a rare arm incurs a ``higher cost'' than displacing the same amount from a common arm. The Shahshahani metric (\cite{shahshahani1979new}) formalizes this intuition by weighting tangent directions inversely proportional to the square root of probabilities (refer to Appendix \ref{sec:Shahshahani} for details):
\[
\langle u, v\rangle_{\mathrm{Shah};\,\mathbf{p}} 
\;=\; \sum_{i=1}^{K+M}\frac{u_i v_i}{p_i},
\qquad u,v\in T_{\mathbf{p}}\Delta^{K+M-1}.
\]
This definition corresponds precisely to the Fisher-information metric (\cite{amari2019fisher}) of the categorical family, up to a constant factor. Given a smooth function \(F:\Delta^{K+M}\!\to\mathbb{R}\), the associated \emph{natural gradient} is derived by projecting and reweighting the Euclidean gradient:
\[
\mathrm{grad}_{\mathrm{Shah}} F(\mathbf{p})
\;=\;\mathfrak{J}(\mathbf{p})\,\nabla F(\mathbf{p})
\;=\;\mathbf{p}\odot\bigl(\nabla F(\mathbf{p})-\langle \mathbf{p},\nabla F(\mathbf{p})\rangle\,\mathbf{1}\bigr)
\in T_{\mathbf{p}}\Delta^{K+M-1}.
\]

This structure simplifies significantly for the good-arm subsystem. The evolution of \(\dot y\) in \eqref{eq:good_shape} expresses a Shahshahani gradient flow of the potential \(\Phi\), scaled by \(\kappa(p)\):
\[
\dot y \;=\; \kappa(p)\,\mathrm{grad}_{\mathrm{Shah}}\Phi(y),
\qquad \Phi(y):=\tfrac12\|y\|_2^2 = \tfrac12\sum_{i=1}^{K}y_i^2.
\]
Consequently, within the good block, the dynamics perform \emph{natural-gradient ascent} on a potential that incentivizes concentration. This mechanism drives the mass toward a single dominant good arm.


\begin{TakeawayDrop}{Takeaway}
In the zero-noise limit, GRPO reduces to \emph{pure selection}: centered rewards induce a replicator-type drift, with effective ``fitness'' given by the natural-gradient signal \(\mathfrak{J}(\mathbf{p})\mathbf{A}\), shifting probability mass from below-average to above-average types, and introducing no additional nonlinearities beyond the softmax Jacobian.

\end{TakeawayDrop}

\paragraph{KL-regularized mirror ascent and the replicator limit.}

 A complementary discrete-time perspective employs entropic (KL) regularization. Let \(\p\in\Delta^{K+M-1}\) and \(\A\in\mathbb{R}^{K+M}\) be fixed. For a step size \(\eta>0\), we consider the KL-regularized maximization problem (cf. Proposition \ref{prop:mirror-ascent-detailed}):

\begin{equation}
\label{eq:mirror-problem}
\p^{+}\;=\;\arg\max_{\q\in\Delta^{K}}
\Bigl\{\;\langle \A,\q\rangle \;-\; \tfrac{1}{\eta}\,D_{\mathrm{KL}}(\q\;\|\;\p)\;\Bigr\} \Longrightarrow\;\;
\p^{+} \;=\; \frac{\p\odot e^{\eta \A}}{\mathbf{1}^\top\!\bigl(\p\odot e^{\eta \A}\bigr)}.
\end{equation}

Since the objective function is strictly concave with respect to \(q\), the maximizer \(p^{+}\) exists and is unique. Furthermore, it adopts the familiar \emph{multiplicative-weights} (or \emph{exponentiated-gradient}) form.

For small \(\eta\), the first-order expansion of the mirror-ascent step \eqref{eq:eg-update} corresponds to an Euler step of the \emph{natural-gradient (replicator) flow}:
\[
\p^{+}-\p \;=\; \eta\,\p\odot\bigl(\A-\langle \p,\A\rangle\,\mathbf{1}\bigr) \;+\; O(\eta^2),
\]
which implies
\[
\dot \p=\eta\,\mathfrak{J}(\p)\,A \;=\; \eta\, \p\odot\big(A-\langle \p,A\rangle\,\mathbf{1}\big).
\]
In essence, entropic mirror ascent discretizes the intrinsic geometry. It constitutes steepest ascent with respect to the Shahshahani metric, and its infinitesimal limit recovers replicator dynamics on the simplex.

In summary, the comparison is as follows:
\[
\left\{
\begin{array}{ll}
\text{GRPO :} &
\dot{\mathbf{p}}\;=\;\eta\,\mathfrak{J}(\mathbf{p})^2\,\mathbf{A}, \\[0.75em]
\text{GRPO with KL regularization :} &
\dot \p \;=\;\eta\, \mathfrak{J}(\p)\,\A.
\end{array}
\right.
\]

The resulting analysis highlights three key structural properties:

\begin{itemize}
    \item \textbf{Preservation of the simplex constraint.} The GRPO-induced dynamics are intrinsically tangent and multiplicative; probabilities remain nonnegative and sum to unity.
    
    \item \textbf{Decoupling of good and bad arm dynamics.} In the decomposition \(\mathbf{p}=((1-p)y,pz)\), the \emph{shape} \(y\) and \(y\) of the good/bad-arm distribution evolve via a Shahshahani gradient flow, whereas the total bad mass \(p\) is monotonically suppressed.
    \item \textbf{Geometric consistency.} The Shahshahani metric captures the natural geometry of \(\Delta\). Both the continuous-time GRPO flow \eqref{eq:pdot_general} and the discrete-time KL-mirror step \eqref{eq:eg-update} represent steepest-ascent procedures under this geometry. The former utilizes the natural-gradient signal \(\mathfrak{J}(\mathbf{p})\mathbf{A}\), while the latter reduces to the replicator flow \(\p\odot(\A-\langle \p,\A\rangle\mathbf{1})\) in the small-step limit.
\end{itemize}


\subsection{Finite Sampling Cause Genetic Drift Noise }
If the GRPO mean-field limit is viewed as an analogue to replicator-style natural selection, then the finite sampling of rollouts introduces a stochasticity equivalent to genetic drift.
GRPO-style updates rely on \emph{group estimates}: for each prompt, one averages over a finite set of $G$ rollouts to form a normalized advantage and then applies a policy-gradient step. Replacing population expectations by the empirical group mean introduces an additional randomness even if the underlying reward model were fixed. This ``finite-$G$'' effect is conceptually separate from reward noise (e.g.\ $\delta_{\mathrm{FN}},\delta_{\mathrm{FP}}$): it is simply Monte Carlo error from having only $G$ samples. Consequently, the learning dynamics do not follow the deterministic mean-field drift exactly; instead, they fluctuate around it, with a typical fluctuation size shrinking like $O(G^{-1/2})$ (and proportional to the learning rate $\eta$). For example, if $\hat y$ is the empirical frequency from $G$ i.i.d.\ categorical draws with mean $y$, then
\[
\sqrt{G}\,(\hat y-y)\ \Rightarrow\ \mathcal{N}\!\big(0,\;\text{Diag}(y)-y\,y^\top\big).
\]

At the level of the probability vector $\mathbf p$ on the simplex, this sampling-induced randomness has the canonical Fisher/Wright--Fisher geometry. Indeed, if $I\sim\mathrm{Cat}(\mathbf p)$ and $e_I$ is the corresponding one-hot vector, the score feature $e_I-\mathbf p$ has covariance
\[
\mathrm{Cov}\!\Big[\frac{1}{G}\sum_{g=1}^G\bigl(e_{I_g}-\p\bigr)\Big]
\;=\;\frac{1}{G}\Bigl(\mathrm{Diag}(\p)-\p\p^\top\Bigr)
\;:=\;\frac{1}{G}\,\Sigma(\p).
\]
so averaging $G$ i.i.d.\ rollouts produces fluctuations of order $\Sigma(\mathbf p)/G$. In a diffusion (continuous-time) approximation, this appears as a Wright--Fisher-type noise with diffusion speed, $\nu$, added on top of the deterministic drift derived earlier, as demonstrated in \cite{dang2025assessing}:
\[
dp \;=\;\dot{p}\,dt \;+\; \frac{\eta\sqrt{\nu}}{\sqrt{G}}\;\Sigma(p)^{1/2}\,dW_t.
\]
Advantage normalization and noisy rewards (i.e., $\frac{1}{G}\sum_{g=1}^G \tilde{\boldsymbol{r}}\bigl(e_{I_g}-\p\bigr)$) primarily rescale the \emph{amplitude} of this diffusion by an order-one, state-dependent factor, while the simplex-shaped matrix $\Sigma(\mathbf p)$ (and thus the vanishing of noise near the boundary) remains the same.

\section{Mean-Field Dynamics of Bad Mass in GRPO}
\label{subsec:mean-field-bad-mass}

In the preceding sections, we analyzed the mean-field probability dynamics induced by REINFORCE-style updates and identified the aggregate bad probability mass as a convenient one-dimensional summary of the learning trajectory. In this section we extend that analysis to GRPO, viewed as a concrete instantiation of a group-normalized policy-gradient method. Relative to the basic REINFORCE update, GRPO typically incorporates two additional mechanisms: (i) PPO-style importance sampling with ratio clipping, and (ii) KL regularization. Our goal here is to derive the corresponding mean-field evolution for the bad mass and to clarify which parts of the algorithm influence the leading-order drift.

\paragraph{The role of clipping and importance sampling.}
The GRPO update inherits PPO's importance sampling and ratio clipping by modifying the objective in \eqref{eq:objective-function}. Concretely, each per-sample advantage $A_i$ is reweighted by a clipped importance ratio,
\[
A_i \;\longmapsto\; A_i\, \mathrm{clip}\!\Bigl(
\,\frac{\pi_{\mathrm{new}}(i)}{\pi_{\mathrm{old}}(i)},
\;1-\varepsilon,\;1+\varepsilon'
\Bigr),
\]
where $\varepsilon,\varepsilon'>0$ are the PPO/GRPO clip thresholds \cite{ppo}. In the mean-field limit, our calculation in Appendix~\ref{app:bad-arm-ppo-grpo} shows that, in the small-step regime with fixed thresholds and $\eta\ll \varepsilon,\varepsilon'$, clipping and importance sampling do not alter the leading-order mean-field drift. Their contribution is absorbed into the $\mathcal{O}(\eta^2)$ remainder, and therefore does not affect the first-order phase portrait. We refer to Appendix~\ref{app:bad-arm-ppo-grpo} for the detailed expansion. 

We now state the resulting closed mean-field equation for the aggregate bad mass, together with its internal-time logit form and the corresponding small-heterogeneity refinement:


In the multi-bad-arm setting, define the aggregate bad mass and the within-block normalized states by
\[
p(t):=\sum_{m=1}^M p_{b_m}(t)\in[0,1],\qquad
y_j(t):=\frac{p_j(t)}{1-p(t)}\in\Delta^{K-1},\qquad
z_m(t):=\frac{p_{b_m}(t)}{p(t)}\in\Delta^{M-1},
\]
and set the within-block collision masses
\[
s_2(t):=\|y(t)\|_2^2\in\Big[\tfrac1K,1\Big],\qquad
t_2(t):=\|z(t)\|_2^2\in\Big[\tfrac1M,1\Big].
\]
The associated geometry factor is
\[
C_{\mathrm{geo}}(t):=s_2(t)+t_2(t)\in\Big[\tfrac1K+\tfrac1M,\,2\Big],
\]
we have:
\begin{ThmBoxCapital}{Theorem}
\begin{theorem}[Bad-mass ODE, internal-time logit form, and first-order geometry reduction]
\label{thm:bad-ode-2}
Under group-normalized GRPO with small stepsize $\eta\ll \varepsilon, \varepsilon'$(PPO-clipping factors) and fresh on-policy groups, the aggregate bad mass obeys the mean-field ODE
\begin{equation}
\label{eq:bad-ode}
\dot{p}(t)
\;=\;
-\;\eta\,\frac{J}{\sigma\!\bigl(p(t)\bigr)}\;[p(t)\bigl(1-p(t)\bigr)]^2\;
C_{\mathrm{geo}}(t)
\;+\;\mathcal{O}(\eta^2),
\end{equation}
where $J=$TPR-FPR reflects the good--bad advantage gap, and $\sigma(p)>0$ is the group-normalization scale.

\paragraph{Internal-time logit form.}
Assume $J\neq 0$ and define the logit
\[
L(t):=\log\frac{p(t)}{1-p(t)},
\]
together with the internal time change
\begin{equation}
\label{eq:tau-def}
\tau(t)\;:=\;\int_0^t \eta\,\frac{|J|}{\sigma(p(u))}\,p(u)\bigl(1-p(u)\bigr)\,du.
\end{equation}
Viewing $p,y,z$ as functions of $\tau$ via $t=t(\tau)$, one has, whenever $p(t)\in(0,1)$,
\begin{equation}
\label{eq:logit-internal}
\frac{dL}{d\tau}\;=\;-\operatorname{sign}(J)\,C_{\mathrm{geo}}(\tau)
\;=\;-\operatorname{sign}(J)\,\bigl(s_2(\tau)+t_2(\tau)\bigr).
\end{equation}
If $J=0$, the deterministic drift term vanishes at this order.

\paragraph{Small-heterogeneity regime.}
Write the within-block states at $\tau=0$ as
\[
y(0)=\unif_K+v_0,\quad \sum_{j=1}^K (v_0)_j=0,
\qquad
z(0)=\unif_M+w_0,\quad \sum_{m=1}^M (w_0)_m=0,
\]
where $\unif_K=(1/K,\dots,1/K)$ and $\unif_M=(1/M,\dots,1/M)$. Define the heterogeneities
\[
\zeta_0:=\|v_0\|_2^2=s_2(0)-\frac{1}{K},
\qquad
\xi_0:=\|w_0\|_2^2=t_2(0)-\frac{1}{M}.
\]
Then Theorem~\ref{thm:etaExpansion} implies that, in the near-uniform regime (that is, while the conditions of that theorem hold),
\begin{equation}
\label{eq:L-expansion-in-main}
L(\tau)=L(0)
-\operatorname{sign}(J)\Big(\tfrac1K+\tfrac1M\Big)\tau
-\operatorname{sign}(J)\frac{K}{2}\zeta_0\Big(e^{\frac{2}{K}\tau}-1\Big)
-\operatorname{sign}(J)\frac{M}{2}\xi_0\Big(1-e^{-\frac{2}{M}\tau}\Big)
+\widetilde R_L(\tau),
\end{equation}
with a remainder $\widetilde R_L(\tau)$ controlled by the same $O(\zeta_0^{3/2}e^{3\tau/K}+\xi_0^{3/2})$ bound as in Theorem~\ref{thm:etaExpansion}.

In particular, to first order the dependence of the bad-mass drift on the within-block initialization $y(0),z(0)$ enters only through the two scalars $(\zeta_0,\xi_0)$. All finer details contribute only at order $O(\zeta_0^{3/2})$ and $O(\xi_0^{3/2})$ and higher.

\paragraph{Sign structure.}
Since $C_{\mathrm{geo}}(\tau)>0$ for all interior states, \eqref{eq:logit-internal} yields the global monotonicity:
\begin{itemize}[leftmargin=*,itemsep=0.2em]
\item If $J>0$, then $L$ decreases in $\tau$, hence $p(t)$ decreases monotonically toward $0$ (learning succeeds).
\item If $J<0$, then $L$ increases in $\tau$, hence $p(t)$ increases monotonically toward $1$ (anti-learning).
\item If $J=0$, the deterministic drift vanishes at this order (neutral evolution).
\end{itemize}
\end{theorem}
\end{ThmBoxCapital}

\subsection{KL Regularization: From Phase Transition to Interior Equilibrium}
\label{sec:kl-phase}

The introduction of a KL penalty toward a reference bad mass $p_{\mathrm{ref}}$ incorporates a restoring drift within the probability space. When enforcing the KL term across the two classes (the forward-KL format), the contribution to the dynamics is given by:
\[
\dot p\big|_{\mathrm{KL}}
=-\beta\,p(1-p)\Big(\log\tfrac{p}{1-p}-\log\tfrac{p_{\mathrm{ref}}}{1-p_{\mathrm{ref}}}\Big).
\]
Consequently, the KL-regularized bad-mass ODE takes the form:
\begin{equation}\label{eq:bad-ode-kl}
\dot{p}
=
-\,\eta\,\frac{J}{\sigma(p)}\,[p(1-p)]^2\,C(y,z)
\;-\;\beta\,p(1-p)\Big(\log\tfrac{p}{1-p}-\log\tfrac{p_{\mathrm{ref}}}{1-p_{\mathrm{ref}}}\Big).
\end{equation}
An analogous form holds for the full reverse-KL penalty, where the logit gap $\ell-\ell_{\mathrm{ref}}$ is replaced by $\ell-\ell_{\mathrm{ref}}-D_{\mathrm{KL}}(y\|y^{\mathrm{ref}})+D_{\mathrm{KL}}(z\|z^{\mathrm{ref}})$; a detailed derivation is provided in Appendix~\ref{sec:kl}.

An immediate consequence of \eqref{eq:bad-ode-kl} is that the regularized dynamics now support interior fixed points $p^\star \in (0,1)$. At these points, the reward-driven drift is exactly balanced by the KL anchoring:
\begin{equation}\label{eq:fp-balance}
\beta\big(L(p^\star)-L(p_{\mathrm{ref}})\big)
=
-\,\eta\,\frac{J}{\sigma(p^\star)}\,p^\star(1-p^\star)\,C(y,z),
\qquad
L(p):=\log\tfrac{p}{1-p}.
\end{equation}

\paragraph{Unique interior equilibrium.} 
For any $\beta>0$ and fixed $(y,z)$, the regularized dynamics admit a unique globally stable fixed point $p^\star \in (0,1)$, as established in Appendix~\ref{sec:kl}. The position of this equilibrium is fundamentally determined by the sign of the alignment $J$:
\[
\begin{aligned}
&\text{If } J>0: & 0 < p^\star < p_{\mathrm{ref}} \quad &\text{and} \quad p^\star \searrow 0 \text{ as } \beta \downarrow 0; \\
&\text{If } J=0: & p^\star = p_{\mathrm{ref}}; & \\
&\text{If } J<0: & p_{\mathrm{ref}} < p^\star < 1 \quad &\text{and} \quad p^\star \nearrow 1 \text{ as } \beta \downarrow 0.
\end{aligned}
\]
In this sense, KL regularization transforms the sharp phase transition at $J=0$, which previously resulted in boundary collapse to $p=0$ or $p=1$, into a smooth and stabilized interior equilibrium for any $\beta>0$.

\paragraph{Asymptotic regimes.}
The behavior of the equilibrium depends on the relative strength of the KL anchoring. In the \textbf{strong-KL regime} ($\beta \to \infty$), we can linearize \eqref{eq:fp-balance} around $p_{\mathrm{ref}}$ to find:
\[
p^\star
\approx
p_{\mathrm{ref}}
-\frac{\eta J}{\beta}\,
\frac{\big[p_{\mathrm{ref}}(1-p_{\mathrm{ref}})\big]^2}{\sigma(p_{\mathrm{ref}})}\,
C(y,z).
\]
Here, $p^\star$ converges to $p_{\mathrm{ref}}$ at a rate of $\mathcal{O}(\beta^{-1})$. Conversely, in the \textbf{weak-KL regime} ($\beta \downarrow 0$) with $J<0$, we let $\varepsilon := 1-p$ to obtain:
\[
1-p^\star
\;\sim\;
\frac{\beta}{c}\,\log\!\frac{c}{\beta},
\qquad
c:=\frac{-\,\eta\,J}{\sigma(1)}\,C(y,z) > 0.
\]
Notably, even an infinitesimal $\beta > 0$ is sufficient to prevent total collapse to $p^\star=1$, although the equilibrium may drift arbitrarily close to the bad vertex as $\beta$ vanishes. Complementary to our mean-field ODE view, prior work studies the noise-free GRPO update with KL anchoring and shows that the induced success-probability dynamics follow a fixed-point iteration whose limit depends explicitly on the KL strength \cite{mroueh2025reinforcement}.
 We refer the reader to Appendix \ref{sec:kl} for more details on the results presented here.

\begin{figure}
    \centering
    \includegraphics[width=1\linewidth]{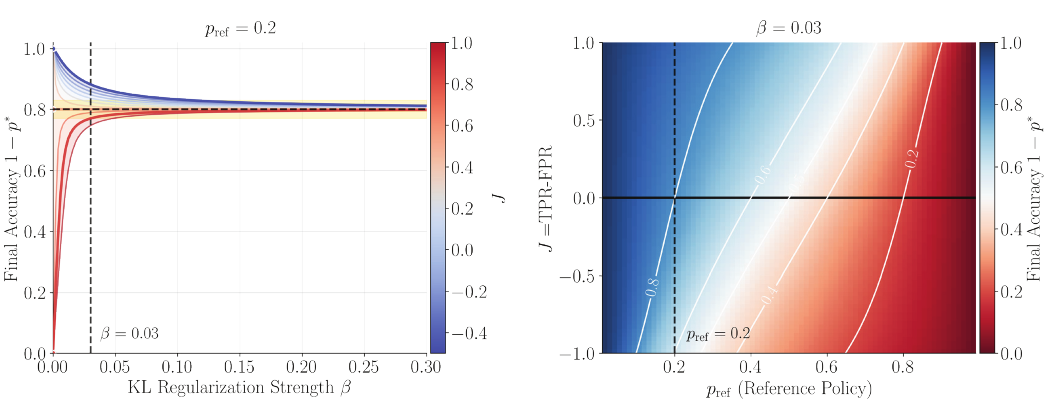}
    \caption{\textbf{KL regularization smooths the phase transition at $J=0$ via an interior fixed point.} 
    Incorporating a KL penalty toward a reference mass $p_{\mathrm{ref}}$ introduces a restoring drift, resulting in the regularized ODE \eqref{eq:bad-ode-kl}. For any $\beta>0$, the system possesses a unique stable interior equilibrium $p^\star \in (0,1)$ defined by the balance condition \eqref{eq:fp-balance}. This equilibrium satisfies $p^\star < p_{\mathrm{ref}}$ for $J>0$, $p^\star = p_{\mathrm{ref}}$ at $J=0$, and $p^\star > p_{\mathrm{ref}}$ for $J<0$. The reward-driven component is modulated by the multi-block collision factor $C(y,z) = \|y\|_2^2 + \|z\|_2^2 \in [1/K+1/M, 2]$. As $\beta \to \infty$, the equilibrium approaches $p_{\mathrm{ref}}$ due to strong anchoring; as $\beta \downarrow 0$, it approaches the reward-driven boundary, effectively smoothing the boundary collapse into a controlled interior state.}
    \label{fig:placeholder}
\end{figure}

\section{Experiments}
\label{sec:experiments}

We validate our theoretical analysis on programmatically verifiable Python coding tasks, testing whether real-world RL training exhibits the predicted phase transition at $J=0$, and whether noise merely rescales convergence speed without altering the basin of attraction. 

\subsection{Experimental Hypotheses}

Our analysis makes two sharp predictions:
\begin{itemize}[leftmargin=*,itemsep=0.3em]
    \item \textbf{($\mathcal{H}_1$) Phase Transition at $J=0$:}  Learning improves accuracy only when Youden's index $J=\mathrm{TPR}-\mathrm{FPR}$ is positive. At $J=0$, training yields no systematic improvement (neutral drift). For $J<0$, accuracy degrades as the system anti-learns.
    
    \item \textbf{($\mathcal{H}_2$) Rate, not Fate:} For any $J>0$, the sign of the reward signal determines the basin of attraction. Noise level affects only the convergence rate, more rollouts or training steps accelerate progress toward the same equilibrium without changing the asymptotic outcome.
\end{itemize}

\subsection{Setup}

\paragraph{Task and data.}
We use Python code generation with programmatic verification via unit tests. Our corpus is a filtered subset of high-quality problems from OpenR1~\cite{openr1} with $N_{\mathrm{train}}=10{,}239$ training prompts and $N_{\mathrm{val}}=594$ validation prompts. Each instance includes a natural-language specification, input/output examples, and a test harness combining public and hidden test cases.

\paragraph{Model and evaluation.}
We fine-tune \textbf{Qwen2.5-3B} as the base policy, evaluating $\mathbb{E}[\text{pass@1}]$ (the fraction of problems solved on the first attempt) averaged over five independent runs for each hyperparameter configuration.

\paragraph{Training algorithm.}
We employ standard \textbf{GRPO} with per-group advantage standardization and importance ratio clipping via the VeRL library~\cite{verl}. Each group contains $G=8$ rollouts per prompt, with returns normalized to zero mean and unit variance before computing advantages. We set the KL penalty to $\beta=0$ to isolate the pure reward-driven dynamics. Complete hyperparameters appear in Appendix~\ref{app:hparams}.

\paragraph{Synthetic verifier noise.}
Let $z\in\{0,1\}$ denote the \emph{true} correctness of a rollout under an oracle checker. The operational reward $r\in\{0,1\}$ is produced by a noisy checker characterized by
\[
\mathrm{TPR}=\Pr(r=1\mid z=1),
\qquad 
\mathrm{FPR}=\Pr(r=1\mid z=0),
\qquad 
J=\mathrm{TPR}-\mathrm{FPR}.
\]
We implement noise by independently flipping the oracle outcome with Bernoulli trials:
\[
r=\begin{cases}
1 \text{ with probability } \mathrm{TPR} &\text{if } z=1, \text{ else } 0\\
1 \text{ with probability } \mathrm{FPR} &\text{if } z=0, \text{ else } 0.
\end{cases}
\]
We explore a grid $\mathcal{J}\subset[-0.1,1]$ with multiple $(\mathrm{TPR},\mathrm{FPR})$ factorizations for each target $J$ to disentangle the effects of signal quality from error type prevalence.

\paragraph{Protocol.}
For each $J\in\mathcal{J}$, we train the base model for two epochs (1,410 gradient steps), logging metrics every 5 steps. Each configuration is run five times with different random seeds; we report mean and standard deviation of pass@1 on the validation set. All other hyperparameters remain fixed across noise conditions.

\paragraph{Baseline.}
Our primary baseline is GRPO with the noise-free oracle ($J=1$), representing the performance ceiling under perfect verification.

\subsection{Results}

\begin{table}[h]
  \centering
  \small
  \caption{\textbf{Validation Accuracy \& Noise Sensitivity.} Validation accuracy after two epochs across noise conditions. Improvement is visualized relative to the Baseline model performance. Bars to the left indicate degradation; bars to the right indicate gain.}
  \label{tab:noisegrid-main}
  
  \begin{tblr}{
    colspec = {
      Q[c,m] 
      Q[c,m] 
      S[table-format=2.2, table-space-text-post=\%] 
      S[table-format=-2.1, table-space-text-post=\%] 
      Q[r,m, wd=1.5cm] 
      Q[l,m, wd=1.5cm] 
    },
    row{1} = {bg=Navy, fg=white, font=\bfseries},
    cell{1}{5} = {c=2}{c}, 
    row{2} = {fg=Slate},
    vline{6} = {2-Z}{1pt, solid, AxisColor},
    hline{1,Z} = {1.2pt, Navy},
    hline{2} = {0.6pt, Navy},
    row{2-Z} = {3pt},
  }
    $J$ & $(\text{FPR, FNR})$ & {$\mathbf{E}$[Pass@1]} & {Improvement from the Base model} &  & \\
    
    $-0.1$ & (0.60, 0.50) & 0.16\% & \textcolor{Rose}{-12.6\%} 
           & {\color{Rose}\rule{1.2cm}{5pt}} & \\ 
    
    $0.0$  & (0.50, 0.50) & 13.4\% & \textcolor{Emerald}{+0.6\%} 
           & & {\color{Emerald}\rule{0.05cm}{5pt}} \\ 
    
    $0.3$  & (0.00, 0.70) & 16.0\% & \textcolor{Emerald}{+3.2\%} 
           & & {\color{Emerald}\rule{0.32cm}{5pt}} \\
    
    $0.3$  & (0.70, 0.00) & 14.6\% & \textcolor{Emerald}{+1.8\%} 
           & & {\color{Emerald}\rule{0.18cm}{5pt}} \\
    
    $0.7$  & (0.20, 0.10) & \textbf{18.6\%} & \textbf{\textcolor{Emerald}{+5.8\%}} 
           & & {\color{Emerald}\rule{0.58cm}{5pt}} \\
    
    $1.0$  & (0.00, 0.00) & \textbf{20.8\%} & \textbf{\textcolor{Emerald}{+8.0\%}} 
           & & {\color{Emerald}\rule{0.80cm}{5pt}} \\
           
  \end{tblr}
\end{table}

\paragraph{Phase transition confirmed ($\mathcal{H}_1$).}
Figure~\ref{fig:phase-transition-experimental-results} and Table~\ref{tab:noisegrid-main} confirms a sharp qualitative boundary at $J=0$. For $J>0$, all configurations show monotonic improvement in pass@1, with stronger signal ($J$ closer to 1) yielding faster convergence and higher final accuracy. At the critical point $J=0$, training produces minimal improvement ($+0.6$\%), consistent with neutral drift where reward noise cancels directional information. For $J<0$, accuracy actively degrades ($-12.6$\%), demonstrating anti-learning as predicted, the system systematically moves toward the bad equilibrium.

\paragraph{Noise rescales speed, not fate ($\mathcal{H}_2$).}
Figure~\ref{fig:phase-transition-experimental-results} shows learning trajectories across noise levels. For all $J>0$ conditions, the accuracy curves increase monotonically over the training horizon. Our experiments are limited to $1410$ steps (two epochs), so we remain agnostic about the exact asymptotic behavior; however, the observed trajectories are consistent with the hypothesis that both the noise–free and noisy regimes converge to the same basin of attraction.

The basin of attraction thus appears qualitatively unchanged: noisy signals with $J>0$ still drive the system toward the good equilibrium, only at a reduced velocity. This aligns with our theoretical prediction that noise rescales the multiplicative factor in $\dot{p}$ but preserves the sign structure that governs the long–term dynamics.

Notably, even heavily degraded signals ($J=0.3$) still enable learning, though convergence is substantially slower. The asymmetry between false positives and false negatives matters: at fixed $J=0.3$, the configuration (FPR=0.00, FNR=0.70) achieves 15.98\% while (FPR=0.70, FNR=0.00) reaches 14.64\%, suggesting FNs are more tolerable than FPs in this regime. This supports the theoretical investigation that convergence rate is $O(t^{-2})$ if $\text{FN}=0$ vs. $O(t^{-1})$  for $\text{FN}>0$.


\subsection{Limitations and Future Directions}

\paragraph{Oracle imperfection.}
Our verification relies on a finite test suite. While designed for high coverage, incomplete tests introduce systematic bias in estimated $(\mathrm{TPR},\mathrm{FPR})$, particularly for edge cases. 

\paragraph{Context length effects.}
As model performance declines (especially for $J<0$), there is potential to increase the generation of longer responses that exceed the maximum token limit. VeRL's handling of truncated rollouts, assigning them reward zero and high clipping ratios, introduces systematic false negatives that can shift the effective $J$ downward. This may explain some of the asymmetry between predicted and observed behavior in the anti-learning regime.

\paragraph{Generalization.}
Our experiments focus on Python coding with Qwen2.5-3B. The phase transition at $J=0$ is a fundamental property of the learning dynamics and should generalize across domains and architectures, but the specific decay rates and noise tolerance may vary with task complexity, model capacity, and verifier characteristics. Extensions to mathematical reasoning, creative writing with LLM-as-Judge, and larger models remain important future directions.

\paragraph{Time-dependent noise.}
Although our ODE framework is generalizable to arbitrary time-dependent noise, our experiments employ fixed noise rates. In practical settings, the simultaneous evolution of policy and reward models may induce drift in $\mathrm{TPR}(t)$ and $\mathrm{FPR}(t)$. While our mean-field analysis theoretically supports time-dependency, a full investigation of these co-evolutionary dynamics is reserved for future work.


\section{Conclusion}\label{sec:conclusion}

We asked a simple but operational question: \emph{how much slop is too much} in RLVR, when the verifier is imperfect and the learning algorithm repeatedly amplifies its feedback.
Our analysis shows that, for group-normalized policy-gradient methods (e.g., GRPO), the qualitative outcome is governed by a single scalar:
\[
J \;=\; \mathrm{TPR}-\mathrm{FPR}.
\]
When $J>0$, the verifier is net-informative and RL pushes probability mass toward correct solutions; when $J=0$, the signal is effectively chance-level and learning becomes neutral drift; when $J<0$, the signal is net-misleading and the updates become \emph{anti-learning}, systematically driving the policy toward incorrect modes.

\begin{TakeawayDrop}{Main finding}
\textbf{Group-normalized RL is directionally consistent under noisy rewards whenever $J=\mathrm{TPR}-\mathrm{FPR}>0$.}
In that regime, the aggregate bad-mode mass decreases monotonically and accuracy improves; noise primarily reduces the \emph{speed} of progress rather than changing the eventual basin of attraction (``rate, not fate'').
When $J<0$, the direction flips and performance collapses.
\end{TakeawayDrop}

\paragraph{What this paper contributed.}
Beyond identifying the $J=0$ boundary, we developed a minimal and predictive framework for noisy-reward RLVR:
\begin{itemize}[leftmargin=1.2em,itemsep=0.25em]
    \item \textbf{A multi-armed bandit abstraction for LLM completions.}
    We coarse-grain sequences into recurring \textbf{reasoning modes} (arms), making sequence-level RLVR analytically tractable.

    \item \textbf{A mean-field probability-simplex view of GRPO.}
    Group normalization induces a replicator-style (\textbf{natural-selection}) flow that redistributes probability mass based on relative advantage.

    \item \textbf{A closed-form governing variable: the bad-mode mass $p(t)$.}
    The dynamics decouple into (i) an \textbf{outer} evolution of $p(t)$ (good vs.\ bad families) and (ii) an \textbf{inner} competition within each family.
    This yields an explicit drift whose sign depends only on $\operatorname{sign}(J)$, producing the observed \textbf{phase transition} at $J=0$.

    \item \textbf{Rate laws and learnability insights.}
    In the learning regime ($J>0$), noise changes convergence rates (e.g., $t^{-1}$ vs.\ $t^{-2}$ tails depending on variance degeneracy) and predicts that prompts are most learnable at \textbf{intermediate difficulty} (roughly when the model is near $50$--$50$ between good and bad).

    \item \textbf{Geometry and diversity implications.}
    The simplex geometry (Shahshahani/Fisher) clarifies why GRPO produces \textbf{winner-take-all} behavior inside the good manifold (symmetry breaking / diversity collapse), even when multiple correct modes exist.

    \item \textbf{Where practical GRPO details enter.}
    PPO-style importance sampling and clipping do not change the leading-order drift in the small-step mean-field regime, while \textbf{KL anchoring} turns boundary collapse into a \textbf{unique interior equilibrium}, smoothing the sharpness of the transition without eliminating the fundamental dependence on $J$.
\end{itemize}

\begin{TakeawayDrop}{Practical takeaways for RLVR with noisy verifiers}
\begin{itemize}[leftmargin=0em,itemsep=0.25em]
    \item \textbf{Measure (or estimate) $J=\mathrm{TPR}-\mathrm{FPR}$ early.}
    If $J\le 0$, scaling RL compute will not fix the problem, it will stagnate or actively degrade the policy.

    \item \textbf{If $J>0$, compute helps (mostly) by buying time.}
    Noisy-but-informative rewards tend to slow training rather than changing its qualitative endpoint.

    \item \textbf{False positives are especially dangerous.}
    Holding $J$ fixed, the noise \emph{structure} can change speed; empirically, high FPR is often more damaging than high FNR.

    \item \textbf{Use KL regularization for stability, not as a substitute for signal quality.}
    KL anchoring can prevent extreme collapse and yields controlled interior behavior, but it cannot turn a net-misleading verifier into a learning signal.
\end{itemize}
\end{TakeawayDrop}

\paragraph{Closing perspective.}
Overall, \textbf{RLV$^\varepsilon$R} provides a simple analytic lens for understanding when RLVR is viable under imperfect verification: the verifier’s \emph{net discriminative power} (captured by $J$) determines the fate, while algorithmic details and noise structure shape the rate and stability.
This gives both a concrete diagnostic for verifier quality and a principled foundation for designing more robust RLVR pipelines in domains where clean ground-truth supervision is unavailable.




\bibliography{colm2024_conference}
\bibliographystyle{colm2024_conference}

\newpage
\appendix
\newpage

\section{LLM as Multi-arm Bandit}\label{app:MAB}

 \subsection{Multi-armed bandits.}\label{subsec:Multi-armed bandits}
The multi-armed bandit (MAB) problem is a  model in optimization and probability that focuses on the exploration–exploitation trade-off. In this problem setup, a decision maker repeatedly selects one of \(K\) actions (``arms''); upon pulling arm \(a_t\in[K]\) at round \(t\), a stochastic reward \(R_t(a_t)\) is observed, drawn from an unknown distribution with mean \(\mu_{a_t}\). The goal is to maximize cumulative reward or equivalently minimize regret, despite this uncertainty:
\[
R_T \;=\; T\,\mu_* \;-\; \sum_{t=1}^T \mu_{a_t},
\qquad
\mu_* \;:=\; \max_{a\in[K]} \mu_a.
\]
This setting captures a wide range of real systems where feedback is noisy, delayed, or partial: online recommendation, A/B testing, adaptive science experiments, and (as emphasized in this work) coarse-grained evaluation of generative models. Classical algorithms balance information gathering with reward maximization (e.g., optimism/UCB, posterior sampling, or gradient-based updates), and their guarantees hinge on the number of arms, reward signal quality, and the horizon \(T\). In our context, the bandit abstraction serves as a tractable surrogate for complex, high-dimensional decision spaces while preserving the essential statistical structure of learning under uncertainty (\cite{lattimore2020bandit}. 
\subsection{Bandit Abstraction for LLMs}

In the context of generative AI, specifically large language models, a given problem (such as a request or prompt) can yield multiple potential solutions, particularly when these models operate with a non-zero temperature setting. Recall that the temperature parameter influences output of the final layer of the model, where it directly affects the selection of subsequent tokens from the logit vectors using the softmax mechanism adjusted by the specified temperature. This selection process at the token level results in the generation of various sequences, some of which are correct (if the domain is verifiable), while others may be incorrect. Although the space of possible sequences that an LLM can generate is theoretically infinite akin to the hypothetical scenario of a monkey randomly typing and eventually producing a proof of the Goldbach conjecture, the total number of answers is finite due to the maximum response length that is feasible for the model to generate these answers.

For a fixed prompt \(x\), an LLM samples a completion \(y\) from \(\pi_{\omega}(y\mid x)\). With nonzero temperature, the raw support over \emph{all} token sequences can be very large (in principle, unbounded). In practice, inference and training impose a maximum generation length \(L_{\max}\) (e.g., \texttt{max\_new\_tokens}) and an end-of-sequence token \(\langle\mathrm{eos}\rangle\). Let \(\mathcal{V}\) denote the finite vocabulary. The admissible completions are then drawn from the truncated set
\[
\mathcal{Y}_{\le L_{\max}} \;=\; \bigcup_{\ell=1}^{L_{\max}} \mathcal{V}^{\ell},
\qquad
\bigl|\mathcal{Y}_{\le L_{\max}}\bigr|
\;\le\; \sum_{\ell=1}^{L_{\max}} |\mathcal{V}|^{\ell}
\;=\; \frac{|\mathcal{V}|^{L_{\max}+1}-|\mathcal{V}|}{|\mathcal{V}|-1},
\]
so the effective support is \emph{finite}. (Stop-sequences and \(\langle\mathrm{eos}\rangle\) further reduce this set in practice.) 
Given a fixed prompt \(x\), a large language model (LLM) samples an output sequence \(y\) from its conditional policy \(\pi_{\omega}(y\,|\,x)\) (with base parameters \(\omega\)). In case of truncation, we can write the truncated policy as
\[
\pi_{\omega}^{(L)}(y\mid x) \;\propto\; \pi_{\omega}(y\mid x)\,\mathbf{1}\{y\in\mathcal{Y}_{\le L_{\max}}\}.
\]

\subsection{Coarse-graining into Modes}
 For a nonzero sampling temperature, the model typically admits many distinct answers to the same prompt, often spanning a very large (potentially infinite) support.  However, in the practice, due  to the limitation on  the output length (controlled by max tokens), the space of possible solution is practically   \emph{coarsen} into a finite collection of representatives \emph{reasoning modes} (or solution prototypes).
 By clustering the reasoning-equivalent response together as a one reasoning mode/arm, we can map outputs via a surjective map
\[
\phi:\;\mathcal Y_{\le L_{\max}} \longrightarrow 
\mathcal{H}\;=\;\{h_1,\dots,h_{K+M}\},
\]
where each mode \(h\in\mathcal H\) represents a literal or semantic/evaluative equivalence class (e.g., logically equivalent answers, rubric-equivalent or string matching equivalency).

In the next step, we can partition the modes into \emph{good} (correct) and \emph{bad} (incorrect) solutions,
\[
\mathcal{H}\;=\;\mathcal{H}^{+}\ \cup\ \mathcal{H}^{-},\qquad
|\mathcal{H}^{+}|=K,\ \ |\mathcal{H}^{-}|=M.
\] 
Without loss of generality, we index good modes by \(i\in\{1,\dots,K\}\) and bad modes by \(i\in\{K+1,\dots,K+M\}\). Sampling a response is now equivalent to pulling one arm from a \((K\!+\!M)\)-armed bandit with pull probabilities \(\pi_{\theta}(h_i\,|\,x)\). We then work with the induced categorical policy over modes,
\[
\pi_{\theta}(h_i\,|\,x)\;=\;\frac{\exp(\theta_i)}{\sum_{j=1}^{K+M} \exp(\theta_j)}
\;=\;\mathrm{softmax}(\theta)_i,
\]
where \(\theta=(\theta_1,\dots,\theta_{K+M})\) are \emph{effective logits} that summarize, for the fixed prompt \(x\), the aggregate probability mass the base model places on each mode. These logits are not a one-to-one reparameterization of \(\omega\); rather, they are low-dimensional coordinates (unique up to an additive constant) on the probability simplex over \(\mathcal{H}\) induced by \(\pi_{\omega}(\cdot\,|\,x)\).

In this work, since we are interested mostly in the total probability of bad arms, as we discussed in the $\S \ref{sec:llm-bandit}$, we define the bad arms mass probability by partitioning modes into good (correct) and bad (incorrect), 
\(
\mathcal{H}=\mathcal{H}^+\cup\mathcal{H}^-,
\ |\mathcal{H}^+|=K,\ |\mathcal{H}^-|=M,
\)
such that $p \;=\;\sum_{h\in\mathcal{H}^-}\pi_\theta(h\mid x)$

\newpage
\section{Noisy Rewards and Youden’s $J$ Index }
\label{appendix:retrieval-noisy-rewards}
\begin{figure}
    \centering
    \includegraphics[width=1\linewidth]{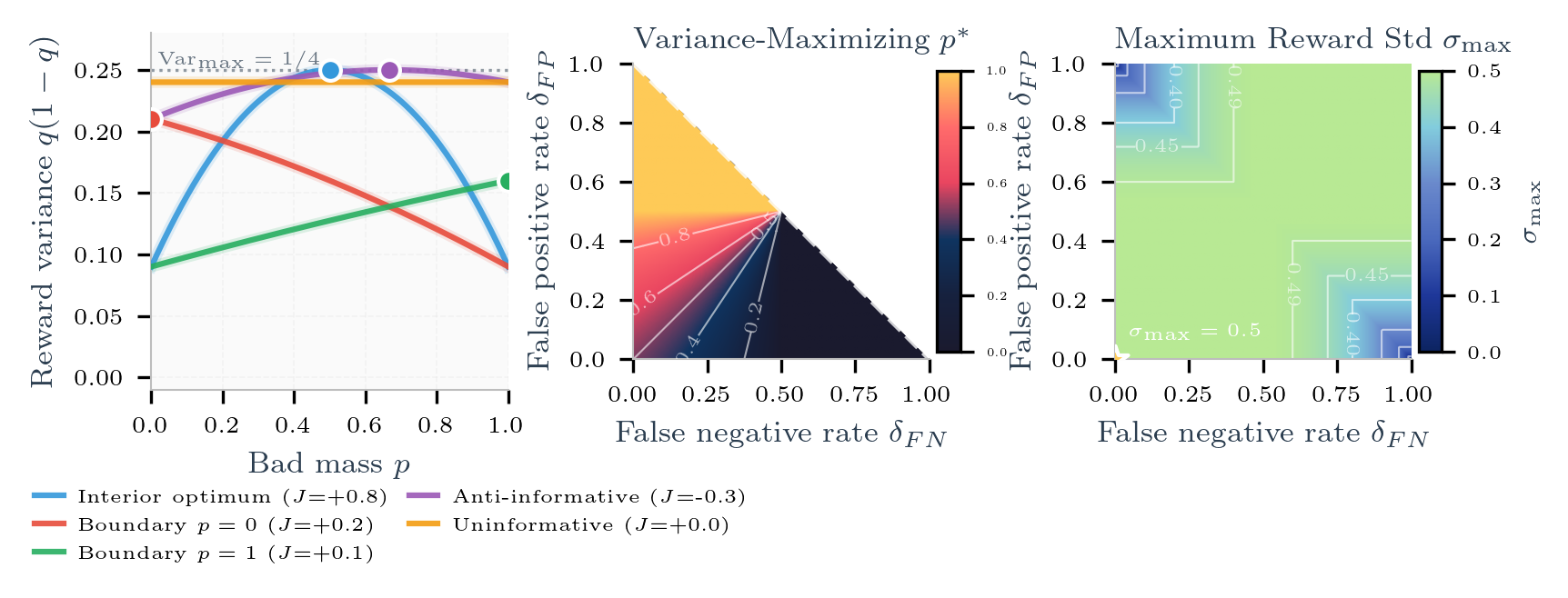}
    \caption{\textbf{Reward-variance geometry under noisy Bernoulli rewards.}
(Left) Reward variance $\Var(r)=q(p)\bigl(1-q(p)\bigr)$ as a function of bad mass $p$ for representative noise settings; markers indicate
$p^\star=\arg\max_{p\in[0,1]}\Var(r)$ (equivalently $q(p)=\tfrac12$ when attainable, otherwise the boundary $p\in\{0,1\}$).
(Middle) Heatmap of $p^\star(\delta_{\mathrm{FN}},\delta_{\mathrm{FP}})$ in the informative region $J>0$, with the dashed diagonal marking
the phase boundary $J=1-\delta_{\mathrm{FN}}-\delta_{\mathrm{FP}}=0$ and contours showing level sets of $p^\star$.
(Right) Maximum achievable reward standard deviation $\sigma_{\max}(\delta_{\mathrm{FN}},\delta_{\mathrm{FP}})
=\max_{p\in[0,1]}\sqrt{q(p)(1-q(p))}$ with contours. Throughout,
$q(p)=(1-\delta_{\mathrm{FN}})-Jp$ and $J=1-\delta_{\mathrm{FN}}-\delta_{\mathrm{FP}}$.}
\label{fig:variance_triptych}
\end{figure}



Recall the definition of noise that we had in \eqref{eq:def-deltas}:
\begin{equation*}
\delta_{\mathrm{FN}}=\Pr(r=0\mid \text{good}),\qquad
\delta_{\mathrm{FP}}=\Pr(r=1\mid \text{bad}),
\end{equation*}
and the Youden's Index, \eqref{eq:def-J}, as 
\begin{equation*}
J := 1-\delta_{\mathrm{FN}}-\delta_{\mathrm{FP}}
\;=\;\mathrm{TPR}-\mathrm{FPR}\in[-1,1].
\end{equation*}
where \(p=\Pr(\text{bad})\) denote the current bad mass (so \(\Pr(\text{good})=1-p\)). With this setup, the expected reward of a single pull is
\begin{equation}
\label{eq:q-definition}
\begin{aligned}
q(p)\;:=\;\mathbb{E}[r]
&=\mathbb{E}[r\mid \text{good}]\,(1-p)+\mathbb{E}[r\mid \text{bad}]\,p\\
&=(1-p)\,(1-\delta_{\mathrm{FN}})+p\,\delta_{\mathrm{FP}}
\;=\;(1-\delta_{\mathrm{FN}})\;-\;J\,p.
\end{aligned}
\end{equation}
Since \(r\) is Bernoulli with mean \(q\), its variance is
\begin{equation}
\label{eq:sigma-def}
\sigma^2(p)\;:=\;\mathrm{Var}(r)\;=\;q(p)\,\bigl(1-q(p)\bigr).
\end{equation}

we can also directly verify this property: 
\begin{equation}\label{eq:sigma-def}
\begin{aligned}
\mathrm{Var}(r)
&= (1-p)(1-\delta_{\mathrm{FN}})\,\delta_{\mathrm{FN}}
 \;+\; p\,\delta_{\mathrm{FP}}(1-\delta_{\mathrm{FP}})
 \;+\; p(1-p)\,J^2 \\
&=\big((1-\delta_{\mathrm{FN}})-Jp\big)\big(\delta_{\mathrm{FN}}+Jp\big)
\;=\; q(1-q).
\end{aligned}
\end{equation}
Notice that $q, \sigma(p)\equiv q,\sigma(\delta_{\text{FN}}, \delta_{FP},p)$, but for  brevity, we denote it as $\sigma(p)$ and $q(p)$.

It is good to notice that since \(q(p) \in [\delta_{\mathrm{FP}}, 1 - \delta_{\mathrm{FN}}]\), the variance term 
\(\sigma(p) = \sqrt{q(p)\bigl(1 - q(p)\bigr)}\) is maximized at \(q(p) = \tfrac{1}{2}\), 
provided \(\tfrac{1}{2}\) lies inside this interval. Solving \(q(p) = \tfrac{1}{2}\) for \(p\) yields
\[
p^\star=\text{clip}(\frac{\frac{1}{2}-\delta_{\text{FN}}}{J},0,1)
\]
such that the value of \(p\) that maximizes \(\sigma(p)\), assuming \(J = 1 - \delta_{\mathrm{FN}} - \delta_{\mathrm{FP}} > 0\). In the edge case where \(\tfrac{1}{2} \notin [\delta_{\mathrm{FP}}, 1 - \delta_{\mathrm{FN}}]\) (i.e., for extremely noisy graders),
the maximum of \(\sigma(p)\) occurs at the boundary: \(q = \delta_{\mathrm{FP}}\) or \(q = 1 - \delta_{\mathrm{FN}}\), 
whichever is closer to \(\tfrac{1}{2}\). Equivalently, \(p^\star\) clips to \(0\) or \(1\) in this regime (see Fig.~\ref{fig:variance_triptych}). 


Group-based policy typically updates normalized rewards within a prompt-specific group of \(G\) rollouts. While alternatives exist (e.g., leave-one-out baselines~\cite{kool2019buy} or centered-but-unstandardized variants~\cite{liu2025understanding}), we adopt a simple \(z\)-score normalization (as in GRPO~\cite{shao2024deepseekmath}):
\begin{equation}
\label{eq:zscore}
\tilde r \;=\; \frac{r - q(p)}{\sigma(p)}.
\end{equation}
Conditioning on the latent correctness, this yields
\[
\tilde r\mid\text{good}=
\begin{cases}
\dfrac{1-q}{\sigma}, & \text{w.p. } 1-\delta_{\mathrm{FN}},\\[4pt]
\dfrac{-q}{\sigma}, & \text{w.p. } \delta_{\mathrm{FN}},
\end{cases}
\qquad
\tilde r\mid\text{bad} =
\begin{cases}
\dfrac{1-q}{\sigma}, & \text{w.p. } \delta_{\mathrm{FP}},\\[4pt]
\dfrac{-q}{\sigma}, & \text{w.p. } 1-\delta_{\mathrm{FP}}.
\end{cases}
\]

Taking expectations gives the block-symmetric conditional means
\begin{equation}
\label{eq:cond-means-tilder}
\mathbb{E}[\tilde r\mid\text{good}] \;=\; \frac{J\,p}{\sigma(p)},
\qquad
\mathbb{E}[\tilde r\mid\text{bad}]  \;=\; -\,\frac{J\,(1-p)}{\sigma(p)},
\end{equation}
and global centering holds automatically:
\begin{equation}
\label{eq:global-centering}
\mathbb{E}[\tilde r]
=(1-p)\,\mathbb{E}[\tilde r\mid\text{good}] + p\,\mathbb{E}[\tilde r\mid\text{bad}]
=0,
\end{equation}
which is desirable for stable, scale-invariant updates. These expressions demonstrate that Youden’s index \(J\) governs the sign and magnitude of the expected normalized reward for good versus bad arms.

\begin{remark}
If one omits division by \(\sigma(p)\) (the ``centered-only’’ modification of GRPO~\cite{liu2025understanding}), then
\begin{equation}
\label{eq:centered-only}
\mathbb{E}[\tilde r\mid\text{good}] \;=\; J\,p,\qquad
\mathbb{E}[\tilde r\mid\text{bad}] \;=\; -\,J\,(1-p),\qquad
\mathbb{E}[\tilde r]=0.
\end{equation}
\end{remark}

\begin{remark}
Some works use a \(\{\pm 1\}\)-valued reward \(S:=2r-1\). Then \(\mathbb{E}[S]=2q-1\) and \(\mathrm{Var}(S)=4\,q(1-q)\).
Equations~\eqref{eq:zscore}–\eqref{eq:cond-means-tilder} map to this reward by the linear rescaling \(S=2r-1\); the resulting normalization differs only by a constant factor of \(2\).
\end{remark}

\begin{remark}
    For noise-free case, $J=1$, the expectation values take a simpler form
\begin{equation*}
\label{eq:cond-means-tilder-2}
\mathbb{E}[\tilde r\mid\text{good}] \;=\; \sqrt{\frac{p}{1-p}},
\qquad
\mathbb{E}[\tilde r\mid\text{bad}]  \;=\; \sqrt{\frac{1-p}{p}},\qquad \mathbb{E}[\tilde r\mid\text{good}]-\mathbb{E}[\tilde r\mid\text{bad}]=\frac{1}{\sqrt{p(1-p)}}.
\end{equation*}
\end{remark}

\newpage


\section{Mean Field Dynamics..}
\label{appendix:Group-Normalized RL with Noisy Reward}

In this section, we analyze the evolution of good and bad arms using a mean-field approximation of the multi-armed bandit (MAB) system (see Appendix~\ref{app:MAB}). Our derivation accounts for a general noisy environment where $J \in [-1,1]$. The noise-free scenario is treated as a specialization of this framework; specifically, by setting $J=1$, we recover the standard clean-reward dynamics without requiring further modification.

\subsection{Dynamics of the Bad Arms}
Consider a $(K{+}M)$-arm bandit comprising $K$ \emph{good} arms and $M$ \emph{bad} arms $\{b_1,\ldots,b_M\}$. Let the policy be defined as $\mathbf p=\mathrm{softmax}(\boldsymbol\theta)\in\Delta^{K+M-1}$ where:
\[
\mathbf p=(p_1,\ldots,p_K,\ p_{b_1},\ldots,p_{b_M}),
\qquad
p:=\sum_{m=1}^M p_{b_m}\in[0,1],
\qquad
\alpha:=1-p,
\]
and we define the normalized within-block coordinates:
\[
p_j=\alpha\,y_j\ \ (j\le K),\quad
y=(y_1,\ldots,y_K)\in\Delta^{K-1},
\qquad
p_{b_m}=p\,z_m\ \ (m\le M),\quad
z=(z_1,\ldots,z_M)\in\Delta^{M-1}.
\]
Equivalently, the probability vector factors as follows:
\[
\mathbf p
=
\begin{bmatrix}
\alpha\,y\\[2pt]
p\,z
\end{bmatrix},
\qquad
\|y\|_2^2=\sum_{j=1}^K y_j^2\in\Big[\tfrac1K,1\Big],
\qquad
\|z\|_2^2=\sum_{m=1}^M z_m^2\in\Big[\tfrac1M,1\Big].
\]

Let the conditional expected advantage reward be denoted by $A_i \ :=\ \mathbb E[\tilde r\mid I=i]$, with $\mathbf A=\big(A_1,\ldots,A_K,A_{b_1},\ldots,A_{b_M}\big)$ and the mean advantage $\bar A:=\langle \mathbf p,\mathbf A\rangle=\sum_i p_i A_i$. We assume \emph{block symmetry} within both blocks, a condition that arises naturally in the binary-reward scenario analyzed in Appendix~\ref{appendix:retrieval-noisy-rewards} (refer to \eqref{eq:cond-means-tilder}):
\begin{equation}\label{eq:ab-ag}
A_j=a_{\mathrm g}(p)\ \ (j\le K),\qquad
A_{b_m}=a_{\mathrm b}(p)\ \ (m\le M),\qquad
\Delta r(p):=a_{\mathrm b}(p)-a_{\mathrm g}(p).
\end{equation}
It follows that $\bar A=\alpha\,a_{\mathrm g}(p)+p\,a_{\mathrm b}(p)$, which implies $a_{\mathrm g}(p)-\bar A=-p\,\Delta r(p)$ and $a_{\mathrm b}(p)-\bar A=\alpha\,\Delta r(p)$.

\begin{proposition}[Expected directions in \({\boldsymbol{\theta}}\) and \(\mathbf{p}\) space]\label{prop:expectations}
Given \(\mathbf p=\mathrm{softmax}(\boldsymbol{\theta})\) and the softmax Jacobian \(\mathfrak{J}(\mathbf{p}):=\mathrm{Diag}(\mathbf{p})-\mathbf{p}\mathbf{p}^\top\), for a step size $\eta$ and group size $G$:
\[
\E[\Delta\boldsymbol{\theta}\mid \p] \;=\; \eta\,\mathfrak{J}(\p)\,\mathbf{A},
\qquad
\E[\Delta \p\mid \p] \;\approx\; \mathfrak{J}(\p)\,\E[\Delta\boldsymbol{\theta}\mid \p]
\;=\; \eta\,\mathfrak{J}(\p)^2 \mathbf{A},
\]
to the first order in \(\Delta \boldsymbol{\theta}\).
\end{proposition}

\emph{Proof Sketch}. Let \(e_i\) denote the \(i\)-th standard basis vector. Define \(\pi_\theta(i)=p_i=\exp(\theta_i)/\sum_k \exp(\theta_k)\) as the softmax policy. For a realized arm \(I\), the gradient is:
\[
\nabla_{\boldsymbol\theta}\log \pi_\theta(I)=e_I-\mathbf p.
\]
The REINFORCE estimator (\cite{williams1992simple}) for \(\nabla_{\boldsymbol\theta}\E[r]\) is $g = \tilde r\,(e_I-\mathbf p)$. Taking the conditional expectation given \(\mathbf p\) yields:
\[
\E[g\mid \mathbf p]
=\sum_{i} p_i\,\E[\tilde r\mid I=i]\,(e_i-\mathbf p)
=\bigl(\mathrm{Diag}(\mathbf p)-\mathbf p\mathbf p^\top\bigr)\,\mathbf A
=\mathfrak{J}(\mathbf p)\,\mathbf A.
\]
This confirms the stated form and the first identity in Proposition~\ref{prop:expectations}. For the second part, refer to Lemma~\ref{lem:softmax-tangent}.

\begin{remark}[The Importance of Coupling Terms]
Retaining the full Jacobian, including the rank-one term $\mathbf{p}\mathbf{p}^{\top}$, is essential because it couples all arms through collision terms. Specifically, the total bad-mass drift depends on the collisions within both the good and bad blocks via $\|y\|_2^2$ and $\|z\|_2^2$ (see \eqref{aux3}). Omitting the $\mathbf p\mathbf p^\top$ term would spuriously decouple the blocks and result in incomplete mean-field dynamics.
\end{remark}

\begin{corollary}[First-order softmax pushforward]
\label{cor:pushforward}
For a small logit increment $\Delta\boldsymbol\theta$:
\[
\Delta \mathbf p \;=\; \mathfrak J(\mathbf p)\,\Delta\boldsymbol\theta
\;=\; \operatorname{Diag}(\mathbf p)\,\Delta\boldsymbol\theta \;-\; \mathbf p\,\mu,
\qquad
\mu:=\langle \mathbf p,\Delta\boldsymbol\theta\rangle=\sum_k p_k\,\Delta\theta_k,
\]
which implies $\Delta p_i = p_i\big(\Delta\theta_i-\mu\big)$.
\end{corollary}

Applying Proposition~\ref{prop:expectations} and the relation for $\bar A$, we obtain the following expectations (where conditioning on $\mathbf p$ is suppressed for brevity):
\begin{align}
\label{eq:theta-good-j}
\mathbb E[\Delta\theta_j]
&=-\,\eta\,p(1-p)\,\Delta r(p)\;y_j,
\qquad j=1,\ldots,K,\\
\label{eq:theta-bad}
\mathbb E[\Delta\theta_{b_m}]
&=\eta\,p(1-p)\,\Delta r(p)\;z_m,
\qquad m=1,\ldots,M.
\end{align}
The expected step therefore follows the block-form direction:
\begin{equation}\label{aux1}
\E[\Delta\boldsymbol{\theta}]\;=\;\eta\,p(1-p)\,\Delta r(p)\;
\begin{bmatrix}
-\,y\\[2pt] z
\end{bmatrix}.
\end{equation}

Since $\mathfrak J(\mathbf p)\mathbf 1=0$, the update is centered:
\begin{equation}
\label{eq:theta-sum-zero}
\sum_{i}\mathbb E[\Delta\theta_i]=\eta\,\mathbf 1^\top \mathfrak J(\mathbf p)\mathbf A=0.
\end{equation}
Moreover, within each block the logit increment is collinear with the current within-block distribution:
\[
\E[\Delta\theta_j]-y_j\sum_{k=1}^K \E[\Delta\theta_k]=0,
\qquad
\E[\Delta\theta_{b_m}]-z_m\sum_{\ell=1}^M \E[\Delta\theta_{b_\ell}]=0.
\]
In other words, there is no arm-specific drift \emph{within} a block in logit space; arms move in lockstep
proportional to $y$ (good block) and $z$ (bad block).

Following \eqref{eq:cond-means-tilder}, the expected advantages relative to the noise level $J=\mathrm{TPR}-\mathrm{FPR}$ are expressed as:
\begin{equation}\label{eq:ag-ab-delta_r}
a_{\mathrm g}(p)=\frac{J\,p}{\sigma(p)},
\qquad
a_{\mathrm b}(p)=-\,\frac{J\,(1-p)}{\sigma(p)}, \qquad \Delta r(p)=-\dfrac{J}{\sigma(p)}.
\end{equation}

\paragraph{Total Bad-Mass Drift}
By Corollary~\ref{cor:pushforward}, the softmax-centering scalar $\mu$ becomes:
\begin{equation}
\label{aux2}
\mu
=\eta\,p(1-p)\,\Delta r(p)\,
\Big(p\,\|z\|_2^2-(1-p)\,\|y\|_2^2\Big).
\end{equation}
Summing the bad components provides the total bad-mass drift:
\begin{equation}\label{aux3}
\mathbb E[\Delta p]
=
-\eta\,\frac{J}{\sigma(p)}\,[p(1-p)]^2\,\Big(\|y\|_2^2+\|z\|_2^2\Big).
\end{equation}
In the case where $M=1$, then $z=(1)$ and $\|z\|_2^2=1$, which recovers the $(K{+}1)$ formula.

\paragraph{Within-Bad Dynamics in Normalized Coordinates}
Using the identity $\Delta z_m=\frac{1}{p}(\Delta p_{b_m}-z_m\,\Delta p)$ and substituting the first-order drift, we find:
\begin{equation}
\label{eq:Edzm}
\mathbb E[\Delta z_m]
=
-\eta\,\frac{J}{\sigma(p)}\,p(1-p)\;z_m\Big(z_m-\|z\|_2^2\Big),
\qquad m=1,\ldots,M.
\end{equation}
In vector form, this is expressed as $\mathbb E[\Delta z] = \eta\,p(1-p)\,\Delta r(p)\,\Big(z\odot z-\|z\|_2^2\,z\Big)$. Consequently, for an informative grader ($J>0$), the bad-block dynamics exhibit the opposite sign of the good-block collision field, tending to spread bad mass toward a uniform distribution on $\Delta^{M-1}$.

\subsection{Dynamics of the Good Arms}

Regarding the good arms, the combination of Corollary~\ref{cor:pushforward} with \eqref{eq:theta-good-j} and \eqref{aux2} provides the probability increments for the good block:
\begin{equation}
\label{eq:good-prob-vector}
\mathbb E[\Delta \mathbf p_{\mathrm{good}}]
\;=\;
-\,\eta\,p(1-p)^2\,\Delta r(p)\;
\Big( y\odot y \;+\; \big[p\,\|z\|_2^2-(1-p)\|y\|_2^2\big]\,y \Big).
\end{equation}
In componentwise form, substituting $p_j=(1-p)y_j$, we obtain:
\begin{align}
\label{eq:good-prob-j}
\mathbb E[\Delta p_j]
&=(1-p)\,y_j\Big(\mathbb E[\Delta\theta_j]-\mu\Big) \notag\\
&=-\,\eta\,p(1-p)^2\,\Delta r(p)\;y_j\Big(y_j+p\,\|z\|_2^2-(1-p)\,\|y\|_2^2\Big),
\qquad j=1,\ldots,K.
\end{align}

Summing \eqref{eq:good-prob-j} over all $j$ and applying the constraint $\sum_j y_j=1$ yields:
\[
\sum_{j=1}^K \mathbb E[\Delta p_j]
\;=\;-\,\eta\,[p(1-p)]^2\,\Delta r(p)\,\big(\|y\|_2^2+\|z\|_2^2\big)
\;=\;-\,\mathbb E[\Delta p].
\]
Consequently, the drift for the total good mass $\alpha:=1-p$ is given by:
\begin{equation}
\label{eq:alpha-drift}
\mathbb E[\Delta \alpha]\;=\;-\mathbb E[\Delta p]
\;=\;-\,\eta\,[p(1-p)]^2\,\Delta r(p)\,\big(\|y\|_2^2+\|z\|_2^2\big).
\end{equation}

By utilizing the relationship $y_j=p_j/\alpha$, we can apply the exact identity:
\begin{equation}
\label{eq:Dy-identity}
\Delta y_j=\frac{1}{\alpha}\Big(\Delta p_j-y_j\,\Delta \alpha\Big).
\end{equation}
Substituting \eqref{eq:good-prob-j} through \eqref{eq:alpha-drift} into \eqref{eq:Dy-identity} and simplifying leads to the within-good drift:
\begin{equation}
\label{eq:Edyj}
\mathbb E[\Delta y_j]
\;=\;-\,\eta\,p\,\alpha\,\Delta r(p)\;y_j\Big(y_j-\|y\|_2^2\Big),
\qquad j=1,\ldots,K.
\end{equation}
In vector form, this is expressed as:
\begin{equation}
\label{eq:Edy-vector}
\mathbb E[\Delta y]
\;=\;-\,\eta\,p(1-p)\,\Delta r(p)\,\Big(y\odot y-\|y\|_2^2\,y\Big).
\end{equation}
Notably, $\sum_j \mathbb E[\Delta y_j]=0$, which confirms that the simplex remains invariant as expected. The fixed points of \eqref{eq:Edy-vector} are located at the barycenter and the vertices. When $\Delta r(p)<0$, a condition signifying an informative grader that favors good arms over bad arms, the uniform point becomes unstable and the vertices act as attractors.

Substituting $\Delta r(p)=-J/\sigma(p)$ from \eqref{eq:ag-ab-delta_r} into \eqref{eq:Edyj} results in:
\[
\mathbb E[\Delta y_j]
\;=\;\eta\,\frac{J}{\sigma(p)}\,p(1-p)\;y_j\Big(y_j-\|y\|_2^2\Big).
\]
For $J>0$, arms where $y_j>\|y\|_2^2$ will grow while those where $y_j<\|y\|_2^2$ shrink, representing a deterministic sharpening within the good block.

\subsection{From Expectation-Based Updates to ODEs: The Small-Step Bridge}\label{appendix:small_step}

Consider an expectation-level logit update:
\[
\boldsymbol{\theta}^{(t+1)} \;=\; \boldsymbol{\theta}^{(t)} \;+\; \eta\, g\!\big(\p^{(t)}\big),
\qquad
\Delta\boldsymbol{\theta} \;=\; \eta\,g(\p),
\]
where $g$ represents the per-step expected gradient. Through the softmax mapping $\p = \mathrm{softmax}(\boldsymbol{\theta})$, a small logit update is defined as $\Delta\p \;=\;\J(\p)\,\Delta\theta$, with $\J(\p) \;=\; \mathrm{Diag}(\p)-\p\p^\top$. Substituting $\Delta\theta = \eta\,g(\p)$ yields the expected increment for the policy:
\[
\Delta\p
\;=\;
\eta\,\big(\mathrm{Diag}(\p)-\p\p^\top\big)\,g(\p)
\;+\; O(\eta^2).
\]

\paragraph{Option 1: Unit Time per Iteration}
We may treat the iteration index itself as continuous time. Let $t\in\mathbb{R}$ denote a continuous extension of the discrete counter, where a single algorithmic update corresponds to a unit time step $\Delta t = 1$. By defining $\p(t)\approx \p^{(t)}$, the relationship is:
\[
\frac{\p^{(t+1)}-\p^{(t)}}{\Delta t}
\;\approx\;
\frac{d\p}{dt}(t)
\;=\;
\dot\p(t).
\]
Aligning this with the discrete update $\p^{(t+1)}-\p^{(t)} = \Delta\p$ results in the following ordinary differential equation (ODE):
\begin{equation}
\label{eq:softmax-flow-eta}
\dot{\p}(t)
\;=\;
\eta\,\big(\mathrm{Diag}(\p(t))-\p(t)\p(t)^\top\big)\,g\big(\p(t)\big).
\end{equation}
The expectation-level GRPO update thus serves as a forward-Euler discretization of the continuous-time dynamics in \eqref{eq:softmax-flow-eta} with a unit step size.

This ODE provides an accurate proxy within the small-learning-rate regime. The local truncation error of the Euler step satisfies $\|\p^{+}-\p - \dot{\p}\| = O(\eta^2)$, and given that $\max_a \eta\,|g_a(\p)| \ll 1$, no coordinate of $\p$ shifts excessively in a single iteration. Geometrically, \eqref{eq:softmax-flow-eta} remains a natural-gradient (Shahshahani) flow:
\[
\dot{\p}
\;=\;
\eta\,\mathbf{G}(\p)\,\nabla_{\!\theta}\! \mathcal{L},
\qquad
\mathbf{G}(\p)
\;=\;
\mathrm{Diag}(\p)-\p\p^\top,
\]
where $\mathbf{G}(\p)$ represents the Fisher metric tensor on the simplex. The factor $\eta$ scales the velocity along this geometric flow. By approximating discrete differences with derivatives, \eqref{eq:Edy-vector} and \eqref{eq:softmax-flow-eta} transform into the coupled ODEs outlined in \eqref{eq:coupled_dynamics}:
 \begin{subequations}
\begin{align*}
    \dot y &= \kappa(p) \left( y \odot y - \|y\|_2^2 y \right), \\
    \dot z &= -\kappa(p) \left( z \odot z - \|z\|_2^2 z \right), \\
    \dot p &= -\eta \frac{J}{\sigma(p)} [p(1-p)]^2 \left( \|y\|_2^2 + \|z\|_2^2 \right),
\end{align*}
\end{subequations}
where we define the proportionality factor $\kappa(p) := \eta \frac{J}{\sigma(p)} p(1-p)$.

\paragraph{Option 2: Alternative Time Rescaling}
An alternative approach involves absorbing the learning rate directly into the time variable. By defining a rescaled time $\mathsf{t} = \eta t$, each discrete update advances $\mathsf{t}$ by $\Delta \mathsf{t} = \eta$. Using the chain rule for $\p(\mathsf{t}) \coloneqq \p(t)$, we find:
\[
\frac{d\p}{d\mathsf{t}}
\;=\;
\frac{1}{\eta}\,\frac{d\p}{dt}
\;=\;
\big(\mathrm{Diag}(\p)-\p\p^\top\big)\,g(\p),
\]
which simplifies \eqref{eq:softmax-flow-eta} to:
\begin{equation}
\label{eq:softmax-flow}
\frac{d\p}{d\mathsf{t}}
\;=\;
\big(\mathrm{Diag}(\p(\mathsf{t}))-\p(\mathsf{t})\p(\mathsf{t})^\top\big)\,g\big(\p(\mathsf{t})\big).
\end{equation}
This represents the standard gradient-flow limit. 

\begin{remark}
While the trajectories in policy space remain identical across both time parametrizations, this work utilizes the unit time per iteration notation to maintain the visibility of mean-field correction terms as they relate to $\eta$.
\end{remark}

\newpage

\section{Maximal Learnability}\label{sec:max-learnability}
\begin{figure}
    \centering
    \includegraphics[width=1\linewidth]{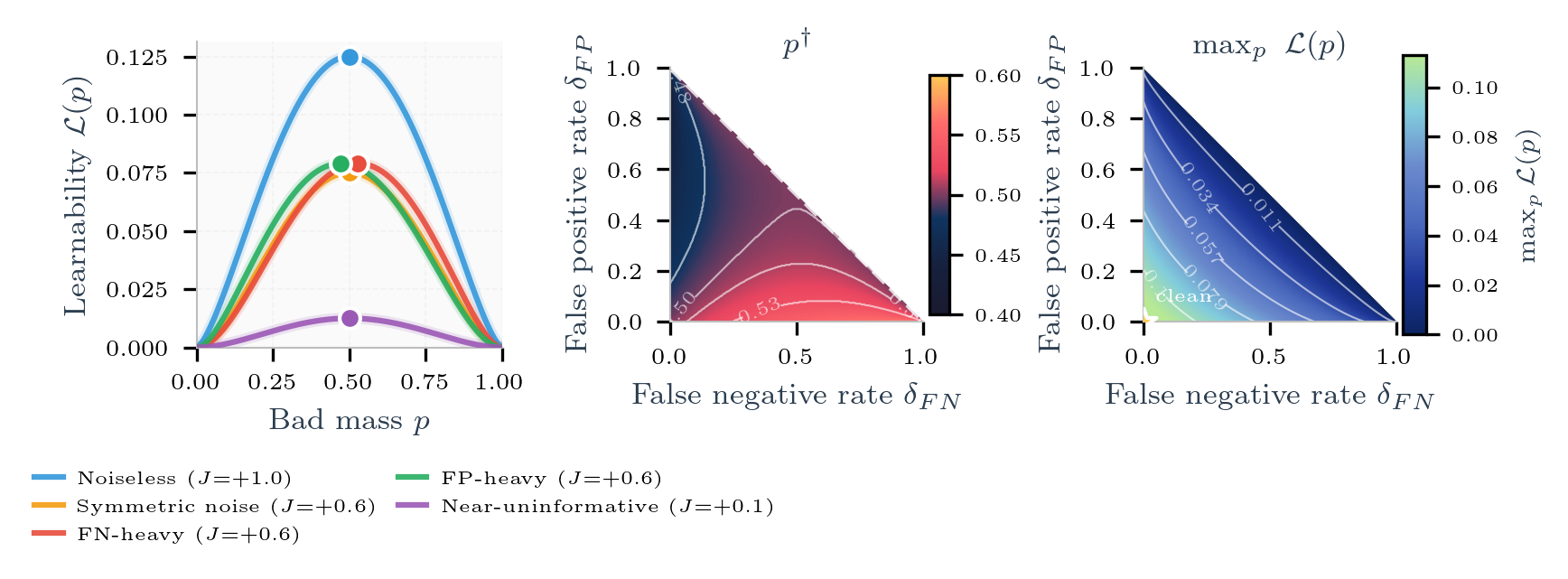}
    \caption{\textbf{Learnability-maximizing bad mass $p^\dagger$ under group $z$-scored rewards.}
We plot the \emph{instantaneous learnability speed}
$\mathcal{L}(p)=\frac{J}{\sigma(p)}[p(1-p)]^{2}$, which controls the magnitude of the one-step bad-mass drift
$|\Delta p|\propto \mathcal{L}(p)$ (up to a positive constant), where
$\sigma(p)=\sqrt{q(p)(1-q(p))}$, $q(p)=(1-\delta_{\mathrm{FN}})-Jp$, and $J=1-\delta_{\mathrm{FN}}-\delta_{\mathrm{FP}}$.
(A) Curves of $\mathcal{L}(p)$ versus $p$ for representative noise settings; markers indicate
$p^\dagger=\arg\max_{p\in[0,1]}\mathcal{L}(p)$.
(B) Heatmap of $p^\dagger(\delta_{\mathrm{FN}},\delta_{\mathrm{FP}})$ in the informative region $J>0$ (masked for $J\le 0$);
the dashed diagonal marks $J=0$ and contours indicate level sets of $p^\dagger$.
(C) Heatmap of the maximum instantaneous learnability $\max_{p}\mathcal{L}(p)$, showing how noise reduces the peak achievable drift.
Symmetric noise ($\delta_{\mathrm{FN}}=\delta_{\mathrm{FP}}$) preserves the midpoint optimum $p^\dagger=\tfrac12$, while asymmetric
noise shifts the maximizer away from $\tfrac12$ by reweighting the signal through $\sigma(p)^{-1}$.}
\label{fig:learnability_pdagger_triptych}
\end{figure}
We quantify a prompt's \emph{learnability} by the instantaneous rate at which GRPO reduces its latent bad-mode mass
\(p=\Pr(\text{bad}\mid x)\).
Under the block-symmetric mean-field approximation derived in
Appendix~\ref{appendix:Group-Normalized RL with Noisy Reward} (see \eqref{aux3}, the (unregularized) one-step drift of \(p\) takes the form)
\begin{equation}
\label{eq:learnability-drift}
|\Delta p|
\;\propto\;
\Delta(p)\,[p(1-p)]^{2},
\qquad
\Delta(p)
:=\E[\tilde r\mid \text{good}] - \E[\tilde r\mid \text{bad}],
\end{equation}
up to an overall positive step-size constant and smooth factors that vary slowly with \(p\).
Here \(\tilde r\) denotes the group-normalized reward.

\paragraph{Normalized separation under noisy rewards.}
With \(z\)-score normalization \eqref{eq:zscore}, the conditional means \eqref{eq:cond-means-tilder} imply a simple closed form for
the separation in normalized units:
\begin{equation}
\label{eq:Delta-gap}
\Delta(p)
\;=\;
\frac{J}{\sigma(p)},
\qquad
J=1-\delta_{\mathrm{FN}}-\delta_{\mathrm{FP}},
\qquad
\sigma(p)=\sqrt{q(p)\bigl(1-q(p)\bigr)},
\qquad
q(p)=(1-\delta_{\mathrm{FN}})-Jp.
\end{equation}
Consequently, the \emph{learnability speed} (i.e., the \(p\)-dependent component of the drift) is
\begin{equation}
\label{eq:learnability-speed}
\mathcal{L}(p;\delta_{\mathrm{FN}},\delta_{\mathrm{FP}})
\;:=\;
\frac{J}{\sigma(p)}\,[p(1-p)]^{2},
\qquad\text{so that}\qquad
|\Delta p|\;\propto\;\mathcal{L}(p;\delta_{\mathrm{FN}},\delta_{\mathrm{FP}}).
\end{equation}
Throughout this section we focus on the \emph{informative} regime \(J>0\) (the grader is better than random). When \(J=0\) the
signal vanishes, and when \(J<0\) it is anti-informative and must be corrected
(Remark~\ref{rem:J-negative}).

\paragraph{Noiseless case (\(J=1\)).}
When \(\delta_{\mathrm{FN}}=\delta_{\mathrm{FP}}=0\), we have \(q(p)=1-p\) and \(\sigma(p)=\sqrt{p(1-p)}\), hence
\begin{equation}
\label{eq:noiseless-L}
\mathcal{L}(p;0,0)
\;=\;
\frac{1}{\sqrt{p(1-p)}}\,[p(1-p)]^{2}
\;=\;
[p(1-p)]^{3/2}.
\end{equation}
This is maximized at \(p^\star=\tfrac12\).
Thus, the largest single-step reduction in bad mass occurs on ``medium-difficulty'' prompts where the model is roughly
\(50\)--\(50\) between good and bad solutions.
At the extremes \(p\to 0\) (almost always good) or \(p\to 1\) (almost always bad), the factor \(p(1-p)\) vanishes and learning
slows down sharply: additional GRPO steps make only marginal progress on highly saturated prompts.

\paragraph{Noisy grading: what changes.}
With noise, the learnability speed becomes
\begin{equation}
\label{eq:noisy-L}
\mathcal{L}(p;\delta_{\mathrm{FN}},\delta_{\mathrm{FP}})
\;=\;
\frac{J}{\sqrt{q(p)\bigl(1-q(p)\bigr)}}\,[p(1-p)]^{2},
\qquad
q(p)=(1-\delta_{\mathrm{FN}})-Jp.
\end{equation}
This expression highlights two distinct effects:
\begin{enumerate}
\item \textbf{Global shrinkage via \(J\).}
Increasing noise decreases \(J=1-\delta_{\mathrm{FN}}-\delta_{\mathrm{FP}}\), uniformly reducing learnability.
In the limit \(\delta_{\mathrm{FN}}+\delta_{\mathrm{FP}}\to 1\), the grader becomes uninformative and
\(\mathcal{L}\to 0\) for all \(p\).

\item \textbf{Reweighting via \(\sigma(p)^{-1}\).}
Group \(z\)-scoring divides by the reward standard deviation \(\sigma(p)\), so the effective learning signal is amplified when the
reward distribution is highly concentrated (small \(\sigma\)) and attenuated when it is maximally noisy (large \(\sigma\)).
Equivalently, the normalized separation is \(\Delta(p)=J/\sigma(p)\).
\end{enumerate}

\smallskip
\noindent\emph{Symmetric noise.}
If \(\delta_{\mathrm{FN}}=\delta_{\mathrm{FP}}=\delta\), then \(J=1-2\delta\) and
\(q(1-p)=1-q(p)\), implying \(\sigma(1-p)=\sigma(p)\).
Since \(p(1-p)\) is also symmetric, \(\mathcal{L}(1-p)=\mathcal{L}(p)\), and the maximizer remains at the symmetry point
\(p^\star=\tfrac12\).
Moreover, \(q(\tfrac12)=\tfrac12\) and \(\sigma(\tfrac12)=\tfrac12\), yielding the explicit peak value
\begin{equation}
\label{eq:symmetric-Lmax}
\mathcal{L}_{\max}(\delta,\delta)
\;=\;
\mathcal{L}\!\left(\tfrac12;\delta,\delta\right)
\;=\;
\frac{J}{1/2}\Big(\tfrac14\Big)^2
\;=\;
\frac{1-2\delta}{8}.
\end{equation}
Hence symmetric noise does \emph{not} shift the most-learnable difficulty, but it reduces the maximal attainable learning speed,
collapsing to zero as \(\delta\to\tfrac12\).

\smallskip
\noindent\emph{Asymmetric noise.}
When \(\delta_{\mathrm{FN}}\neq \delta_{\mathrm{FP}}\), the symmetry \(p\leftrightarrow 1-p\) is broken and the maximizer shifts to an
interior point \(p^\dagger\in(0,1)\) that balances the mixture factor \([p(1-p)]^2\) against the normalization term \(\sigma(p)\).
Differentiating \(\log\mathcal{L}\) yields the stationarity condition
\begin{equation}
\label{eq:pdagger-FOC}
p^\dagger\ \text{ solves }\ 
2\,\frac{1-2p}{p(1-p)}
\;+\;
\frac{J}{2}\,\frac{1-2q(p)}{q(p)\bigl(1-q(p)\bigr)}
\;=\;0,
\qquad
q(p)=(1-\delta_{\mathrm{FN}})-Jp,
\end{equation}
with boundary clipping if no interior maximizer exists.
Intuitively, the term \([p(1-p)]^2\) favors intermediate difficulty, while the factor \(1/\sigma(p)\) reweights the signal in a way
that can skew the optimum when false negatives and false positives are imbalanced. Equation~\eqref{eq:pdagger-FOC} admits an explicit solution but not a simple elementary one in general:
after substituting \(q(p)=(1-\delta_{\mathrm{FN}})-Jp\) and clearing denominators, the condition reduces to a cubic polynomial in
\(p\). Thus \(p^\dagger\) can be written in closed form via Cardano's formula, although the expression is cumbersome; in practice,
we select the real root in \(p\in(0,1)\) (or clip to \(\{0,1\}\) if the maximizer lies on the boundary).
A notable simplification occurs under symmetric noise \(\delta_{\mathrm{FN}}=\delta_{\mathrm{FP}}\), where the invariance
\(\mathcal{L}(1-p)=\mathcal{L}(p)\) forces \(p^\dagger=\tfrac12\). Fig.~\ref{fig:learnability_pdagger_triptych} shows $p\dagger$ for full range of noises in the learning phase.



\begin{remark}[Uninformative or adversarial graders]\label{rem:J-negative}
If \(J=0\) (i.e., \(\delta_{\mathrm{FN}}+\delta_{\mathrm{FP}}=1\)), then \(\Delta(p)=J/\sigma(p)=0\) and the mean update provides no
systematic signal to reduce bad mass. If \(J<0\), the grader is worse than random and \(\Delta(p)\) flips sign; equivalently,
swapping labels \(r\mapsto 1-r\) restores an effective \(J'>0\) and recovers the analysis above.
\end{remark}

\newpage 

\section{Lyapunov analysis and the role of $J$}

We consider a bandit configuration comprising $K$ good arms and $M$ bad arms, , as discussed in Appendix.~\ref{appendix:Group-Normalized RL with Noisy Reward}, defined by the probability vector:
\[
\p=(p_1,\ldots,p_K,\ p_{b_1},\ldots,p_{b_m})\in\Delta^{K+M-1},
\qquad
\sum_{j=1}^K p_j+\sum_{m=1}^M p_{b_m}=1.
\]
We define the total mass of the good and bad blocks respectively as:
\[
s(\p):=\sum_{j=1}^K p_j\in[0,1],\qquad
P_{\mathrm{bad}}(\p):=\sum_{m=1}^M p_{b_m}=1-s(\p),
\qquad
\mathfrak J(\p)=\mathrm{Diag}(\p)-\p\p^\top.
\]
Let $\mathbf 1\in\mathbb R^{K+M}$ denote the all-ones vector and let
\[
\mathbf 1_G:=(\underbrace{1,\ldots,1}_{K},\underbrace{0,\ldots,0}_{M})\in\mathbb R^{K+M}
\]
serve as the indicator vector for the good block.

For $s(\p)\in(0,1)$, it is convenient to introduce within-block normalized coordinates:
\[
y_j:=\frac{p_j}{s(\p)}\ (j\le K),\qquad
z_m:=\frac{p_{b_m}}{P_{\mathrm{bad}}(\p)}\ (m\le M).
\]
In this representation, $y\in\Delta^{K-1}$ and $z\in\Delta^{M-1}$, such that $\p=(s\,y,\ (1-s)\,z)$.

\paragraph{Block symmetry and GRPO parametrization.}
We assume a structure of block symmetry and state dependence defined by:
\[
A_j(\p)=a_{\mathrm g}\big(s(\p)\big)\quad(j\le K),\qquad
A_{b_m}(\p)=a_{\mathrm b}\big(s(\p)\big)\quad(m\le M).
\]
The resulting gap between good and bad arms is denoted as:
\[
\Delta(s):=a_{\mathrm g}(s)-a_{\mathrm b}(s).
\]
Under the GRPO specialization examined in this work, we set:
\[
a_{\mathrm g}(s)=\frac{J\,s}{\sigma(s)},\qquad
a_{\mathrm b}(s)=-\frac{J\,(1-s)}{\sigma(s)},\qquad
\sigma(s)>0.
\]
This formulation implies that:
\begin{equation}
\label{eq:gap}
\Delta(s)=\frac{J}{\sigma(s)}.
\end{equation}
The advantage vector can then be expressed as:
\[
\A(\p)=a_{\mathrm b}(s(\p))\,\mathbf 1+\Delta\big(s(\p)\big)\,\mathbf 1_G.
\]

\paragraph{GRPO mean-field flow.}
Our analysis focuses on the GRPO mean-field ordinary differential equation (ODE):
\begin{equation}
\label{eq:block-flow}
\dot \p \;=\; \eta\,\mathfrak J(\p)^2\,\A(\p),\qquad \eta>0.
\end{equation}


\begin{theorem}[Dichotomy by the sign of $J$; exchange of stability at $J=0$]
\label{thm:sign-J}
Assume the block-symmetric structure defined above, with $\Delta$ given by \eqref{eq:gap}.
Define a scalar potential $F:\Delta^{K+M-1}\to\mathbb R$ such that:
\begin{equation}
\label{eq:F-def}
F(\p) \;:=\; F\big(s(\p)\big),\qquad
F'(s)\;=\;\Delta(s).
\end{equation}
Given that $s(\p)\in[0,1]$ and $\Delta$ is integrable on $[0,1]$, $F$ remains bounded on $\Delta^{K+M-1}$.
For any constant $C$ satisfying $C \ge \sup_{\p} F(\p)$, we define the standard decreasing Lyapunov function:
\begin{equation}
\label{eq:V-def}
V(\p)\;:=\;C-F(\p)\;\ge\;0.
\end{equation}
Along any trajectory $\p(t)$ of \eqref{eq:block-flow} originating at $s(0)\in(0,1)$, the following properties hold:
\begin{enumerate}[label=(\roman*)]
\item \textbf{Lyapunov identity.} For all $t \ge 0$:
\begin{equation}
\label{eq:Vdot}
\frac{d}{dt}V\big(\p(t)\big)
\;=\;
-\eta\,\big\|\mathfrak J\big(\p(t)\big)\,\A\big(\p(t)\big)\big\|_2^2
\;\le\;0,
\end{equation}
where equality holds if and only if $\mathfrak J(\p(t))\,\A(\p(t))=0$. This is equivalent to:
\begin{equation}
\label{eq:Fdot}
\frac{d}{dt}F\big(\p(t)\big)
\;=\;
\eta\,\big\|\mathfrak J\big(\p(t)\big)\,\A\big(\p(t)\big)\big\|_2^2
\;\ge\;0.
\end{equation}

\item \textbf{Explicit field and the sign of $\dot s$.}
Let $s=s(\p)$ and $P_{\mathrm{bad}}:=1-s$. The components of the field are given by:
\[
\big[\mathfrak J(\p)\A(\p)\big]_j
= p_j\,P_{\mathrm{bad}}\;\Delta(s)\quad(j\le K),
\qquad
\big[\mathfrak J(\p)\A(\p)\big]_{b_m}
= -\,p_{b_m}\,s\;\Delta(s)\quad(m\le M).
\]
Consequently, for $J\neq 0$:
\begin{equation}
\label{eq:dots}
\dot s(t)
\;=\;
\frac{1}{\Delta\big(s(t)\big)}\,\frac{d}{dt}F\big(\p(t)\big)
\;=\;
\eta\,\Delta\big(s(t)\big)\Big(
P_{\mathrm{bad}}(t)^2\sum_{j=1}^K p_j(t)^2
\;+\;
s(t)^2\sum_{m=1}^M p_{b_m}(t)^2
\Big).
\end{equation}
In normalized coordinates, this simplifies to:
\begin{equation}
\label{eq:dots-yz}
\dot s
\;=\;
\eta\,\Delta(s)\,[s(1-s)]^2\Big(\|y\|_2^2+\|z\|_2^2\Big).
\end{equation}
Thus, for any $J \neq 0$, the sign of $\dot s(t)$ is identical to the sign of $J$.

\item \textbf{Global behavior and exchange of stability.}
\begin{itemize}
\item If $J>0$, $s(t)$ is strictly increasing such that $s(t)\uparrow 1$. In this case, every $\omega$--limit point lies in the good face $\Delta_G:=\{\,\p:\;p_{b_1}=\cdots=p_{b_M}=0\,\}$.
\item If $J<0$, $s(t)$ is strictly decreasing such that $s(t)\downarrow 0$. Here, every $\omega$--limit point lies in the bad face $\Delta_B:=\{\,\p:\;s(\p)=0\,\}$.
\item If $J=0$, the advantage $A_i(\p)$ is constant across all arms, implying $\dot\p\equiv 0$. Under these conditions, every point in the simplex is an equilibrium. 
\end{itemize}
As $J$ crosses zero, the global attractor switches from $\Delta_B$ to $\Delta_G$. This represents a parameter-driven exchange of stability, or phase transition, at $J=0$.

\item \textbf{Quantitative tail behavior.}
For $J\neq 0$ and bounded $\sigma$, and following the bounds established in Appendix \ref{appendix:retrieval-noisy-rewards}, we observe $0<\sigma_{\min}\le \sigma(s)\le \sigma_{\max}\leq \frac{1}{4}$. By Jensen's inequality, we obtain:
\begin{equation}
|\dot s(t)|
\;\ge\;
\eta\,\frac{|J|}{\sigma_{\max}}\Big(\frac{1}{K}+\frac{1}{M}\Big)s(t)^2\big(1-s(t)\big)^2.
\end{equation}
For $J>0$, there exists a time $T_{1/2}$ such that for all $t \ge T_{1/2}$, $s(t)\ge\frac12$ and:
\begin{equation}
\label{eq:1over-t}
P_{\mathrm{bad}}(t)
\;\le\;
\left( \frac{1}{P_{\mathrm{bad}}(T_{1/2})}
+\frac{\eta J}{4\sigma_{\max}}\Big(\frac{1}{K}+\frac{1}{M}\Big)\big(t-T_{1/2}\big) \right)^{-1}.
\end{equation}
Similarly, for $J<0$, $s(t)$ follows an $\mathcal{O}(1/t)$ decay after a finite transient period.
\end{enumerate}
\end{theorem}

\begin{proof}
Let $F(s):=\int_0^s \Delta(u)\,du$. Since $F(\p)=F(s(\p))$ and $s(\p)=\sum_{j\le K}p_j$, it follows that $\nabla F(\p)=\Delta(s)\,\mathbf 1_G$. Utilizing the fact that $\mathfrak J(\p)\mathbf 1=\mathbf 0$, we find:
\[
\mathfrak J(\p)\,\nabla F(\p)
=\Delta(s)\,\mathfrak J(\p)\mathbf 1_G
=\mathfrak J(\p)\A(\p).
\]
Applying the chain rule and the symmetry of $\mathfrak J(\p)$ yields:
\[
\frac{d}{dt}F(\p(t))
=\nabla F(\p(t))^\top\dot\p(t)
=\eta\,\big\|\mathfrak J(\p(t))\A(\p(t))\big\|_2^2,
\]
which confirms \eqref{eq:Fdot} and, by extension, \eqref{eq:Vdot}.

The coordinate-wise expressions in (ii) arise from the identity $[\mathfrak J(\p)\A]_i=p_i(A_i-\bar A)$, where $\bar A=s\,a_{\mathrm g}(s)+(1-s)\,a_{\mathrm b}(s)$. The resulting differential inequality in (iv) is obtained by combining \eqref{eq:dots} with Jensen's lower bounds and the definition $\Delta(s)=J/\sigma(s)$. Integrating over $[T_{1/2},t]$ while assuming $s(t)\ge\frac12$ produces the bound in \eqref{eq:1over-t}.
\end{proof}

\begin{remark}[Interpretation]
The dynamics described by \eqref{eq:dots} and \eqref{eq:dots-yz} indicate that the direction of mass transfer between blocks is governed exclusively by the sign of $J$. For $J>0$, the good face $\Delta_G$ is globally attracting, whereas $J<0$ renders the bad face $\Delta_B$ the attractor. At the critical value $J=0$, the system undergoes a phase transition characterized by a degenerate continuum of equilibria.
\end{remark}


\paragraph{Decomposition dynamics and Shahshahani structure.}
For $s\in(0,1)$, we decompose the flow into coordinates $(y,z,s)$. Equation \eqref{eq:block-flow} then reduces to the following system:
\begin{equation}
\label{eq:ys-flow}
\dot y
= \kappa(s)\;\big(y\odot y-\|y\|_2^2\,y\big),
\qquad
\dot z
= -\kappa(s)\;\big(z\odot z-\|z\|_2^2\,z\big),
\qquad
\dot s
= \eta\,\Delta(s)\,[s(1-s)]^2\big(\|y\|_2^2+\|z\|_2^2\big),
\end{equation}
where $\kappa(s) := \eta\,\Delta(s)\,s(1-s)$. This reveals that $\Phi(y):=\tfrac12\|y\|_2^2$ serves as a state-scaled Shahshahani-gradient potential on $\Delta^{K-1}$, while $\Psi(z):=\tfrac12\|z\|_2^2$ acts as the corresponding potential on $\Delta^{M-1}$ but with an inverted sign. For $J>0$, we observe $\dot\Phi\ge 0$ and $\dot\Psi\le 0$. These inequalities reverse when $J<0$.

\paragraph{Probabilistic interpretation of the potential $\Phi(y)$.}
The function $\Phi(y) = \frac{1}{2}\sum_{i=1}^K y_i^2$ represents the collision probability, also known as the Herfindahl-Hirschman index. It quantifies the concentration within the good block, where $\Phi(y)=1/2K$ at the uniform distribution and $\Phi(y)=1/2$ at any vertex. An increase in $\Phi$ signifies that mass is condensing onto a smaller subset of arms, while a decrease indicates a more uniform distribution.

The induced intra-good dynamics are expressed as $\dot y = \kappa(s)\,\operatorname{grad}_{\mathrm{Shah}}\Phi(y)$ . This shows that $y$ follows a Shahshahani gradient flow of $\Phi$. Consequently, if $J>0$, the collision probability increases over time, leading to a rich-get-richer effect where the distribution concentrates toward a vertex. Conversely, if $J<0$, the flow promotes diversity by pushing $y$ toward a uniform distribution.

\paragraph{Coordination-game correspondence.}
The replicator field $\dot y_i \propto y_i(y_i-\|y\|_2^2)$ corresponds to the replicator dynamics of a symmetric pure coordination game. In this context, the payoff for each arm is equal to its current population fraction. Our Shahshahani-gradient identity formally states that the replicator dynamics ascend this potential when $\kappa > 0$ and descend it when $\kappa < 0$.

\newpage

\section{Bad arm dynamics under PPO/GRPO style importance sampling and clipping}
\label{app:bad-arm-ppo-grpo}

We demonstrate that, within the small-step regime, importance sampling (IS) as employed in PPO and GRPO does not alter the leading-order mean-field ODE established in Appendix~C. Specifically, IS modifies the conditional mean drift only at order $O(\eta^2)$, ensuring that the $O(\eta)$ ODE limit remains identical to the on-policy (REINFORCE) system.

\paragraph{Setup.}
Consider a softmax policy $\pi_\theta$ over arms $i\in\{b,1,\dots,K\}$. During an update, samples are drawn from an "old" policy $\pi_{\rm old}$ with probability vector $\mathbf p=\mathbf p_{\rm old}$. The parameters are then updated to a "new" policy $\pi_{\rm new}$ with probability vector $\mathbf p^+=\mathbf p_{\rm new}$.
We define the exact importance ratio as:
\[
\rho_i \;:=\; \frac{\pi_{\rm new}(i)}{\pi_{\rm old}(i)} \;=\; \frac{p_i^+}{p_i}.
\]
This ratio is well-defined because the softmax function has full support ($p_i>0$). Let $\tilde r$ denote the scalar signal used in the update, and let:
\[
A_i(\mathbf p)\;:=\;\E[\tilde r\mid I=i,\mathbf p].
\]
We assume $\sup_{i,\mathbf p}|A_i(\mathbf p)|<\infty$. In this context, the small-step regime implies $\|\Delta\theta\|=O(\eta)$ and consequently $\|\mathbf p^+-\mathbf p\|=O(\eta)$, a result supported by Lemma~\ref{lem:softmax-tangent}.

\paragraph{IS score-function update.}
The IS-corrected score-function update is given by:
\[
g_{\rm IS}
\;:=\;
\tilde r\,\rho_I\,\nabla_\theta\log\pi_{\rm new}(I),
\qquad I\sim \pi_{\rm old},
\qquad
\Delta\theta=\eta\,g_{\rm IS}.
\]
For a softmax policy, where $\nabla_\theta\log\pi_{\rm new}(i)=e_i-\mathbf p^+$, the update simplifies to $g_{\rm IS}=\tilde r\,\rho_I\,(e_I-\mathbf p^+)$.

\begin{proposition}[IS affects the mean logit drift only at $O(\eta^2)$]
\label{prop:is-oeta2}
Define the on-policy vector field:
\[
G(\mathbf p)\;:=\;\sum_i p_i\,A_i(\mathbf p)\,(e_i-\mathbf p).
\]
In the small-step regime, the following holds:
\[
\E[\Delta\theta\mid \mathbf p]
\;=\;
\eta\,G(\mathbf p)\;+\;O(\eta^2).
\]
\end{proposition}

\begin{proof}
Taking the conditional expectation and applying $\E[\tilde r\mid I=i,\mathbf p]=A_i(\mathbf p)$, we have:
\[
\E[g_{\rm IS}\mid \mathbf p]
=
\sum_i p_i\,\rho_i\,A_i(\mathbf p)\,(e_i-\mathbf p^+).
\]
By utilizing the identity $p_i\rho_i=p_i^+$, the expression becomes:
\[
\E[g_{\rm IS}\mid \mathbf p]
=
\sum_i p_i^+\,A_i(\mathbf p)\,(e_i-\mathbf p^+)
=:\widetilde G(\mathbf p^+;\mathbf p),
\quad
\widetilde G(\mathbf q;\mathbf p):=\sum_i q_i\,A_i(\mathbf p)\,(e_i-\mathbf q).
\]
For a fixed $\mathbf p$, the map $\mathbf q\mapsto \widetilde G(\mathbf q;\mathbf p)$ is a smooth polynomial in $\mathbf q$ with bounded coefficients. It is therefore locally Lipschitz in $\mathbf q$. Given that $\|\mathbf p^+-\mathbf p\|=O(\eta)$, it follows that:
\[
\widetilde G(\mathbf p^+;\mathbf p)=\widetilde G(\mathbf p;\mathbf p)+O(\|\mathbf p^+-\mathbf p\|)
=G(\mathbf p)+O(\eta).
\]
Multiplying by $\eta$ yields the desired result: $\E[\Delta\theta\mid \mathbf p]=\eta\,G(\mathbf p)+O(\eta^2)$.
\end{proof}

\paragraph{Expanded form and relation to Appendix~C.}
This conclusion can also be derived by examining the identity $\rho_i=1+\Delta p_i/p_i$, where $\Delta p_i:=p_i^+-p_i$. This implies that $\E[\rho_i\mid \mathbf p]=1+\E[\Delta p_i\mid \mathbf p]/p_i$. Since replacing $\mathbf p^+$ with $\mathbf p$ introduces only an $O(\eta)$ error, its contribution after multiplying by $\eta$ is $O(\eta^2)$. Consequently:
\begin{equation}
\label{eq:is-expanded-mean}
\E[\Delta\theta\mid \mathbf p]
=
\eta\sum_i \Bigl(p_i+\E[\Delta p_i\mid \mathbf p]\Bigr)\,A_i(\mathbf p)\,(e_i-\mathbf p)
\;+\;O(\eta^2).
\end{equation}
By letting $\bar A:=\sum_i p_iA_i(\mathbf p)$ and $\bar A':=\sum_i \E[\Delta p_i\mid \mathbf p]A_i(\mathbf p)$, the coordinate-wise updates are expressed as:
\begin{align}
\label{eq:is-expanded-coordinates}
\E[\Delta\theta_i\mid \mathbf p]
&=
\eta\,p_i\bigl(A_i(\mathbf p)-\bar A\bigr)
+\eta\Bigl(\E[\Delta p_i\mid \mathbf p]\,A_i(\mathbf p)-p_i\,\bar A'\Bigr)
+O(\eta^2).
\end{align}
This structure highlights the "extra terms" that appear when the IS mean update is expanded. To show these terms are second order, we apply the softmax pushforward from Lemma~\ref{lem:softmax-tangent}, yielding $\E[\Delta\mathbf p\mid \mathbf p]=\mathfrak J(\mathbf p)\,\E[\Delta\theta\mid \mathbf p]+O(\eta^2)$. Because $\E[\Delta\theta\mid \mathbf p]=O(\eta)$, the term $\E[\Delta p_i\mid \mathbf p]$ is also $O(\eta)$, confirming that the correction in \eqref{eq:is-expanded-coordinates} is $O(\eta^2)$.

\paragraph{Block-symmetric specialization.}
Under the assumption of block symmetry, we have $A_{b_m}(\mathbf p)=a_{\mathrm b}(p)$ and $A_j(\mathbf p)=a_{\mathrm g}(p)$ for $m=1,\ldots,M$ and  $j=1,\ldots,K$. The mean advantage is then $\bar A=p\,a_{\mathrm b}(p)+(1-p)a_{\mathrm g}(p)$. The leading terms in \eqref{eq:is-expanded-coordinates} result in:
\[
\E[\Delta\theta_{b_m}\mid \mathbf p]
=
\eta\,p(1-p)\big(a_{\mathrm b}(p)-a_{\mathrm g}(p)\big)z_m+O(\eta^2),
\]
\[
\E[\Delta\theta_j\mid \mathbf p]
=
-\eta\,p(1-p)\big(a_{\mathrm b}(p)-a_{\mathrm g}(p)\big)\,y_j+O(\eta^2).
\]

These expressions match the REINFORCE dynamics to first order in $\eta$.

\begin{corollary}[Invariance of the leading-order mean-field ODE]
\label{cor:is-ode}
In the small-step regime, the conditional expectation of the probability update satisfies:
\[
\E[\Delta\mathbf p\mid \mathbf p]
=
\eta\,\mathfrak J(\mathbf p)\,G(\mathbf p)+O(\eta^2).
\]
Consequently, the leading-order mean-field ODE for good and bad arm dynamics remains as derived in Appendix~C.
\end{corollary}

\paragraph{Clipping in PPO and GRPO.}
When clipping is applied to the ratios such that $\widehat\rho_i=\mathrm{clip}(\rho_i,1-\varepsilon,1+\varepsilon')$, the ratio $\rho_i$ approaches $1$ as $\eta \to 0$. For sufficiently small $\eta$, the clipping mechanism remains inactive. Therefore, clipping does not influence the leading-order ODE.

\paragraph{Conclusion.}
In the small-step regime, the inclusion of importance sampling and clipping with fixed thresholds does not change the leading-order mean-field ODE. The impact of these techniques is confined to the $O(\eta^2)$ terms.

\newpage

\section{KL Regularization}
\label{sec:kl}

In this section, we analyze how KL regularization modifies the policy dynamics in our multi-armed bandit abstraction. Many practical algorithms (e.g., PPO/GRPO-style methods) add a KL penalty to keep the learned policy close to a reference policy, and we study the resulting mean-field drift in our multi-good/multi-bad model.

\medskip
We begin by recalling the notation for the \emph{multi-good/bad-arm} model introduced in Appendix~\ref{appendix:Group-Normalized RL with Noisy Reward}. Given $K$ good arms and $M$ bad arms, we represent the policy as
\[
\mathbf p=(p_1,\ldots,p_K,\;p_{b_1},\ldots,p_{b_M})\in\Delta^{K+M-1}.
\]
The \emph{total bad mass} and the within-block compositions are
\[
p \;=\; \sum_{m=1}^M p_{b_m}\in[0,1],\qquad
y_j \;=\; \frac{p_j}{1-p}\ \ (j\le K),\qquad
z_m \;=\; \frac{p_{b_m}}{p}\ \ (m\le M).
\]
Equivalently,
\[
p_j=(1-p)\,y_j\quad(j\le K),\qquad
p_{b_m}=p\,z_m\quad(m\le M),
\]
with $y\in\Delta^{K-1}$ and $z\in\Delta^{M-1}$.

\paragraph{Reference policy.}
We define the KL reference policy $\mathbf p^{\mathrm{ref}}$ analogously:
\[
\mathbf p^{\mathrm{ref}}
=\big((1-p_{\mathrm{ref}})\,y^{\mathrm{ref}}_1,\ldots,(1-p_{\mathrm{ref}})\,y^{\mathrm{ref}}_K,\;
p_{\mathrm{ref}}\,z^{\mathrm{ref}}_1,\ldots,p_{\mathrm{ref}}\,z^{\mathrm{ref}}_M\big),
\]
where $p_{\mathrm{ref}}\in(0,1)$, $y^{\mathrm{ref}}\in\Delta^{K-1}$, and $z^{\mathrm{ref}}\in\Delta^{M-1}$. We also use the bad-mass log-odds
\[
\ell(p):=\log\frac{p}{1-p},\qquad
\ell_{\mathrm{ref}}:=\log\frac{p_{\mathrm{ref}}}{1-p_{\mathrm{ref}}}.
\]

\paragraph{Replicator/natural-gradient form (recall).}
Throughout, we view a penalty $\Phi(\mathbf p)$ as inducing a Shahshahani (natural-gradient) flow on the simplex. Concretely, for the objective contribution $-\beta\,\Phi(\mathbf p)$, the induced replicator flow is
\begin{equation}
\label{eq:replicator-recall}
\dot p_i \;=\; -\,\beta\,p_i\Big(\partial_{p_i}\Phi(\mathbf p)\;-\;\sum_{k}p_k\,\partial_{p_k}\Phi(\mathbf p)\Big).
\end{equation}
Equivalently, if we interpret $\Delta\theta_i$ as a small logit increment, then to first order
\begin{equation}
\label{eq:logit-smallstep-recall}
\Delta p_i \;=\; p_i\Big(\Delta\theta_i-\sum_k p_k\Delta\theta_k\Big)\;+\;O(\|\Delta\theta\|^2).
\end{equation}
Thus, choosing \emph{centered} increments with $\sum_k p_k\Delta\theta_k=0$ yields $\Delta p_i = p_i\Delta\theta_i$ at first order, which makes it easy to realize a desired replicator drift.

\begin{lemma}[Exact reverse-KL decomposition (multi-bad setting)]
\label{lem:kl-decomp}
The reverse KL divergence decomposes into a two-class term (bad vs.\ good) plus within-block terms:
\[
D_{\mathrm{KL}}(\mathbf p\|\mathbf p^{\mathrm{ref}})
=
\underbrace{p\log\frac{p}{p_{\mathrm{ref}}}+(1-p)\log\frac{1-p}{1-p_{\mathrm{ref}}}}_{\text{two-class (bad vs.\ good)}}
\;+\;
\underbrace{(1-p)\,D_{\mathrm{KL}}(y\|y^{\mathrm{ref}})}_{\text{within-good}}
\;+\;
\underbrace{p\,D_{\mathrm{KL}}(z\|z^{\mathrm{ref}})}_{\text{within-bad}}.
\]
\end{lemma}
\emph{Proof.} Substituting $p_j=(1-p)y_j$ and $p_{b_m}=p z_m$, we get
\[
\log\frac{p_j}{p_j^{\mathrm{ref}}}
=\log\frac{1-p}{1-p_{\mathrm{ref}}}+\log\frac{y_j}{y_j^{\mathrm{ref}}},
\qquad
\log\frac{p_{b_m}}{p_{b_m}^{\mathrm{ref}}}
=\log\frac{p}{p_{\mathrm{ref}}}+\log\frac{z_m}{z_m^{\mathrm{ref}}}.
\]
Summing $\sum_{j\le K}p_j(\cdot)$ and $\sum_{m\le M}p_{b_m}(\cdot)$ yields the stated decomposition. \qed

\paragraph{Two KL choices.}
In practice, one may penalize either (A) only the two-class (inter-block) divergence or (B) the full reverse-KL divergence.

\begin{proposition}[Two-class KL penalty (bad vs.\ good only)]
\label{prop:kl-two-class}
Let
\[
\Phi_{\mathrm{2c}}(p)
:=p\log\frac{p}{p_{\mathrm{ref}}}+(1-p)\log\frac{1-p}{1-p_{\mathrm{ref}}}.
\]
The Shahshahani/natural-gradient (replicator) flow for $-\beta\,\Phi_{\mathrm{2c}}$ satisfies
\[
\dot p\Big|_{\mathrm{KL,2c}}
=-\,\beta\,p(1-p)\big(\ell-\ell_{\mathrm{ref}}\big),
\qquad
\dot\ell\Big|_{\mathrm{KL,2c}}
=-\,\beta\big(\ell-\ell_{\mathrm{ref}}\big).
\]
A centered small-step logit update realizing this flow (i.e., $\sum_i p_i\,\Delta\theta_i^{\mathrm{KL}}=0$) is
\[
\Delta\theta^{\mathrm{KL}}_{b_m}=-\,\beta(1-p)\,(\ell-\ell_{\mathrm{ref}})\quad(m\le M),
\qquad
\Delta\theta^{\mathrm{KL}}_{j}=+\,\beta p\,(\ell-\ell_{\mathrm{ref}})\quad(j\le K).
\]
Because $\Delta\theta_i^{\mathrm{KL}}$ is constant within each block, this KL term induces no deterministic drift in $y$ or $z$; it acts only on the total bad mass $p$.

\emph{Proof sketch (with key steps).}
View $\Phi_{\mathrm{2c}}$ as a function on the full simplex via $p=\sum_{m=1}^M p_{b_m}$. Then
\[
\partial_{p_{b_m}}\Phi_{\mathrm{2c}}(\mathbf p)=\Phi_{\mathrm{2c}}'(p),
\qquad
\partial_{p_j}\Phi_{\mathrm{2c}}(\mathbf p)=0.
\]
Moreover,
\[
\sum_i p_i\,\partial_{p_i}\Phi_{\mathrm{2c}}(\mathbf p)
=\sum_{m=1}^M p_{b_m}\,\Phi_{\mathrm{2c}}'(p)
= p\,\Phi_{\mathrm{2c}}'(p).
\]
Plugging into the replicator formula \eqref{eq:replicator-recall} gives the coordinate-wise drifts
\[
\dot p_{b_m}
=-\beta\,p_{b_m}\Big(\Phi'_{\mathrm{2c}}(p)-p\,\Phi'_{\mathrm{2c}}(p)\Big)
=-\beta\,p_{b_m}(1-p)\,\Phi'_{\mathrm{2c}}(p),
\]
\[
\dot p_{j}
=-\beta\,p_{j}\Big(0-p\,\Phi'_{\mathrm{2c}}(p)\Big)
=+\beta\,p_{j}p\,\Phi'_{\mathrm{2c}}(p).
\]
Summing over bad arms yields
\[
\dot p
=\sum_{m=1}^M \dot p_{b_m}
=-\beta(1-p)\Phi'_{\mathrm{2c}}(p)\sum_{m=1}^M p_{b_m}
=-\beta\,p(1-p)\,\Phi'_{\mathrm{2c}}(p).
\]
A direct derivative computation shows
\[
\Phi'_{\mathrm{2c}}(p)
=\log\frac{p}{p_{\mathrm{ref}}}-\log\frac{1-p}{1-p_{\mathrm{ref}}}
=\log\frac{p}{1-p}-\log\frac{p_{\mathrm{ref}}}{1-p_{\mathrm{ref}}}
=\ell-\ell_{\mathrm{ref}},
\]
which gives the claimed $\dot p$ equation. Since $\dot\ell=\dot p/[p(1-p)]$, we obtain $\dot\ell=-\beta(\ell-\ell_{\mathrm{ref}})$.

Finally, there is no within-block drift: for instance, for any $m,r\le M$,
\[
\frac{d}{dt}\log\frac{p_{b_m}}{p_{b_r}}
=\frac{\dot p_{b_m}}{p_{b_m}}-\frac{\dot p_{b_r}}{p_{b_r}}
=-\beta(1-p)\Phi'_{\mathrm{2c}}(p)+\beta(1-p)\Phi'_{\mathrm{2c}}(p)=0,
\]
so all ratios $p_{b_m}/p_{b_r}$ are constant and hence $z$ is constant; the same argument applies to $y$.

For the logit realization, take the \emph{centered} velocity
\[
\Delta\theta_i^{\mathrm{KL}}=-\beta\Big(\partial_{p_i}\Phi_{\mathrm{2c}}-\sum_k p_k\partial_{p_k}\Phi_{\mathrm{2c}}\Big),
\]
which gives exactly the stated block-constant increments (substitute $\Phi'_{\mathrm{2c}}(p)=\ell-\ell_{\mathrm{ref}}$). By centering, $\sum_i p_i\Delta\theta_i^{\mathrm{KL}}=0$, and then \eqref{eq:logit-smallstep-recall} implies
\[
\Delta p=\sum_{m=1}^M \Delta p_{b_m}
=\sum_{m=1}^M p_{b_m}\Delta\theta_{b_m}^{\mathrm{KL}}
=-\beta\,p(1-p)\,(\ell-\ell_{\mathrm{ref}}),
\]
matching the continuous-time drift. \qed
\end{proposition}

\begin{proposition}[Full reverse-KL penalty (multi-bad setting)]
\label{prop:kl-full}
Let $\Phi_{\mathrm{full}}(\mathbf p):=D_{\mathrm{KL}}(\mathbf p\|\mathbf p^{\mathrm{ref}})$.
The replicator flow for $-\beta\,\Phi_{\mathrm{full}}$ induces the bad-mass drift
\[
\dot p\Big|_{\mathrm{KL,full}}
=-\,\beta\,p(1-p)\Big(\ell-\ell_{\mathrm{ref}}
- D_{\mathrm{KL}}(y\|y^{\mathrm{ref}})
+ D_{\mathrm{KL}}(z\|z^{\mathrm{ref}})\Big),
\]
and hence
\[
\dot\ell\Big|_{\mathrm{KL,full}}
=-\,\beta\Big(\ell-\ell_{\mathrm{ref}}
- D_{\mathrm{KL}}(y\|y^{\mathrm{ref}})
+ D_{\mathrm{KL}}(z\|z^{\mathrm{ref}})\Big).
\]
A canonical centered logit increment follows the natural-gradient form
\[
\Delta\theta^{\mathrm{KL}}_i
=-\,\beta\Big(\log\tfrac{p_i}{p^{\mathrm{ref}}_i}-\sum_k p_k\log\tfrac{p_k}{p^{\mathrm{ref}}_k}\Big),
\]
which satisfies $\sum_i p_i\Delta\theta^{\mathrm{KL}}_i=0$ and yields the stated $p$-drift.

Moreover, in $(y,z)$-coordinates, the full reverse-KL induces independent within-block replicator pulls:
\[
\dot y_j\Big|_{\mathrm{KL,full}}
=-\,\beta\,y_j\Big(\log\tfrac{y_j}{y_j^{\mathrm{ref}}}-\sum_{i=1}^K y_i\log\tfrac{y_i}{y_i^{\mathrm{ref}}}\Big),
\qquad
\dot z_m\Big|_{\mathrm{KL,full}}
=-\,\beta\,z_m\Big(\log\tfrac{z_m}{z_m^{\mathrm{ref}}}-\sum_{r=1}^M z_r\log\tfrac{z_r}{z_r^{\mathrm{ref}}}\Big).
\]

\emph{Proof (outline with the main algebra).}
For $\Phi_{\mathrm{full}}(\mathbf p)=\sum_i p_i\log\frac{p_i}{p_i^{\mathrm{ref}}}$, we have
$\partial_{p_i}\Phi_{\mathrm{full}}=\log\frac{p_i}{p_i^{\mathrm{ref}}}+1$; the constant $+1$ cancels under the mean-subtraction in \eqref{eq:replicator-recall}, yielding
\[
\dot p_i=-\beta\,p_i\Big(\log\tfrac{p_i}{p_i^{\mathrm{ref}}}-\sum_k p_k\log\tfrac{p_k}{p_k^{\mathrm{ref}}}\Big).
\]
To get $\dot p$, sum over bad indices:
\[
\dot p
=\sum_{m=1}^M \dot p_{b_m}
=-\beta\Big(\sum_{m=1}^M p_{b_m}\log\tfrac{p_{b_m}}{p_{b_m}^{\mathrm{ref}}}
\;-\;p\sum_k p_k\log\tfrac{p_k}{p_k^{\mathrm{ref}}}\Big).
\]
Using the block parametrization,
\[
\log\frac{p_{b_m}}{p_{b_m}^{\mathrm{ref}}}
=\log\frac{p}{p_{\mathrm{ref}}}+\log\frac{z_m}{z_m^{\mathrm{ref}}},
\qquad
\log\frac{p_{j}}{p_{j}^{\mathrm{ref}}}
=\log\frac{1-p}{1-p_{\mathrm{ref}}}+\log\frac{y_j}{y_j^{\mathrm{ref}}},
\]
we obtain
\[
\sum_{m=1}^M p_{b_m}\log\tfrac{p_{b_m}}{p_{b_m}^{\mathrm{ref}}}
=p\log\tfrac{p}{p_{\mathrm{ref}}}+p\,D_{\mathrm{KL}}(z\|z^{\mathrm{ref}}),
\]
and
\[
\sum_{k} p_k\log\tfrac{p_k}{p_k^{\mathrm{ref}}}
=\Phi_{\mathrm{2c}}(p)+(1-p)D_{\mathrm{KL}}(y\|y^{\mathrm{ref}})+pD_{\mathrm{KL}}(z\|z^{\mathrm{ref}})
\quad\text{(by Lemma~\ref{lem:kl-decomp}).}
\]
Substituting and simplifying yields
\[
\dot p
=-\beta\,p(1-p)\Big(\log\tfrac{p}{p_{\mathrm{ref}}}-\log\tfrac{1-p}{1-p_{\mathrm{ref}}}
- D_{\mathrm{KL}}(y\|y^{\mathrm{ref}})
+ D_{\mathrm{KL}}(z\|z^{\mathrm{ref}})\Big),
\]
and the log-odds form follows since
$\log\tfrac{p}{p_{\mathrm{ref}}}-\log\tfrac{1-p}{1-p_{\mathrm{ref}}}=\ell-\ell_{\mathrm{ref}}$.
Finally, the $(y,z)$ equations follow by differentiating $y_j=p_j/(1-p)$ and $z_m=p_{b_m}/p$ and substituting the above $\dot p_i$ expressions; the same cancellations as in Proposition~\ref{prop:kl-two-class} reduce each block to its own replicator pull toward the corresponding within-block reference. \qed
\end{proposition}

\begin{remark}[Comparison of KL penalties]
The two-class penalty $\Phi_{\mathrm{2c}}(p)$ depends \emph{only} on the aggregate bad mass $p$ (equivalently, on $\ell=\log\frac{p}{1-p}$). Consequently, its Shahshahani field is \emph{block-constant}: all good arms receive the same logit increment and all bad arms receive the same logit increment. This implies that the within-block compositions are invariant,
\[
\dot y\Big|_{\mathrm{KL,2c}}=0,\qquad \dot z\Big|_{\mathrm{KL,2c}}=0,
\]
and the KL regularizer acts solely as a logistic contraction of $\ell$ toward $\ell_{\mathrm{ref}}$:
\[
\dot\ell\Big|_{\mathrm{KL,2c}}=-\beta(\ell-\ell_{\mathrm{ref}}).
\]

In contrast, the full reverse-KL splits as
\[
D_{\mathrm{KL}}(\mathbf p\|\mathbf p^{\mathrm{ref}})
=\Phi_{\mathrm{2c}}(p)+(1-p)D_{\mathrm{KL}}(y\|y^{\mathrm{ref}})+pD_{\mathrm{KL}}(z\|z^{\mathrm{ref}}),
\]
so it produces \emph{two} effects simultaneously: (i) independent within-block replicator pulls that damp deviations of $y$ and $z$ from $y^{\mathrm{ref}}$ and $z^{\mathrm{ref}}$, and (ii) an additional scalar feedback into the bad-mass drift. Concretely, the effective restoring force on the log-odds becomes
\[
\dot\ell\Big|_{\mathrm{KL,full}}
=-\beta\Big(\ell-\ell_{\mathrm{ref}}-D_{\mathrm{KL}}(y\|y^{\mathrm{ref}})+D_{\mathrm{KL}}(z\|z^{\mathrm{ref}})\Big).
\]
Thus, if the bad block is more ``misaligned'' than the good block
($D_{\mathrm{KL}}(z\|z^{\mathrm{ref}})>D_{\mathrm{KL}}(y\|y^{\mathrm{ref}})$),
the KL term strengthens the push to \emph{decrease} $p$; whereas a larger within-good mismatch
($D_{\mathrm{KL}}(y\|y^{\mathrm{ref}})$ large) can partially counteract that push because it is weighted by $(1-p)$ in the decomposition. In particular, when $y=y^{\mathrm{ref}}$ and $z=z^{\mathrm{ref}}$, the full reverse-KL reduces to the two-class behavior.
\end{remark}

\subsection{Full Mean-Field ODE for the Bad Mass with KL}
\label{subsec:mf-ode-kl}

We now integrate the reward-driven mean-field drift with the KL penalties derived above. For this analysis, we omit the clipping or importance-ratio factors typically found in PPO/GRPO: our previous results indicate these contribute only second-order corrections and do not alter the fundamental phase portrait of the mean-field drift.

Define the following quantities:
\[
\alpha(p):=\eta\,\frac{J}{\sigma(p)}\,p(1-p),\qquad
s_2:=\|y\|_2^2\in\Big[\tfrac1K,\,1\Big],\qquad
t_2:=\|z\|_2^2\in\Big[\tfrac1M,\,1\Big].
\]
Based on the reward term alone, the baseline drift of the \emph{total bad mass} is
\begin{equation}
\label{eq:reward-baseline}
\mathbb E[\Delta p]_{\mathrm{reward}}
=
-\,\alpha(p)\,p(1-p)\,\big(s_2+t_2\big)
=
-\,\eta\,\frac{J}{\sigma(p)}\,[p(1-p)]^2\,\big(s_2+t_2\big).
\end{equation}

\paragraph{Combined Dynamics.}
Combining the reward drift with each KL choice yields the following governed ODEs:
\begin{align}
\label{eq:mf-kl-two-class}
\dot p
&=
-\,\eta\,\frac{J}{\sigma(p)}\,[p(1-p)]^{2}\,\big(s_2+t_2\big)
\;-\;\beta\,p(1-p)\big(\ell-\ell_{\mathrm{ref}}\big)
&&\text{(two-class KL)}\\[2pt]
\label{eq:mf-kl-full}
\dot p
&=
-\,\eta\,\frac{J}{\sigma(p)}\,[p(1-p)]^{2}\,\big(s_2+t_2\big)
\;-\;\beta\,p(1-p)\Big(\ell-\ell_{\mathrm{ref}}
- D_{\mathrm{KL}}(y\|y^{\mathrm{ref}})
+ D_{\mathrm{KL}}(z\|z^{\mathrm{ref}})\Big)
&&\text{(full reverse-KL).}
\end{align}

\subsection{Nullcline, Interior Equilibrium, and the Prevention of Collapse}
\label{sec:nullcline-kl}

Using $s_2=\|y\|_2^2$, $t_2=\|z\|_2^2$, and $\ell(p)=\log\!\frac{p}{1-p}$, the mean-field ODE regularized by the two-class KL penalty is
\[
\dot p
=
p(1-p)\,\Big\{-\,\eta\,\frac{J}{\sigma(p)}\,p(1-p)\,\big(s_2+t_2\big)\;-\;\beta\big(\ell-\ell_{\mathrm{ref}}\big)\Big\}.
\]
This formulation allows us to characterize the stationary behavior of the system.

\begin{definition}[Nullcline]
The interior nullcline $\{\,\dot p=0\,\}$ on the interval $(0,1)$ is defined by the graph
\begin{equation}
\label{eq:nullcline}
\beta\big(\ell(p)-\ell_{\mathrm{ref}}\big)
\;=\;
-\,\eta\,\frac{J}{\sigma(p)}\,p(1-p)\,\big(s_2+t_2\big).
\end{equation}
\end{definition}

\begin{theorem}[Existence, uniqueness, and stability of the interior equilibrium]
\label{thm:unique-stable}
For any penalty strength $\beta>0$ and fixed within-block distributions $(y,z)$:
\begin{enumerate}
\item There exists a unique equilibrium $p^\star\in(0,1)$ satisfying \eqref{eq:nullcline}.
\item This equilibrium $p^\star$ is globally asymptotically stable on $(0,1)$.
\end{enumerate}
\end{theorem}
\emph{Proof sketch.} On the open interval $(0,1)$, the right-hand side of \eqref{eq:nullcline} is continuous and bounded. In contrast, the log-odds $\ell(p)$ is strictly increasing, with $\ell(p)\to-\infty$ as $p\downarrow0$ and $\ell(p)\to+\infty$ as $p\uparrow1$. Existence and uniqueness follow from the intermediate value theorem and monotonicity. Asymptotic stability is verified by noting that $\partial_p\dot p(p^\star)<0$ upon linearization.
\qed
\begin{corollary}[KL Regularization prevents collapse for $J<0$]
\label{cor:raise-accuracy}
When $J<0$, the right-hand side of \eqref{eq:nullcline} is strictly positive, implying $\ell(p^\star)>\ell_{\mathrm{ref}}$ and thus $p^\star>p_{\mathrm{ref}}$. While the unregularized dynamics ($\beta=0$) would drive $p(t)\uparrow 1$ (leading to a collapse onto the bad-arm block), any $\beta>0$ ensures the existence of a stable $p^\star < 1$. Consequently, the long-run accuracy $1-p^\star$ remains strictly positive and is necessarily higher than in the $\beta=0$ regime.
\end{corollary}

\paragraph{Asymptotic behavior of $p^\star$.}
Letting $\sigma_\star:=\sigma(p^\star)$, we examine the equilibrium under varying KL strengths:
\begin{itemize}
\item \textbf{Strong KL ($\beta\to\infty$):}
Expanding \eqref{eq:nullcline} around $p_{\mathrm{ref}}$, we find
\begin{equation}
p^\star
=
p_{\mathrm{ref}}
-\frac{\eta\,J}{\beta}\,
\frac{\big[p_{\mathrm{ref}}(1-p_{\mathrm{ref}})\big]^2}{\sigma(p_{\mathrm{ref}})}\,
\big(s_2+t_2\big)
\;+\;O\!\Big(\tfrac{1}{\beta^2}\Big).
\end{equation}
For $J<0$, the correction term is positive and scales as $O(\beta^{-1})$, indicating that $p^\star$ approaches $p_{\mathrm{ref}}$ from above as $\beta$ increases.

\item \textbf{Weak KL ($\beta\downarrow 0$, $J<0$):}
By setting $\varepsilon:=1-p$ and balancing leading-order terms near $p=1$, we obtain
\begin{equation}
1-p^\star
\;\sim\;
\frac{\beta}{c}\,\log\!\frac{c}{\beta},
\qquad
c:=-\frac{\,\eta\,J}{\sigma(1)}\,\big(s_2+t_2\big)\;>\;0.
\end{equation}
This confirms that any non-zero $\beta$ is sufficient to prevent total collapse ($p^\star < 1$).
\end{itemize}

\newpage
\section{Properties of the Simplex}
\label{sec:Shahshahani}

This section collects the geometric and algebraic facts about the probability simplex that we repeatedly use in the main text and appendix. Let $d \ge 2$ be an integer and let $\mathbf{1} \in \mathbb{R}^d$ denote the all-ones vector. We consider the probability simplex
\[
\Delta^{d-1}
\;:=\;
\big\{\p\in\mathbb{R}^d_{\ge 0}:\mathbf{1}^\top \p=1\big\},
\]
whose affine tangent space at any interior point $\p$ is
\[
T_{\p}\Delta^{d-1}
\;:=\;
\big\{\mathbf{v}\in\mathbb{R}^d:\mathbf{1}^\top \mathbf{v}=0\big\}.
\]
All definitions and statements below apply verbatim to any simplex-valued variable (e.g.\ $\p$ or $\y$) by renaming.

A recurring character in simplex geometry is the matrix
\[
\mathfrak{J}(\p)
\;:=\;
\operatorname{Diag}(\p)-\p\,\p^\top,
\]
which simultaneously plays three roles: (i) it is the Jacobian of the softmax map, (ii) it projects directions back to the simplex tangent space, and (iii) it is the inverse of the Shahshahani metric when restricted to the tangent. We build these facts in steps.

\begin{lemma}[Softmax differential is tangent; Jacobian form]
\label{lem:softmax-tangent}
Let $\p=\mathrm{softmax}(\boldsymbol\theta)\in\Delta^{d-1}$ with components $p_i(\boldsymbol\theta) = e^{\theta_i} / Z(\boldsymbol\theta)$, where $Z(\boldsymbol\theta) = \sum_{j=1}^d e^{\theta_j}$. Then
\[
\frac{\partial p_i}{\partial \theta_j}=p_i(\delta_{ij}-p_j)
\quad\Longrightarrow\quad
\mathrm{d}\p \;=\; \mathfrak{J}(\p)\,\mathrm{d}\boldsymbol\theta.
\]
Moreover, $\mathfrak{J}(\p)\mathbf{1}=0$ and $\mathbf{1}^\top \mathfrak{J}(\p)=0^\top$, so $\mathrm{Im}\,\mathfrak{J}(\p)\subseteq T_{\p}\Delta^{d-1}$; i.e., the softmax differential always lies in the tangent space.
\end{lemma}

Lemma~\ref{lem:softmax-tangent} already explains why replicator-like dynamics naturally appear when working in logits: any infinitesimal change in $\boldsymbol\theta$ is automatically mapped to a tangent direction in probability space via $\mathfrak{J}(\p)$.

\begin{corollary}[Rank, nullspace, and image]
\label{cor:Jfrak-rank-image}
When $\p$ lies in the interior ($p_i > 0$ for all $i$), the matrix $\mathfrak{J}(\p)$ is symmetric positive semidefinite and satisfies
\[
\mathrm{null}\big(\mathfrak{J}(\p)\big)=\mathrm{span}\{\mathbf{1}\},
\qquad  
\mathrm{rank}\big(\mathfrak{J}(\p)\big)=d-1,
\qquad  
\mathrm{Im}\,\mathfrak{J}(\p)=T_{\p}\Delta^{d-1}.
\]
In particular, $\mathfrak{J}(\p)$ acts as a linear automorphism of the tangent space.
\end{corollary}
The only ``forbidden'' direction is $\mathbf{1}$ (which corresponds to shifting all logits by a constant and hence does not change $\p$). On the tangent space, the directions that actually change probabilities, $\mathfrak{J}(\p)$ is full rank.

\begin{lemma}[A right inverse of $\mathfrak{J}$ on the tangent]
\label{lem:JfrakDinv-right-inverse}
For an interior point $\p$, and for any vector $\boldsymbol v \in T_\p\Delta^{d-1}$,
\[
\mathfrak{J}(\p)\,\operatorname{Diag}(\p)^{-1} \boldsymbol v
\;=\;
\big(I-\p\,\mathbf{1}^\top\big)\,\boldsymbol v
\;=\;
\boldsymbol v.
\]
Equivalently, $\operatorname{Diag}(\p)^{-1}$ behaves like $\mathfrak{J}(\p)^{-1}$ once we restrict attention to tangent directions.
\end{lemma}

\begin{proof}
By direct computation,
\[
\mathfrak{J}(\p)\operatorname{Diag}(\p)^{-1}
=
(\operatorname{Diag}(\p)-\p\p^\top)\operatorname{Diag}(\p)^{-1}
=
I-\p\,\mathbf{1}^\top,
\]
using $\p^\top\operatorname{Diag}(\p)^{-1}=\mathbf{1}^\top$. If $\boldsymbol v \in T_\p\Delta^{d-1}$, then $\mathbf{1}^\top \boldsymbol v=0$, so $(I-\p\,\mathbf{1}^\top)\boldsymbol v=\boldsymbol v$.
\end{proof}

Up to now, $\mathfrak{J}(\p)$ has appeared as a purely algebraic object (a Jacobian with a convenient nullspace).
Next we endow the simplex with a Riemannian metric for which $\mathfrak{J}(\p)$ becomes the natural ``inverse metric'' on the tangent space.

\begin{definition}[Shahshahani metric]
For an interior point $\p$, the Shahshahani inner product on $T_\p\Delta^{d-1}$ is defined as
\begin{equation}\label{eq:shah-metric}
\langle \boldsymbol u,\boldsymbol v\rangle_{\mathrm{Shah}}
\;:=\;
\boldsymbol u^\top \operatorname{Diag}(\p)^{-1} \boldsymbol v
=
\sum_{i=1}^d \frac{u_i v_i}{p_i},
\qquad
\boldsymbol u,\boldsymbol v \in T_{\p}\Delta^{d-1}.
\end{equation}
\end{definition}

\paragraph{Geometric and information-theoretic intuition.}
The Shahshahani metric rescales each coordinate by $1/\sqrt{p_i}$: moving a small component is ``expensive,'' and the induced norm
\[
\|\delta\|_{\p}^2=\langle \delta,\delta\rangle_{\mathrm{Shah}}
=\sum_{i=1}^d \frac{\delta_i^2}{p_i}
\]
measures \emph{relative} (multiplicative) change.
A convenient way to visualize geodesics is the square-root embedding $p \mapsto \sqrt{p}$ (componentwise), which maps the simplex interior to the positive orthant of the unit sphere. Under this embedding, Shahshahani geodesics become great-circle arcs, and the induced geodesic distance is the Bhattacharyya angle
\begin{equation}\label{eq:bhatt}
d_{\mathrm{Shah}}(x,y)
\;=\;
2\,\arccos\!\Big(\sum_{i=1}^d \sqrt{x_i y_i}\Big).
\end{equation}
This geometry is also tied to information theory: the metric tensor is the Hessian of the convex potential $\psi(\p)=\sum_{i=1}^d p_i\log p_i$ (negative Shannon entropy),
\begin{equation}\label{eq:hessian-entropy}
\nabla^2\psi(\p)\;=\;\operatorname{Diag}\!\Big(\tfrac{1}{p_1},\dots,\tfrac{1}{p_d}\Big),
\end{equation}
so the forward KL divergence has the local quadratic expansion
\begin{equation}\label{eq:kl-quadratic}
D_{\mathrm{KL}}(\p+\delta\,\|\,\p)
\;=\;
\tfrac12\,\delta^\top \operatorname{Diag}(\p)^{-1}\delta \;+\; o(\|\delta\|^2),
\qquad \mathbf{1}^\top\delta=0,
\end{equation}
which is precisely why this metric underlies mirror descent and multiplicative-weights updates.

Once a metric is fixed, gradients become metric-dependent: the Shahshahani gradient is the unique tangent vector whose inner product against any tangent direction matches the directional derivative.

\paragraph{Riemannian gradient under the Shahshahani metric.}
For a function $F:\Delta^{d-1}\to\mathbb{R}$, the Shahshahani gradient $\mathrm{grad}_{\mathrm{Shah}}F(\p)\in T_{\p}\Delta^{d-1}$ is defined by
\[
\langle \mathrm{grad}_{\mathrm{Shah}}F(\p),\boldsymbol u\rangle_{\mathrm{Shah}}
\;=\;
\nabla F(\p)^\top \boldsymbol u,
\qquad
\forall\,\boldsymbol u\in T_{\p}\Delta^{d-1},
\]
where $\nabla F(\p)$ denotes the Euclidean gradient in the ambient space.

The next corollary shows the pleasant surprise: under the Shahshahani metric, the gradient is obtained by applying $\mathfrak{J}(\p)$ to the Euclidean gradient, so the same matrix that appears in the softmax Jacobian also governs natural-gradient flow on the simplex.

\begin{corollary}[Natural gradient on the simplex]
\label{cor:nat-grad-frak}
If $F:\Delta^{d-1}\to\mathbb{R}$ is $C^1$ and $\p$ is interior, then
\[
\mathrm{grad}_{\mathrm{Shah}}F(\p)
\;=\;
\mathfrak{J}(\p)\,\nabla F(\p)
\;=\;
\p\odot\big(\nabla F(\p)-\langle \p,\nabla F(\p)\rangle\,\mathbf{1}\big)
\in T_\p\Delta^{d-1}.
\]
\end{corollary}

\begin{proof}
Let $g := \mathfrak{J}(\p)\nabla F(\p)$. From Lemma~\ref{lem:softmax-tangent}, we have $g \in T_\p\Delta^{d-1}$. For any $\boldsymbol u \in T_\p\Delta^{d-1}$, symmetry of $\mathfrak{J}(\p)$ and Lemma~\ref{lem:JfrakDinv-right-inverse} yield
\[
\langle g,\boldsymbol u\rangle_{\mathrm{Shah}}
=
g^\top \operatorname{Diag}(\p)^{-1}\boldsymbol u
=
\nabla F(\p)^\top \mathfrak{J}(\p)\operatorname{Diag}(\p)^{-1}\boldsymbol u
=
\nabla F(\p)^\top \boldsymbol u.
\]
Thus $g$ satisfies the defining property of $\mathrm{grad}_{\mathrm{Shah}}F(\p)$, proving the claim. The componentwise form follows by expanding $\mathfrak{J}(\p)\nabla F(\p)=\operatorname{Diag}(\p)\nabla F(\p)-\p(\p^\top \nabla F(\p))$.
\end{proof}

The following lemma states that $\mathfrak{J}(\p)$ is strictly positive on tangent directions, so it can safely serve as an inverse metric (and as a preconditioner) as long as we stay in the simplex interior.

\begin{lemma}[Positive definiteness of $\mathfrak{J}$ on the tangent]
\label{lem:Jfrak-spd}
For any interior $\p$ and any nonzero $\boldsymbol v \in T_\p\Delta^{d-1}$,
\[
\boldsymbol v^\top \mathfrak{J}(\p)\,\boldsymbol v
\;=\;
\sum_{i=1}^d p_i v_i^2 - \Big(\sum_{i=1}^d p_i v_i\Big)^2
\;=\;
\mathrm{Var}_{i\sim p}(v_i)
\;>\; 0.
\]
Thus, $\mathfrak{J}(\p)$ is symmetric positive definite when restricted to the tangent space.
\end{lemma}

\paragraph{Remark on boundary points.}
If some components satisfy $p_i=0$, the statements remain valid after restricting to the support of $\p$ and the corresponding lower-dimensional face (where $\operatorname{Diag}(\p)^{-1}$ is well-defined).

So far we have described the geometry (metric) and the resulting notion of gradient (natural gradient). We now connect this viewpoint to the standard KL-regularized update used by mirror descent / multiplicative weights, and show that it matches the Shahshahani natural-gradient flow to first order.


\begin{proposition}[Entropic mirror-ascent on the simplex]
\label{prop:mirror-ascent-detailed}
For $\p \in \Delta^{d-1}$, a vector $A \in \mathbb{R}^d$, and a step size $\eta > 0$, the KL-regularized maximization problem
\begin{equation}
\label{eq:mirror-problem}
\p^{+}\;=\;\arg\max_{\q\in\Delta^{d-1}}
\Big\{\;\langle A,\q\rangle \;-\; \tfrac{1}{\eta}\,D_{\mathrm{KL}}(\q\;\|\;\p)\;\Big\}
\end{equation}
satisfies the following properties:
\begin{enumerate}[leftmargin=1.25em,label=(\alph*)]
\item \textbf{Existence and uniqueness.}
The objective is strictly concave on the relative interior of the face determined by $\p$, hence the maximizer $\p^{+}$ exists and is unique.

\item \textbf{Closed-form update (multiplicative weights).}
Let $Z := \sum_{j=1}^d p_j \exp(\eta A_j)$. Then
\begin{equation}
\label{eq:eg-update}
p_i^{+}\;=\;\frac{p_i\,\exp(\eta A_i)}{Z},\qquad i=1,\dots,d.
\end{equation}
Equivalently, $\log p_i^{+} = \log p_i + \eta A_i - \log Z$.

\item \textbf{Trust-region equivalence.}
For any $\rho > 0$, there exists $\lambda>0$ such that $\p^{+}$ also solves $\max_{\q\in\Delta^{d-1}}\{\langle A,\q\rangle:\ D_{\mathrm{KL}}(\q\|\p)\le \rho\}$ with $\eta=1/\lambda$.

\item \textbf{Optimal value.}
The maximum value of the objective equals $\tfrac{1}{\eta}\log Z$.

\item \textbf{Invariance and support.}
The update is invariant under shifts $A \mapsto A + c\mathbf{1}$ and does not create new support in one step.

\item \textbf{Improvement via Jeffreys divergence.}
With $\bar A:=\langle \p,A\rangle$ and $J_{\mathrm{KL}}(\p^{+},\p):=D_{\mathrm{KL}}(\p^{+}\|\p)+D_{\mathrm{KL}}(\p\|\p^{+})$,
\[
\eta\langle A, \p^{+} - \p \rangle = J_{\mathrm{KL}}(\p^{+}, \p) \ge 0.
\]

\item \textbf{First-order expansion (replicator direction).}
For small $\eta$,
\begin{equation}
\label{eq:first-order-replicator}
\p^{+}-\p \;=\; \eta\,\mathfrak{J}(\p)\,A \;+\; O(\eta^2).
\end{equation}

\item \textbf{Local objective gain.}
$\langle A, \p^{+} - \p \rangle = \eta\,\mathrm{Var}_{i\sim p}(A_i) + O(\eta^2)$.
\end{enumerate}
\end{proposition}

\begin{proof}
These claims follow from standard Lagrangian optimality conditions and a Taylor expansion of \eqref{eq:eg-update}. In particular, stationarity yields $q_i \propto p_i e^{\eta A_i}$ and hence \eqref{eq:eg-update}, while the Jeffreys identity follows by summing $D_{\mathrm{KL}}(\p^{+}\|\p)$ and $D_{\mathrm{KL}}(\p\|\p^{+})$.
\end{proof}

\paragraph{Closing the loop.}
The final corollary makes the connection explicit: mirror-ascent is an Euler discretization of Shahshahani natural-gradient flow, which is exactly the replicator-form dynamics that appear in our mean-field ODEs.

\begin{corollary}[Natural-gradient interpretation]
\label{cor:mirror-natgrad}
The mirror-ascent step \eqref{eq:mirror-problem} is equivalent, to first order, to an Euler step of the Shahshahani natural-gradient flow:
\[
\dot \p = \mathfrak{J}(\p)A
=
\p\odot\big(A-\langle \p,A\rangle\,\mathbf{1}\big).
\]
\end{corollary}

\newpage
\section{Inner Dynamics of the Good Arms}
\label{sec:inner-good}

In Appendix~\ref{appendix:Group-Normalized RL with Noisy Reward}, we established the mean-field
dynamics of the good and bad blocks in the explicit $(K{+}M)$-arm model (with $K$ good arms and
$M$ bad arms), and in particular we introduced the within-block coordinates
\[
p(t):=\sum_{m=1}^M p_{b_m}(t)\in[0,1],
\qquad
y_j(t):=\frac{p_j(t)}{1-p(t)}\ (j\le K),
\qquad
z_m(t):=\frac{p_{b_m}(t)}{p(t)}\ (m\le M).
\]
In this section, we go into the details of the \emph{inner good-arm} dynamics, i.e.\ the evolution of
$y(t)\in\Delta^{K-1}$. (The corresponding inner bad-arm dynamics for $z(t)$ is the sign-reversed analogue
and is recorded separately.)

\medskip
To start, we re-derive the inner good dynamics ODE in a more intuitive way.
Let $p_i=\exp(\theta_i)/Z$ with
\[
Z=\sum_{k=1}^K \exp(\theta_k)+\sum_{m=1}^M \exp(\theta_{b_m}),
\qquad
p:=\sum_{m=1}^M p_{b_m},
\]
and define
\[
y_j \;:=\; \frac{p_j}{1-p},\qquad j\in[K].
\]
Then
\begin{equation}
\label{eq:y-softmax-good}
y_j \;=\; \frac{\exp(\theta_j)}{\sum_{k=1}^K \exp(\theta_k)}
\;=\; \big(\mathrm{softmax}(\boldsymbol\theta_{\mathrm{good}})\big)_j,
\end{equation}
so $y=(y_1,\cdots, y_K)$ depends only on $\boldsymbol\theta_{\mathrm{good}}$ (the bad-block logits cancel by normalization).

\begin{lemma}[Pushforward from logits to within-good composition]
\label{lem:pushforward-y}
For any small increment $\Delta\boldsymbol\theta=(\Delta\boldsymbol\theta_{\mathrm{good}},\Delta\boldsymbol\theta_{\mathrm{bad}})$,
\begin{equation}
\label{eq:push-y}
\Delta y
\;=\;
\big(\mathrm{Diag}(y)-yy^\top\big)\,\Delta\boldsymbol\theta_{\mathrm{good}}
\;=\;
y\odot\Big(\Delta\boldsymbol\theta_{\mathrm{good}}-\langle y,\Delta\boldsymbol\theta_{\mathrm{good}}\rangle\,\mathbf 1\Big),
\end{equation}
and in particular $\partial y/\partial \theta_{b_m}=\mathbf 0$ for all $m\in[M]$
(equivalently, $\partial y/\partial \boldsymbol\theta_{\mathrm{bad}}=\mathbf 0$).

\emph{Proof.} From \eqref{eq:y-softmax-good}, write $y_j=\exp(\theta_j-L)$ with $L:=\log\sum_{k\le K}\exp(\theta_k)$, so
$\Delta y_j = y_j(\Delta\theta_j-\Delta L)$ and $\Delta L=\sum_{k\le K} y_k\Delta\theta_k$.
Stacking over $j$ gives \eqref{eq:push-y}. \qedhere
\end{lemma}


Similar to our discussion in Appendix~\ref{appendix:Group-Normalized RL with Noisy Reward}, assume block symmetry with
$A_j=a_{\mathrm g}(p)$ for $j\le K$ and $A_{b_m}=a_{\mathrm b}(p)$ for $m\le M$, and set
$\Delta r(p)=a_{\mathrm b}(p)-a_{\mathrm g}(p)$.
The expected logit step in the good block is
\begin{equation}
\label{eq:good-logit-drift}
\mathbb E[\Delta\boldsymbol\theta_{\mathrm{good}}]
\;=\;
\kappa(p)\,y,\qquad
\kappa(p):=-\,\eta\,p(1-p)\,\Delta r(p),
\end{equation}
so there is no arm-specific preference in $\theta$-space.
Applying Lemma~\ref{lem:pushforward-y} leads to
\begin{equation}
\label{eq:y-drift-euclid}
\mathbb E[\Delta y]
\;=\; \kappa(p)\,\Big(y\odot y-\|y\|_2^2\,y\Big),
\qquad
\mathbb E[\Delta y_j]
\;=\; \kappa(p)\,y_j\big(y_j-\|y\|_2^2\big).
\end{equation}
If $a_{\mathrm g}(p)=\frac{J\,p}{\sigma(p)}$ and $a_{\mathrm b}(p)=-\frac{J(1-p)}{\sigma(p)}$
then $\Delta r(p)=-J/\sigma(p)$ and \eqref{eq:y-drift-euclid} becomes
\[
\mathbb E[\Delta y_j]
\;=\;\eta\,\frac{J}{\sigma(p)}\,p(1-p)\;y_j\Big(y_j-\|y\|_2^2\Big),
\]

\emph{Consequences.} If $J>0$, arms with $y_j>\|y\|_2^2$ grow while those with $y_j<\|y\|_2^2$ shrink
(deterministic sharpening within the good block).
The fixed points in $y$ are the uniform point and the vertices.
If $\Delta r(p)<0$ (informative grader favoring the good block), the uniform point is unstable and the vertices are attracting.\footnote{
It is interesting to notice that if one uses the algorithm that uses natural-gradient/replicator flow instead,
$\dot{\mathbf p}=\eta\,\mathfrak J(\mathbf p)\mathbf A$ with $\mathfrak J(\mathbf p)=\mathrm{Diag}(\mathbf p)-\mathbf p\mathbf p^\top$, then under block symmetry
\[
\dot y_j \;=\; y_j\big(a_{\mathrm g}-\bar A_{\mathrm g}\big)\;=\;0,
\]
so $y$ is exactly drift-free deterministically and only sampling noise perturbs it.}

\begin{lemma}[Geometry and Lyapunov structure]
Consider the ODE on the simplex
\[
\dot y \;=\; \kappa(t)\,\Big(y\odot y - \|y\|_2^2\,y\Big),\qquad 
y=(y_1,\dots,y_K)\in\Delta^{K-1}:=\{y\ge0,\ \sum_i y_i=1\}.
\]
Let $\tau(t):=\int_0^t \kappa(s)\,ds$ and write $y'=\frac{dy}{d\tau}$.
(E.g., here in our noisy GRPO dynamics, $\kappa(t)\propto \frac{J}{\sigma(p(t))}p(t)(1-p(t))$ up to the chosen mean-field time scaling.)
\begin{enumerate}
\item[\normalfont(1)] \textbf{Simplex invariance.} $\sum_i \dot y_i=0$, and if $y_i(0)\ge0$ then $y_i(t)\ge0$ for all $t$. Hence $\Delta^{K-1}$ is forward invariant.
\item[\normalfont(2)] \textbf{Gradient form.} Define
\[
\mathcal L(y):=\frac{1}{3}\sum_{i=1}^K y_i^3-\frac{1}{4}\Big(\sum_{i=1}^K y_i^2\Big)^2=\frac{1}{3}s_3-\frac{1}{4}s_2^2,
\]
such that $s_2(t):=\|y(t)\|_2^2=\sum_i y_i^2$ and $s_3(t):=\sum_i y_i^3$,
then $\nabla\mathcal L(y)=(y_i^2-\|y\|_2^2\,y_i)_i$ and
\[
\dot y=\kappa(t)\,\nabla\mathcal L(y),\qquad 
\frac{d}{dt}\mathcal L\big(y(t)\big)=\kappa(t)\,\|\nabla\mathcal L(y)\|_2^2.
\]
In particular, when $J>0$, then $\kappa(t)\ge0$ and the flow monotonically ascends $\mathcal L$ (and descends it when $J<0$).
\item[\normalfont(3)] \textbf{Monotone concentration.} 
\[
\dot s_2 \;=\; 2\kappa(t)\Big(\sum_i y_i^3 - s_2^2\Big)\;\ge\;0\quad\text{whenever }\kappa(t)\ge0,
\]
because by Cauchy--Schwarz, $\sum_i y_i^3\ge(\sum_i y_i^2)^2$ on the simplex, with equality iff $y$ is uniform on its support. 
Hence for $\kappa\ge0$ the mass generically concentrates while when $\kappa<0$, $s_2$ decreases. 
\item[\normalfont(4)] \textbf{Equilibria.} Stationary points satisfy $y_i\in\{0,s_2\}$ for all $i$. Thus for any $m\in\{1,\dots,K\}$, the points with exactly $m$ nonzero coordinates all equal to $1/m$ are equilibria (uniform on a support of size $m$). For $\kappa>0$, the $m=1$ vertices are (Lyapunov) attractors and the others are saddles. 
\end{enumerate}
\end{lemma}

\begin{remark}[Time reparametrization]
All phase‑portrait statements are independent of $t$ and depend only on the internal time $\tau$.
In $\tau$‑time the ODE is autonomous:
\begin{equation}\label{eq:GD}
    \frac{dy}{d\tau}=y\odot y-\|y\|_2^2\,y=y\odot y-s_2y.
\end{equation}
Solutions in physical time are obtained by composing with $\tau(t)$.\footnote{ We can see how the closed form solution of this ODE look like for simple examples. 
\emph{A: Closed form on the 2‑arm face ($K=2$)}.
On the line $\{(q,1-q):\,q\in[0,1]\}$ one has
$
\frac{dq}{d\tau}=q(1-q)\,(2q-1)
$. The Separating variables gives
\(
\int\frac{dq}{q(1-q)(2q-1)}=\tau+C
\ \Longrightarrow\ 
\log\!\Big(\frac{(2q-1)^2}{q(1-q)}\Big)=\tau+\log G_0,
\)
where $G_0=\frac{(2q_0-1)^2}{q_0(1-q_0)}$.
Hence the explicit solution
\(
\ q(\tau)=\frac12\!\left(1 \pm \sqrt{\frac{G_0\,e^{\tau}}{\,4+G_0\,e^{\tau}\,}}\right)\ .
\)
Choose $+$ if $q_0>\tfrac12$ (converges to $1$ for $\kappa>0$) and $-$ if $q_0<\tfrac12$.

Similarity, for \emph{B. One‑vs‑rest symmetric slice for general $K\ge2$ (full partial fractions)}.
Restrict to the $1$‑D symmetric manifold
\(
y(\tau)=\Big(x(\tau),\,\tfrac{1-x(\tau)}{K-1},\ldots,\tfrac{1-x(\tau)}{K-1}\Big),
\qquad x\in[0,1].
\)
Then
\[
\|y\|_2^2 \;=\; x^2+\frac{(1-x)^2}{K-1},\qquad
\frac{dx}{d\tau} \;=\; x^2 - x\|y\|_2^2 
\;=\; \frac{1}{K-1}\,x(1-x)\,(Kx-1).
\]
Separate variables and partial‑fractions decomposition gives:
\(
\ 
\frac{(Kx(\tau)-1)^{K}}{x(\tau)^{K-1}\,\bigl(1-x(\tau)\bigr)}
\;=\;
\frac{(Kx_0-1)^{K}}{x_0^{K-1}\,(1-x_0)}\;e^{\tau}
\ :=\ G_0\,e^{\tau}.\
\) with fixing the constant with $x(0)=x_0\in(0,1)$.
Equilibria on this slice are $x\in\{0,1/K,1\}$. For $\kappa>0$, $x(\tau)\to0$ if $x_0<1/K$, and $x(\tau)\to1$ if $x_0>1/K$.

For $K=3$ the invariant in (B) 
the explicit solution is in the form of 

\( 
x(\tau)=\frac{1}{3}\!\left(
1 + \sqrt[3]{\,z(\tau)\bigl(1+\sqrt{1-z(\tau)}\bigr)\,}
      + \sqrt[3]{\,z(\tau)\bigl(1-\sqrt{1-z(\tau)}\bigr)\,}
\right),
\quad 
z(\tau)=\frac{G_0\,e^{\tau}}{27+G_0\,e^{\tau}}.\
\)
}

Solving the system from \eqref{eq:GD}:
$$
y_i(\tau) = \frac{\frac{d\gamma}{d\tau}}{\frac{1}{y_i(\tau)}-\gamma(\tau)}
$$
where $\gamma(\tau)$ satisfies $\gamma(0)=0$ and 
$$
\prod_{i=1}^K\left(\frac{1}{y_i(0)}-\gamma(\tau)\right) = \frac{1}{\prod_{i=1}^Ky_i(0)}e^{-\tau}.
$$
By the Implicit Function Theorem, there exists a unique strictly increasing function $\gamma: [0,\infty)\rightarrow \big[0,\frac{1}{\max_{1\leq i\leq K}y_i(0)}\big)$ with these properties, and it tends to $\frac{1}{\max_{1\leq i\leq K}y_i(0)}$ as $\tau\to\infty$. In particular, assuming we have $m$ highest initial values among $y_1(0),\dots,y_K(0)$,  as $\tau\to\infty$, $y_i(\tau)$'s with highest initial values tend to $\frac{1}{m}$ and the rest tend to $0$. 
\end{remark}
We will continue the dynamics of good arms in more details in the following subsection.~\ref{subsection:Dynamics-of-y}
\subsection{Evolution of Collision term, $s_2$}

\begin{lemma}[Evolution and bounds for $s_2=\|y\|_2^2$]\label{lem:s2}
In $\tau$-time one has
\[
\frac{d}{d\tau}s_2 \;=\; 2\bigl(s_3-s_2^2\bigr),
\qquad
\frac{d}{dt}s_2 \;=\; 2\,\kappa(t)\,\bigl(s_3-s_2^2\bigr).
\]
Consequently:
\begin{enumerate}
\item (\emph{Monotonicity}) On the simplex, $s_3\ge s_2^2$ with equality iff $y$ is uniform on its support. Hence $s_2(\tau)$ is nondecreasing (strictly, away from uniform-on-support points) when $\kappa\ge0$.
\item (\emph{Range}) $\displaystyle \frac{1}{K}\le s_2(\tau)\le 1$.
In the multi-bad setting, defining $t_2(\tau):=\|z(\tau)\|_2^2\in[\frac1M,1]$, we have the uniform bound
\[
\frac{1}{K}+\frac{1}{M}\ \le\ s_2(\tau)+t_2(\tau)\ \le\ 2.
\]
\item (\emph{Logistic upper differential}) Since $0\le y_i\le1$,
$s_3\le s_2$, hence
\[
\frac{d}{d\tau}s_2 \;\le\; 2\,s_2(1-s_2),
\quad\Rightarrow\quad
s_2(\tau)\ \le\ \frac{1}{1+\bigl(\frac{1-s_2(0)}{s_2(0)}\bigr)e^{-2\tau}}.
\]
Integrating,
\[
\int_0^\tau s_2(u)\,du
\;\le\;
\frac{1}{2}\,\log\!\Bigl(1+s_2(0)\,\bigl(e^{2\tau}-1\bigr)\Bigr).
\]
\end{enumerate}
\end{lemma}

\begin{corollary}[Exact logit representation and envelopes for $p$ (multi-bad setting)]
\label{cor:logit}
Let $p(\tau)\in(0,1)$ be the \emph{total} bad mass and let $z(\tau)\in\Delta^{M-1}$ be the within-bad composition,
so $t_2(\tau):=\|z(\tau)\|_2^2\in[1/M,1]$.
In $\tau$-time the $p$-equation reads
\[
\frac{dp}{d\tau}\;=\;-\,p(1-p)\,\bigl(s_2(\tau)+t_2(\tau)\bigr).
\]
For the logit $L(\tau):=\log\!\frac{p(\tau)}{1-p(\tau)}$,
\[
\frac{dL}{d\tau}=-(s_2(\tau)+t_2(\tau)),
\qquad
L(\tau)=L(0)-\int_0^\tau \bigl(s_2(u)+t_2(u)\bigr)\,du.
\]
Because $s_2\in\bigl[\frac1K,1\bigr]$ and $t_2\in\bigl[\frac1M,1\bigr]$, comparison yields the envelopes
\[
\frac{p_0}{1-p_0}\,e^{-2\tau}
\;\le\;
\frac{p(\tau)}{1-p(\tau)}
\;\le\;
\frac{p_0}{1-p_0}\,e^{-(\frac1K+\frac1M)\tau},
\]
equivalently
\[
\
\frac{1}{1+\frac{1-p_0}{p_0}\,e^{2\tau}}
\;\le\;
p(\tau)
\;\le\;
\frac{1}{1+\frac{1-p_0}{p_0}\,e^{(\frac1K+\frac1M)\tau}}\
.
\]
\end{corollary}

\begin{corollary}[Hitting-time bracket in internal time (multi-bad setting)]
\label{cor:hitting}
Fix $p_\star\in(0,1)$. The internal time to reach $p(\tau_\star)=p_\star$ is bounded by
\[
\
\frac{1}{2}\,\log\!\frac{p_0(1-p_\star)}{(1-p_0)p_\star}
\;\le\;
\tau_\star
\;\le\;
\frac{1}{\,\frac1K+\frac1M\,}\,\log\!\frac{p_0(1-p_\star)}{(1-p_0)p_\star}\ .
\]
Thus the factor $\|y\|_2^2+\|z\|_2^2$ accelerates the decay of $\logit p$ by a multiplicative factor between
$\frac1K+\frac1M$ (near-uniform within both blocks) and $2$ (maximally concentrated within both blocks).
\end{corollary}

\begin{theorem}[General-$(K,M)$ small-heterogeneity expansion]\label{thm:etaExpansion}
Let $\unif_K=(1/K,\dots,1/K)$ and $\unif_M=(1/M,\dots,1/M)$ and write
\[
y=\unif_K+\vv,\quad \sum_{i=1}^K v_i=0,
\qquad
z=\unif_M+\ww,\quad \sum_{m=1}^M w_m=0.
\]
Define the heterogeneities
\[
\zeta(\tau):=\|\vv(\tau)\|_2^2
\;=\; s_2(\tau)-\frac{1}{K}\;\ge0,
\qquad \zeta_0:=\zeta(0),
\]
\[
\xi(\tau):=\|\ww(\tau)\|_2^2
\;=\; t_2(\tau)-\frac{1}{M}\;\ge0,
\qquad \xi_0:=\xi(0).
\]
Then in internal time $\tau$,
\[
\frac{dL}{d\tau}=-(s_2(\tau)+t_2(\tau))=-(\tfrac1K+\tfrac1M+\zeta(\tau)+\xi(\tau)),
\quad
L(\tau)=L(0)-\int_0^\tau \bigl(s_2(u)+t_2(u)\bigr)\,du.
\]
Moreover:
\begin{enumerate}
\item[\normalfont(i)] \textbf{Linearized $\zeta$-law (good block).}
One has the identity
\[
\zeta'(\tau)=\frac{2}{K}\,\zeta(\tau)\;-\;2\,\zeta(\tau)^2\;+\;2\sum_{i=1}^K v_i(\tau)^3,
\]
and the bound $\bigl|\sum_i v_i^3\bigr|\le \sum_i |v_i|^3\le \|\vv\|_2^3=\zeta^{3/2}$.
Hence, uniformly while $\zeta(\tau)\le1$,
\[
\zeta'(\tau)=\frac{2}{K}\,\zeta(\tau)\;+\;O\!\big(\zeta(\tau)^{3/2}\big).
\]

\item[\normalfont(ii)] \textbf{Linearized $\xi$-law (bad block).}
Because the within-bad dynamics are the sign-reversed collision flow, one analogously has
\[
\xi'(\tau)=-\frac{2}{M}\,\xi(\tau)\;+\;2\,\xi(\tau)^2\;-\;2\sum_{m=1}^M w_m(\tau)^3,
\]
and $\bigl|\sum_m w_m^3\bigr|\le \|\ww\|_2^3=\xi^{3/2}$. Hence, uniformly while $\xi(\tau)\le1$,
\[
\xi'(\tau)=-\frac{2}{M}\,\xi(\tau)\;+\;O\!\big(\xi(\tau)^{3/2}\big).
\]

\item[\normalfont(iii)] \textbf{Asymptotic forms.}
There exist constants $C_K,C_M>0$ (depending only on $K$ and $M$) such that, for all $\tau\ge0$ with
$\sqrt{\zeta_0}\,e^{\tau/K}\le \tfrac12$,
\[
\quad
\zeta(\tau)=\zeta_0\,e^{\frac{2}{K}\tau}\;+\;R_\zeta(\tau),
\qquad |R_\zeta(\tau)|\le C_K\,\zeta_0^{3/2}\,e^{\frac{3}{K}\tau}, \quad
\]
and, for all $\tau\ge0$ with $\sqrt{\xi_0}\le\tfrac12$,
\[
\quad
\xi(\tau)=\xi_0\,e^{-\frac{2}{M}\tau}\;+\;R_\xi(\tau),
\qquad |R_\xi(\tau)|\le C_M\,\xi_0^{3/2}\,e^{-\frac{3}{M}\tau}. \quad
\]

\item[\normalfont(iv)] \textbf{Impact on the $p$-logit.}
Consequently,
\[
\int_0^\tau s_2(u)\,du
=\frac{\tau}{K}
+\frac{K}{2}\,\zeta_0\Big(e^{\frac{2}{K}\tau}-1\Big)
+ R_{I,y}(\tau),
\qquad |R_{I,y}(\tau)|\le C_{K}'\,\zeta_0^{3/2}\,e^{\frac{3}{K}\tau},
\]
\[
\int_0^\tau t_2(u)\,du
=\frac{\tau}{M}
+\frac{M}{2}\,\xi_0\Big(1-e^{-\frac{2}{M}\tau}\Big)
+ R_{I,z}(\tau),
\qquad |R_{I,z}(\tau)|\le C_{M}'\,\xi_0^{3/2},
\]
and hence
\[
\quad
L(\tau)=L(0)-\Big(\tfrac1K+\tfrac1M\Big)\tau
-\frac{K}{2}\,\zeta_0\Big(e^{\frac{2}{K}\tau}-1\Big)
-\frac{M}{2}\,\xi_0\Big(1-e^{-\frac{2}{M}\tau}\Big)
+ R_L(\tau),
\]
with $|R_L(\tau)|\le C_{K,M}''\big(\zeta_0^{3/2}e^{3\tau/K}+\xi_0^{3/2}\big)$ for a constant $C_{K,M}''$.
\end{enumerate}
In particular, to \emph{first order} the only dependence on the initial within-block states is through
$\zeta_0=\|y(0)-\unif_K\|_2^2$ and $\xi_0=\|z(0)-\unif_M\|_2^2$; finer details of $y(0)$ and $z(0)$ enter only at order
$O(\zeta_0^{3/2})$ and $O(\xi_0^{3/2})$ and higher.
\end{theorem}

\begin{corollary}[Small-$\tau$ expansion; “only $(\zeta_0,\xi_0)$ matters” at first order]
\label{cor:smallTau}
Expanding the expression in Theorem~\ref{thm:etaExpansion} for small $\tau$ gives
\[
L(\tau)
= L(0) - \Big(\tfrac1K+\tfrac1M\Big)\tau
- (\zeta_0+\xi_0)\,\tau
- \Big(\frac{\zeta_0}{K}-\frac{\xi_0}{M}\Big)\tau^2
\;+\;O\!\big(\zeta_0\,\tau^3\big)\;+\;O\!\big(\xi_0\,\tau^3\big)
\;+\;O\!\big(\zeta_0^{3/2}\,\tau\big)\;+\;O\!\big(\xi_0^{3/2}\,\tau\big).
\]
Thus, to linear order in $\tau$, the correction to the baseline $-(\frac1K+\frac1M)\tau$ is exactly $-(\zeta_0+\xi_0)\tau$.
Equivalently, at leading order,
\[
\frac{dL}{d\tau}=-(s_2(0)+t_2(0)) + \text{higher-order corrections}.
\]
\end{corollary}

\begin{remark}[Scope and sign]
(i) The $\tau$-phase portraits for $y$ are independent of the path of $p(t)$ (since $p$ only reparametrizes time through $\tau$).  
(ii) If $\kappa(t)\le0$ (e.g.\ $J/\sigma(p)\le0$), replace $\tau$ by $-\lvert\tau\rvert$ to flip directions in physical time; the internal-time identities remain valid.  
(iii) The window $\sqrt{\zeta_0}\,e^{\tau/K}\le \tfrac12$ describes the regime where the linearized law for $\zeta$ dominates; beyond it, the global envelopes in Cor.~\ref{cor:logit} apply.
\end{remark}

\subsection{Asymptotic General inner dynamics  behavior}
\label{subsec:general-case-inner}

\begin{lemma}[Order preservation and winner identity]
\label{lem:order-preservation}
For any \(i\neq j\), define \(\delta_{ij}(\tau):=y_i(\tau)-y_j(\tau)\). Along $\frac{dy}{d\tau}=y\odot y-s_2y$,
\[
\delta_{ij}' \;=\; \delta_{ij}\,\big(y_i+y_j-s_2\big),
\qquad
\Rightarrow\qquad
\delta_{ij}(\tau)=\delta_{ij}(0)\,\exp\!\Big(\!\int_0^\tau (y_i+y_j-s_2)\,du\Big).
\]
Hence \(\mathrm{sign}\,\delta_{ij}(\tau)=\mathrm{sign}\,\delta_{ij}(0)\) for all \(\tau\). In particular,
\[
m \;:=\; \arg\max_{1\le i\le K} y_i(0)
\]
remains the unique maximizer for all \(\tau>0\), and each hyperplane \(\{y_i=y_j\}\) is invariant.
\end{lemma}

\begin{theorem}[Global convergence and generic basins]
\label{thm:global}
Under~\eqref{eq:GD}, the trajectory stays in \(\Delta^{K-1}\) and \(\mathcal L\) is a strict Lyapunov function.
The equilibria are
\[
\mathcal E=\bigcup_{m=1}^K \mathcal E_m,\qquad
\mathcal E_m:=\left\{y:\ y_{i_1}=\cdots=y_{i_m}=1/m,\ y_j=0\ \text{else}\right\}.
\]
For generic initial conditions (no coordinate ties), the trajectory converges to the vertex \(\mathbf e_{m}\)
selected by Lemma~\ref{lem:order-preservation}. The non-vertex equilibria (uniform on \(m\)-subsupports with \(m\ge2\)) are saddles with co-dimension one stable manifolds coinciding with unions of the tie sets \(\{y_i=y_j\}\).
\end{theorem}

\begin{proposition}[Exponential polarization in internal time]
\label{prop:exp-tau}
Let \(m=\arg\max y_i(0)\) and write \(\varepsilon_i(\tau):=y_i(\tau)\) for \(i\ne m\). Linearizing~\eqref{eq:GD} at the vertex \(\mathbf e_m\) yields
\[
\varepsilon_i' \;=\; -\,\varepsilon_i \;+\; O(\varepsilon^2),
\qquad
1-y_m \;=\; \sum_{i\ne m}\varepsilon_i.
\]
Hence there exists \(\tau_0\) and constants \(c_i>0\) such that, for all \(\tau\ge\tau_0\),
\[
y_i(\tau)=c_i\,e^{-\tau}\big(1+o(1)\big),\qquad
1-y_m(\tau)=\Big(\sum_{i\ne m} c_i\Big)e^{-\tau}\big(1+o(1)\big),\qquad
1-s_2(\tau)=\Theta\!\big(e^{-\tau}\big).
\]
\end{proposition}

\begin{remark}[Sharper monotonicity identity]
\label{rem:s2-identity}
The variance form
\[
s_3-s_2^2=\sum_{i=1}^K y_i\,(y_i-s_2)^2 \ \ge\ 0
\]
(with equality iff \(y\) is uniform on its support) implies
\(s_2'(\tau)=2\big(s_3-s_2^2\big)\ge0\), with strict increase away from the saddle sets.
\end{remark}

\subsection{Stability of the Within--Good Equilibria}
\label{sec:stability-within-good}

Consider our the inner-good ODE
\begin{equation}
\label{eq:inner-ode-stability}
\dot y
\;=\;
\kappa(p)\,\big(y\odot y - s_2\,y\big),
\qquad
s_2=\|y\|_2^2,
\qquad
\kappa(p)=\frac{J}{\sigma(p)}\,p(1-p),
\end{equation}
with $y\in\Delta^{K-1}$ and $p\in(0,1)$ treated as quasi--frozen on the slow time scale.
We write $y^\star\in\Delta^{K-1}$ for an equilibrium of~\eqref{eq:inner-ode-stability} and
work on the tangent space $T_{y^\star}\Delta^{K-1}=\{\delta\in\mathbb{R}^K:\sum_j \delta_j=0\}$.

\begin{lemma}[Jacobian on the simplex tangent]
\label{lem:jacobian-inner-good}
Let $y^\star$ be an equilibrium of~\eqref{eq:inner-ode-stability} and write
$s_2^\star = \|y^\star\|_2^2$.
For perturbations $y = y^\star + \delta$ with $\sum_j \delta_j = 0$, the linearization is
\begin{equation}
\dot{\delta}
= J_y\,\delta,
\qquad
J_y
= \kappa(p)\,\Big(\mathrm{diag}(2y^\star) - 2y^\star (y^\star)^\top - s_2^\star I\Big),
\end{equation}
where $J_y$ acts on $T_{y^\star}\Delta^{K-1}$ (the subspace orthogonal to $\mathbf{1}$).
\end{lemma}

\begin{proposition}[Stability of the uniform equilibrium]
\label{prop:uniform-stability}
Consider the uniform equilibrium
\(
y^\star = \tfrac{1}{K}\mathbf{1}
\)
of~\eqref{eq:inner-ode-stability}.
Then $s_2^\star = \tfrac{1}{K}$ and the Jacobian reduces to
\begin{equation}
J_y
=\frac{\kappa(p)}{K}\Big(I - \tfrac{2}{K}\mathbf{1}\mathbf{1}^\top\Big).
\end{equation}
The eigenstructure is:
\begin{itemize}[leftmargin=1.5em]
\item Along the direction $\mathbf{1}$: a zero eigenvalue (simplex invariance).
\item On the $(K{-}1)$-dimensional tangent space $T_{y^\star}\Delta^{K-1}$: eigenvalues
\(
\lambda=\tfrac{\kappa(p)}{K}.
\)
\end{itemize}
Hence:
\[
\begin{cases}
\kappa(p)>0\ (\text{i.e., }J>0): & y^\star\ \text{is linearly unstable (repelling);}\\[2pt]
\kappa(p)<0\ (\text{i.e., }J<0): & y^\star\ \text{is asymptotically stable.}
\end{cases}
\]
\end{proposition}

\begin{proposition}[Stability of the pure--arm equilibria]
\label{prop:vertex-stability}
Consider a vertex / pure--arm equilibrium $y^\star = e_j$ of~\eqref{eq:inner-ode-stability}, for some $j\in\{1,\dots,K\}$.
Then $s_2^\star=1$ and
\begin{equation}
J_y
= \kappa(p)\,\big(\mathrm{diag}(2e_j) - 2e_j e_j^\top - I\big).
\end{equation}
In coordinates this yields
\[
\dot\delta_j = 0, 
\qquad
\dot\delta_i = -\kappa(p)\,\delta_i\quad(i\ne j),
\]
so that perturbations orthogonal to $e_j$ decay or grow according to the sign of $\kappa(p)$.
In particular:
\[
\begin{cases}
\kappa(p)>0: & y^\star=e_j\ \text{is locally asymptotically stable (winner--take--all);}\\[2pt]
\kappa(p)<0: & y^\star=e_j\ \text{is unstable (flow returns toward mixed states).}
\end{cases}
\]
\end{proposition}
Since
\(
\kappa(p)=\tfrac{J}{\sigma(p)}p(1-p)
\)
and $p(1-p)>0$ for $p\in(0,1)$, the sign of the Youden index $J$ dictates the symmetry breaking based:
\begin{corollary}[Stability summary for within--good equilibria]
\label{cor:within-good-stability-summary}
For the inner dynamics~\eqref{eq:inner-ode-stability}, the stability types of the canonical
equilibria are summarized in Table~\ref{tab:within-good-stability}.
\begin{table}[h]
    \centering
    \small
    \begin{tabular}{lcc}
    \toprule
    Equilibrium type & Condition on $J$ & Stability type \\
    \midrule
    Uniform $y_j=1/K$ & $J>0$ & Unstable (diversity collapse) \\
    Uniform $y_j=1/K$ & $J<0$ & Stable (diversity preserved) \\
    Vertex $y=e_j$ & $J>0$ & Stable (specialization) \\
    Vertex $y=e_j$ & $J<0$ & Unstable (reverts to mixture) \\
    \bottomrule
    \end{tabular}
    \caption{Stability of the uniform and pure--arm equilibria for the within--good ODE~\eqref{eq:inner-ode-stability}.}
    \label{tab:within-good-stability}
\end{table}
\end{corollary}

\begin{remark}[Role of the Youden index $J$]
\label{rem:role-of-J}
We can see the effect of noise as
\begin{itemize}[leftmargin=1.5em]
\item $J>0$ (reward alignment): the uniform mixture is destabilized, pure--arm vertices become attractors, and the good arms polarize; diversity collapses.
\item $J<0$ (reward inversion): the uniform mixture is stabilized while vertices are repelling; diversity is maintained.
\end{itemize}
\end{remark}

\subsection{Coupling back to physical time}
\label{subsec:coupling-time}

Let \(p(t)\in(0,1)\) be the \emph{total} bad mass and define \(\kappa(t)=\frac{J}{\sigma(p(t))}\,p(t)\bigl(1-p(t)\bigr)\).
Then \(d\tau/dt=\kappa(t)\) and (in the multi-bad model) the coupled equations read
\[
\dot y=\kappa(t)\big(y\odot y-s_2\,y\big),\qquad
\dot p
=-\,\frac{J}{\sigma(p)}\,[p(1-p)]^{2}\,\big(s_2+t_2\big),
\qquad
s_2=\|y\|_2^2,\ \ t_2=\|z\|_2^2.
\]
In internal time,
\[
\frac{dp}{d\tau}=-\,p(1-p)\,\bigl(s_2(\tau)+t_2(\tau)\bigr).
\]

Along the generic $J>0$ branch we have $s_2(\tau)\to 1$ (good-block polarization) and
$t_2(\tau)\to 1/M$ (bad-block mixing), hence
\[
\frac{dp}{d\tau}=-(1+\tfrac1M)\,p\,(1+o(1)),
\qquad
p(\tau)=C\,e^{-(1+1/M)\tau}\,(1+o(1))\quad(\tau\to\infty),
\]
for some \(C>0\) determined by the initial condition (e.g.\ via the logit identity).

Next, since
\[
\frac{dt}{d\tau}=\frac{1}{\kappa(t)}
=\frac{\sigma\big(p(\tau)\big)}{J\,p(\tau)\,(1-p(\tau))}
=\frac{\sigma\big(p(\tau)\big)}{J\,p(\tau)}\,(1+o(1)),
\]
the physical-time asymptotics are governed by the local behavior of \(\sigma(p)\) near \(p=0\).

\begin{theorem}[Physical-time rates via the local law of \(\sigma(p)\) (multi-bad setting)]
\label{thm:physical-trichotomy}
Assume \(\sigma(p)\sim \sigma_0\,p^\gamma\) as \(p\downarrow0\) with \(\sigma_0>0\) and \(\gamma\in\mathbb R\).
Let \(m=\arg\max_i y_i(0)\) (unique) and suppose we are on the $J>0$ branch so that $p(t)\downarrow0$.
Write \(a:=1+\frac{1}{M}\) (the asymptotic internal-time slope of $\logit p$).
Then, generically, as \(t\to\infty\) or to a finite absorption time \(t_\infty\),
\[
\begin{array}{ll}
\text{\bf(A) } \gamma<1: &
p(t)\ \asymp\ t^{-1/(1-\gamma)},\qquad
1-y_m(t),\ y_i(t)\ (i\ne m)\ \asymp\ t^{-1/[a(1-\gamma)]};\\[6pt]
\text{\bf(B) } \gamma=1: &
p(t)\ \asymp\ e^{-(aJ/\sigma_0)t},\qquad
1-y_m(t),\ y_i(t)\ (i\ne m)\ \asymp\ e^{-(J/\sigma_0)t};\\[6pt]
\text{\bf(C) } \gamma>1: &
p(t)\ \asymp\ (t_\infty-t)^{1/(\gamma-1)},\qquad
1-y_m(t),\ y_i(t)\ (i\ne m)\ \asymp\ (t_\infty-t)^{1/[a(\gamma-1)]}.
\end{array}
\]
All implicit constants depend only on \((y(0),z(0),p(0),J,\sigma_0,\gamma)\).
The power-law exponents for $p(t)$ are universal (independent of $K$ and $M$), while the $y$-rates
depend on $M$ through $a=1+1/M$ (reducing to the one-bad-arm case when $M=1$).
\end{theorem}

\begin{corollary}[Who wins, and how fast?]
\label{cor:winner-and-rate}
For a general initial condition \(y(0)\in\Delta^{K-1}\) with a unique maximizer \(m\), the winning arm is \(m\).
In internal time, the non-winners decay as \(e^{-\tau}\); in physical time, the rates follow
Theorem~\ref{thm:physical-trichotomy}.
\end{corollary}

\begin{corollary}[Sharp \(\tau\)-envelopes for \(p\)]
\label{cor:tau-envelopes}
Using the logit \(L(\tau)=\log\frac{p(\tau)}{1-p(\tau)}\),
\[
\frac{dL}{d\tau}=-(s_2(\tau)+t_2(\tau)).
\]
On the generic $J>0$ branch, $s_2(\tau)\uparrow1$ and $t_2(\tau)\downarrow 1/M$ (with exponentially decaying gaps),
so
\[
L(\tau)=L(0)-\Bigl(1+\frac{1}{M}\Bigr)\tau+O(1),\qquad p(\tau)=\Theta\!\big(e^{-(1+1/M)\tau}\big).
\]
Thus physical-time rates reduce to integrating \(dt/d\tau\sim\sigma(p(\tau))/(Jp(\tau))\), i.e.\ to the local exponent \(\gamma\).
\end{corollary}



\subsection{Dynamics of $y$- }\label{subsection:Dynamics-of-y}
Let $q:=y(0)\in\Delta^{K-1}$ be the initial composition. Reparametrize time by
\[
\tau(t)=\int_0^t \kappa(s)\,ds,\qquad
\kappa(t)=\frac{J}{\sigma(p(t))}\,p(t)\bigl(1-p(t)\bigr).
\]
In $\tau$-time the inner flow is autonomous:
\[
\frac{dy_j}{d\tau}=y_j\bigl(y_j-s_2\bigr),\qquad s_2=\sum_i y_i^2.
\]
Introduce $u_j:=1/y_j$. Then $u_j' - s_2\,u_j=-1$, whose solution is \footnote{
 Writing it in the standard form
$
u_j'(\tau) + a(\tau)\,u_j(\tau) \;=\; b(\tau),
$
we have
$
a(\tau) = -s_2(\tau), 
b(\tau) = -1.
$.
For a linear ODE \(u_j' + a(\tau)u_j = b(\tau)\), the integrating factor is
$
\mu(\tau)
\;=\;
\exp\!\Bigl(\int_0^\tau a(r)\,dr\Bigr).
$
Using \(a(\tau) = -s_2(\tau)\) and the notation
$
S(\tau) \;:=\; \int_0^\tau s_2(r)\,dr,
$
we obtain
$
\mu(\tau)
\;=\;
\exp\!\Bigl(\int_0^\tau -s_2(r)\,dr\Bigr)
\;=\;
e^{-S(\tau)}.
$
Multiplying the ODE by \(\mu(\tau)\) gives
\begin{equation}
e^{-S(\tau)} u_j'(\tau) - s_2(\tau) e^{-S(\tau)} u_j(\tau)
\;=\;
-\,e^{-S(\tau)}.
\end{equation}
By the product rule and the definition of \(S\),
\begin{equation}
\frac{d}{d\tau}\bigl(e^{-S(\tau)} u_j(\tau)\bigr)
\;=\;
e^{-S(\tau)} u_j'(\tau) - s_2(\tau) e^{-S(\tau)} u_j(\tau),
\end{equation}
so the left-hand side becomes an exact derivative and the equation reduces to
\begin{equation}
\frac{d}{d\tau}\bigl(e^{-S(\tau)} u_j(\tau)\bigr)
\;=\;
-\,e^{-S(\tau)}.
\end{equation}

Integrating from \(0\) to \(\tau\) yields
\begin{equation}
e^{-S(\tau)} u_j(\tau) - e^{-S(0)} u_j(0)
\;=\;
- \int_0^\tau e^{-S(s)}\,ds.
\end{equation}
Since \(S(0)=0\), we have \(e^{-S(0)}=1\). Writing the initial composition as \(y_j(0)=q_j\), we have \(u_j(0)=1/q_j\). Thus

\begin{equation}
e^{-S(\tau)} u_j(\tau)
\;=\;
\frac{1}{q_j} - \int_0^\tau e^{-S(s)}\,ds.
\end{equation}
Defining
$
I(\tau) \;:=\; \int_0^\tau e^{-S(s)}\,ds,
$
and multiplying both sides by \(e^{S(\tau)}\), we obtain the explicit solution
\begin{equation}
u_j(\tau)
\;=\;
e^{S(\tau)}\Bigl(\frac{1}{q_j} - I(\tau)\Bigr),
\end{equation}
}
\[
u_j(\tau)=e^{S(\tau)}\!\Bigl(\frac{1}{q_j}-I(\tau)\Bigr),
\quad
S(\tau):=\int_0^\tau s_2(r)\,dr,
\quad
I(\tau):=\int_0^\tau e^{-S(s)}\,ds=\int_0^\tau e^{-\int_0^\tau \sum_i y_i(r)^2\,dr}\,ds.
\]
notice that based on the  simplex constraint $\sum_{j=1}^K y_j(\tau)=1$, we can eliminate the common factor $e^{-S(\tau)}$: \begin{equation}
1
\;=\;
\sum_{\ell=1}^K y_\ell(\tau)
\;=\;
e^{-S(\tau)}\sum_{\ell=1}^K \frac{1}{\frac{1}{q_\ell}-I(\tau)}.
\end{equation}
Hence
\begin{equation}
e^{-S(\tau)}
\;=\;
\biggl[\sum_{\ell=1}^K \frac{1}{\frac{1}{q_\ell}-I(\tau)}\biggr]^{-1}.
\end{equation}
Substituting this back into the expression for $y_j(\tau)$ gives
\begin{equation}\label{eq:y_j_intetm_I}   
y_j(\tau)=\frac{e^{-S(\tau)}}{\frac{1}{q_j}-I(\tau)}
=\frac{\dfrac{q_j}{1-I(\tau)\,q_j}}{\sum_{\ell=1}^K \dfrac{q_\ell}{1-I(\tau)\,q_\ell}},\qquad
I\in[0,y(0)_{\max}).
\end{equation}


\subsection{Evolution of the collision term $s_2$ }
\label{subsec:s2-exact-bounds}

Throughout, let $q:=y(0)\in\Delta^{K-1}$ be the initial within–good composition and
$I(\tau)$ the scalar from Eq.~\eqref{eq:y_j_intetm_I}. For $r\in\{1,2,3\}$ define the
moment sums
\[
\mathsf M_r(I)\;:=\;\sum_{j=1}^K \frac{q_j^{\,r}}{(1-I q_j)^{r}}
\qquad\big(\mathsf M_1>0\ \text{on }I\in[0,1/q_\ast)\big),
\quad q_\ast:=\max_j q_j.
\]

\begin{lemma}[Exact moment formulas and internal-time map]
\label{lem:s2-exact}
Under the change of variable $I=I(\tau)$ from Eq.~\eqref{eq:y_j_intetm_I},
\begin{align}
y_j(\tau(I))
&= \frac{\displaystyle \frac{q_j}{1-I q_j}}{\displaystyle \mathsf M_1(I)},\\[4pt]
s_2(\tau(I))=\|y\|_2^2
&= \frac{\mathsf M_2(I)}{\mathsf M_1(I)^2},
\qquad
s_3(\tau(I))=\sum_j y_j^3
= \frac{\mathsf M_3(I)}{\mathsf M_1(I)^3},\\[4pt]
\tau(I)
&= \int_0^I \mathsf M_1(z)\,dz
= -\sum_{j=1}^K \log\bigl(1-I q_j\bigr),\\[4pt]
S(\tau):=\int_0^\tau s_2(u)\,du
&= \int_0^{I(\tau)} \frac{\mathsf M_2(z)}{\mathsf M_1(z)}\,dz.
\end{align}
\end{lemma}
\emph{Proof.} \eqref{eq:y_j_intetm_I} gives the first line immediately. The
formulas for $s_2$ and $s_3$ follow by summing $y_j^2$ and $y_j^3$. Since $I'(\tau)
= e^{-S(\tau)}=1/\mathsf M_1(I)$, we get $d\tau/dI=\mathsf M_1(I)$, and integrate to
obtain $\tau(I)$. Finally, $S(\tau)=\int s_2 d\tau=\int (\mathsf M_2/\mathsf M_1)\,dI$.
\qedhere
\begin{proposition}[Two-sided integral bounds; refined logit envelope (multi-bad outer bounds)]
\label{prop:integral-bounds}
For all $I\in[0,1/q_\ast)$,
\begin{equation}
\frac{1}{K}\,\tau(I)
\;\le\;
S\big(\tau(I)\big)
\;\le\;
-\,\log\bigl(1-I q_\ast\bigr),
\qquad
\tau(I)=-\sum_{j=1}^K \log\!\bigl(1-I q_j\bigr).
\end{equation}
Define also the bad-block collision integral
\[
T(\tau):=\int_0^\tau t_2(u)\,du
\qquad\text{with}\qquad
t_2(u)=\|z(u)\|_2^2\in\Big[\frac1M,\,1\Big],
\]
so that $\frac{\tau}{M}\le T(\tau)\le \tau$.
Then for the logit $L(\tau)=\log\!\frac{p(\tau)}{1-p(\tau)}$,
\begin{align}
\label{eq:logit-global}
L(\tau)&=L(0)-S(\tau)-T(\tau),\\
\label{eq:logit-outer}
L(0)-2\tau\ &\le\ L(\tau)\ \le\ L(0)-\Big(\frac1K+\frac1M\Big)\tau,\\
\label{eq:logit-refined-implicit}
L\big(\tau(I)\big)\ &\ge\ L(0)-\tau(I)+\log\bigl(1-I q_\ast\bigr)\quad\text{(implicit lower side)}.
\end{align}
\end{proposition}
\emph{Proof.} The bounds on $S$ are as in the one-bad case:
Cauchy–Schwarz on $\{a_j\}=\{\tfrac{q_j}{1-I q_j}\}$ gives
$\mathsf M_1(I)^2\le K\,\mathsf M_2(I)$; hence $\mathsf M_2/\mathsf M_1\ge \mathsf
M_1/K$. Integrating: $S=\int (\mathsf M_2/\mathsf M_1)\,dI \ge \frac1K\int \mathsf
M_1\,dI= \tau/K$. For the upper bound, monotonicity of $t\mapsto \frac{t}{1-I t}$
implies $\frac{q_j^2}{(1-Iq_j)^2}\le \frac{q_\ast}{1-Iq_\ast}\cdot \frac{q_j}{1-Iq_j}$,
so $\mathsf M_2\le \frac{q_\ast}{1-Iq_\ast}\,\mathsf M_1$. Integrating yields
$S\le \int \frac{q_\ast}{1-Iq_\ast}\,dI = -\log(1-Iq_\ast)$.

For the logit, use Corollary~\ref{cor:logit} to write $L(\tau)=L(0)-\int_0^\tau (s_2+t_2)\,du=L(0)-S(\tau)-T(\tau)$,
and then bound $S(\tau)\in[\tau/K,\tau]$ and $T(\tau)\in[\tau/M,\tau]$ to get \eqref{eq:logit-outer}.
For \eqref{eq:logit-refined-implicit}, combine $T(\tau)\le\tau$ and $S(\tau(I))\le-\log(1-Iq_\ast)$. \qedhere
\begin{remark}[What is “collision”?]
Here $s_2=\sum_i y_i^2$ is the usual collision probability on the simplex. The
\emph{collision gap}
\(
s_3-s_2^2=\sum_i y_i(y_i-s_2)^2
\)
drives the slope: in internal time,
\(
\frac{ds_2}{d\tau}=2\bigl(s_3-s_2^2\bigr)\ge 0,
\)
with strict increase off the uniform-on-support sets.
\end{remark}

\begin{proposition}[Pointwise bounds for $s_2'$ (tighter than logistic)]
\label{prop:s2-derivative-bounds}
Let $u(\tau):=\sqrt{s_2(\tau)}\in[1/\sqrt K,1]$ and $y_{\max}(\tau):=\max_i y_i(\tau)$.
Then for all $\tau$,
\begin{align}
\label{eq:s2-identity}
\textbf{(Exact)}\quad
\frac{ds_2}{d\tau}
&=2\sum_{i=1}^K y_i\,(y_i-s_2)^2
=2\bigl(s_3-s_2^2\bigr);\\[3pt]
\label{eq:Lp-sharper}
\textbf{($\ell_p$ upper bound)}\quad
\frac{ds_2}{d\tau}
&\le 2\bigl(s_2^{3/2}-s_2^2\bigr)
=2u^3(1-u)
\quad\ (\text{since }s_3\le s_2^{3/2});\\[3pt]
\label{eq:support-aware}
\textbf{(support-aware)}\quad
\frac{ds_2}{d\tau}
&\le 2\,y_{\max}(\tau)\,s_2(\tau)\,\bigl(K\,s_2(\tau)-1\bigr),
\qquad
\bigl[\text{using }\sum_i (y_i-s_2)^2 = s_2(K s_2-1)\bigr].
\end{align}
Moreover, writing $y=\unif_K+v$ with $\sum v_i=0$ and $\eta:=\|v\|_2^2=s_2-\frac{1}{K}$,
\begin{equation}
\label{eq:eta-bracket}
\frac{ds_2}{d\tau}
= 2\Big(\frac{\eta}{K}-\eta^2+\sum_i v_i^3\Big),
\qquad
\Rightarrow\qquad
2\Big(\frac{\eta}{K}-\eta^2-\eta^{3/2}\Big)\ \le\ \frac{ds_2}{d\tau}\ \le\ 2\Big(\frac{\eta}{K}-\eta^2+\eta^{3/2}\Big).
\end{equation}
\end{proposition}
\emph{Proof.} \eqref{eq:s2-identity} is the variance identity. For
\eqref{eq:Lp-sharper}, use $\|y\|_3\le \|y\|_2$ to get $s_3\le s_2^{3/2}$. For
\eqref{eq:support-aware}, bound $y_i\le y_{\max}$ in \eqref{eq:s2-identity}.
For \eqref{eq:eta-bracket}, expand $s_3$ around $\unif$:
$s_3=1/K^2 + 3\eta/K + \sum v_i^3$; subtract $s_2^2=(1/K+\eta)^2$.
\qedhere
\begin{corollary}[Implicit envelope for $u(\tau)=\sqrt{s_2(\tau)}$]
\label{cor:u-envelope}
Integrating the differential inequality $u'(\tau)\le u(\tau)^2\bigl(1-u(\tau)\bigr)$
from~\eqref{eq:Lp-sharper} yields the implicit bound
\[
\log\!\frac{u(\tau)}{1-u(\tau)}\ -\ \frac{1}{u(\tau)}
\ \le\
\log\!\frac{u_0}{1-u_0}\ -\ \frac{1}{u_0}\ +\ \tau,
\qquad u_0=\sqrt{s_2(0)}=\|q\|_2.
\]
This dominates the logistic envelope $s_2(\tau)\le \bigl(1+[(1-s_2(0))/s_2(0)]e^{-2\tau}\bigr)^{-1}$
whenever $u_0$ is close to $1/\sqrt K$ (near-uniform start).
\end{corollary}

\begin{corollary}[Refined envelopes for $p(\tau)$ (multi-bad outer bounds)]
\label{cor:p-envelopes-refined}
Using $L(\tau)=L(0)-\int_0^\tau (s_2+t_2)\,du$ and Proposition~\ref{prop:integral-bounds}:
\[
\frac{1}{1+\frac{1-p_0}{p_0}\,e^{2\tau}}
\ \le\
p(\tau)
\ \le\
\frac{1}{1+\frac{1-p_0}{p_0}\,e^{(\frac1K+\frac1M)\tau}},
\]
and, with the implicit time $I\mapsto\tau(I)$,
\[
p\!\big(\tau(I)\big)
\ \ge\
\frac{1}{1+\frac{1-p_0}{p_0}\,\exp\!\Big(\tau(I)\;-\;\log(1-I q_\ast)\Big)}.
\]
The last (implicit) lower envelope becomes tight as the good mass polarizes to the maximizer of $q$.
\end{corollary}

\newpage

\section{Inner Dynamics of the Bad Arms}
\label{sec:inner-bad}

This section is the companion to Section~\ref{sec:inner-good}.
When we keep the $M$ bad arms explicitly (instead of aggregating them into a single
virtual bad arm), the \emph{within-bad} composition obeys the \emph{same} collision
ODE as the within-good composition, but with an overall \emph{sign flip}.
As a result, essentially every statement in Section~\ref{sec:inner-good} has a
direct bad-block analogue obtained by the substitutions
\[
(y,K)\ \mapsto\ (z,M),
\qquad
\kappa\ \mapsto\ -\kappa
\quad\text{(equivalently, }\tau\mapsto -\tau\text{ in internal time).}
\]
We record the corresponding results (without reproving them).

\subsection{Within-bad composition and pushforward}

Let the policy over $K$ good arms and $M$ bad arms be
\[
\p=(p_1,\ldots,p_K,\ p_{b_1},\ldots,p_{b_M})\in\Delta^{K+M-1},
\qquad
p:=\sum_{m=1}^M p_{b_m}\in[0,1].
\]
Define the within-bad normalized composition
\[
z_m \;:=\;\frac{p_{b_m}}{p},\qquad m\in[M],\qquad z\in\Delta^{M-1}.
\]
Writing $p_i=\exp(\theta_i)/Z$ with $Z=\sum_{j\le K}e^{\theta_j}+\sum_{m\le M}e^{\theta_{b_m}}$,
the bad normalization cancels the good logits:
\begin{equation}
\label{eq:z-softmax-bad}
z_m \;=\; \frac{\exp(\theta_{b_m})}{\sum_{\ell=1}^M \exp(\theta_{b_\ell})}
\;=\;\big(\mathrm{softmax}(\boldsymbol\theta_{\mathrm{bad}})\big)_m,
\end{equation}
so $z$ depends only on $\boldsymbol\theta_{\mathrm{bad}}$.

\begin{lemma}[Pushforward from logits to within-bad composition]
\label{lem:pushforward-z}
For any small increment $\Delta\boldsymbol\theta=(\Delta\boldsymbol\theta_{\mathrm{good}},\Delta\boldsymbol\theta_{\mathrm{bad}})$,
\begin{equation}
\label{eq:push-z}
\Delta z
\;=\;
\big(\mathrm{Diag}(z)-zz^\top\big)\,\Delta\boldsymbol\theta_{\mathrm{bad}}
\;=\;
z\odot\Big(\Delta\boldsymbol\theta_{\mathrm{bad}}-\langle z,\Delta\boldsymbol\theta_{\mathrm{bad}}\rangle\,\mathbf 1\Big),
\end{equation}
and in particular $\partial z/\partial \boldsymbol\theta_{\mathrm{good}}=\mathbf 0$.
\end{lemma}
\emph{Proof.} Identical to Lemma~\ref{lem:pushforward-y} with $(y,\boldsymbol\theta_{\mathrm{good}},K)$ replaced by $(z,\boldsymbol\theta_{\mathrm{bad}},M)$.
\qedhere
\subsection{Bad-block drift: the same collision field with opposite sign}

Assume the same block symmetry as in Section~\ref{sec:inner-good} / Appendix~\ref{appendix:Group-Normalized RL with Noisy Reward}:
\[
A_j=a_{\mathrm g}(p)\ (j\le K),\qquad
A_{b_m}=a_{\mathrm b}(p)\ (m\le M),\qquad
\Delta r(p):=a_{\mathrm b}(p)-a_{\mathrm g}(p).
\]
Recall the (good-block) scalar from \eqref{eq:good-logit-drift}:
\[
\kappa(p)\;:=\;-\,\eta\,p(1-p)\,\Delta r(p).
\]
Then the expected logit drift in the bad block is the sign-reversed analogue of
\eqref{eq:good-logit-drift}:
\begin{equation}
\label{eq:bad-logit-drift}
\mathbb E[\Delta\boldsymbol\theta_{\mathrm{bad}}]
\;=\;
-\,\kappa(p)\,z,
\qquad
\text{i.e.}\quad
\mathbb E[\Delta\theta_{b_m}] = -\,\kappa(p)\,z_m.
\end{equation}

\begin{proposition}[Within-bad mean drift in $z$-coordinates]
\label{prop:z-drift-euclid}
Applying Lemma~\ref{lem:pushforward-z} to \eqref{eq:bad-logit-drift} yields
\begin{equation}
\label{eq:z-drift-euclid}
\mathbb E[\Delta z]
\;=\;
-\,\kappa(p)\,\Big(z\odot z-\|z\|_2^2\,z\Big),
\qquad
\mathbb E[\Delta z_m]
\;=\;
-\,\kappa(p)\,z_m\big(z_m-\|z\|_2^2\big).
\end{equation}
In the noisy GRPO specialization $a_{\mathrm g}(p)=\frac{Jp}{\sigma(p)}$,
$a_{\mathrm b}(p)=-\frac{J(1-p)}{\sigma(p)}$ (so $\Delta r(p)=-J/\sigma(p)$ and
$\kappa(p)=\eta\frac{J}{\sigma(p)}p(1-p)$), this becomes
\[
\mathbb E[\Delta z_m]
=
-\,\eta\,\frac{J}{\sigma(p)}\,p(1-p)\;z_m\Big(z_m-\|z\|_2^2\Big).
\]
\end{proposition}

\emph{Consequences (bad-block smoothing vs.\ polarization).}
For $J>0$ (hence $\kappa(p)>0$), the sign in \eqref{eq:z-drift-euclid} implies a
\emph{smoothing} effect: components with $z_m>\|z\|_2^2$ shrink while those with
$z_m<\|z\|_2^2$ grow, pushing $z$ toward uniformity on its support.
For $J<0$ the direction reverses and the bad block \emph{polarizes} (winner-take-all among bad modes),
mirroring the good-block behavior when $J>0$.

\subsection{Internal-time form and direct correspondence with Section~\ref{sec:inner-good}}

As in Section~\ref{sec:inner-good}, define the internal time
\[
\tau(t):=\int_0^t \kappa\big(p(s)\big)\,ds,
\]
so that (in mean-field ODE form) the bad composition satisfies
\begin{equation}
\label{eq:GD-bad}
\frac{dz}{d\tau}
\;=\;
-\Big(z\odot z-\|z\|_2^2\,z\Big).
\end{equation}
Equivalently, with $\rho:=-\tau$ one has
\[
\frac{dz}{d\rho}=z\odot z-\|z\|_2^2\,z,
\]
which is \emph{exactly} the same autonomous ODE as \eqref{eq:GD} for $y$, with
$K$ replaced by $M$.

\begin{lemma}[Geometry / Lyapunov structure for the bad block (sign-reversed)]
\label{lem:z-geometry}
Let $z(\tau)\in\Delta^{M-1}$ solve \eqref{eq:GD-bad} and define
\[
t_2(\tau):=\|z(\tau)\|_2^2=\sum_{m=1}^M z_m(\tau)^2,
\qquad
t_3(\tau):=\sum_{m=1}^M z_m(\tau)^3.
\]
Then the statements of the ``Geometry and Lyapunov structure'' lemma in
Section~\ref{sec:inner-good} carry over with $(y,s_2,s_3,K)$ replaced by $(z,t_2,t_3,M)$
and with all monotonicities reversed. Concretely:
\begin{enumerate}
\item[\normalfont(1)] \textbf{Simplex invariance.} $\sum_m z_m(\tau)=1$ and $z_m(\tau)\ge0$ are preserved.
\item[\normalfont(2)] \textbf{Gradient form with opposite sign.}
With the same potential
\[
\mathcal L(z):=\frac{1}{3}\sum_{m=1}^M z_m^3-\frac{1}{4}\Big(\sum_{m=1}^M z_m^2\Big)^2,
\qquad
\nabla\mathcal L(z)=\big(z_m^2-t_2\,z_m\big)_m,
\]
we have
\[
\frac{dz}{d\tau}=-\,\nabla\mathcal L(z),
\qquad
\frac{d}{d\tau}\mathcal L\big(z(\tau)\big)=-\|\nabla\mathcal L(z)\|_2^2\le 0.
\]
\item[\normalfont(3)] \textbf{Monotone \emph{de}-concentration.}
\[
\frac{d}{d\tau}t_2(\tau)
\;=\;
-2\bigl(t_3(\tau)-t_2(\tau)^2\bigr)\;\le\;0,
\]
with equality iff $z$ is uniform on its support.
\item[\normalfont(4)] \textbf{Equilibria.}
Stationary points are exactly the uniform points on a support of size $m$:
for any $m\in\{1,\dots,M\}$, any point with exactly $m$ nonzero entries, each equal to $1/m$,
is an equilibrium.
\end{enumerate}
\end{lemma}

\subsection{Stability and global limits (bad-block counterpart of Section~\ref{sec:stability-within-good})}

Write the within-bad mean-field ODE in physical time as
\begin{equation}
\label{eq:inner-ode-bad}
\dot z
\;=\;
-\kappa(p(t))\,\big(z\odot z - t_2\,z\big),
\qquad
t_2=\|z\|_2^2,
\qquad
\kappa(p)=\frac{J}{\sigma(p)}\,p(1-p)
\ \ \text{(noisy GRPO)}.
\end{equation}

\begin{proposition}[Stability of the uniform and vertex equilibria for the bad block]
\label{prop:within-bad-stability}
The stability conclusions of Propositions~\ref{prop:uniform-stability} and~\ref{prop:vertex-stability}
carry over with $K\mapsto M$ and $\kappa\mapsto -\kappa$:
\begin{itemize}[leftmargin=1.5em]
\item \textbf{Uniform equilibrium.} For $z^\star=\tfrac{1}{M}\mathbf 1$, the nontrivial eigenvalues on the simplex tangent space
are $\lambda=-\kappa(p)/M$. Hence
\[
\begin{cases}
J>0\ (\kappa(p)>0): & z^\star\ \text{is asymptotically stable (bad mass spreads);}\\[2pt]
J<0\ (\kappa(p)<0): & z^\star\ \text{is unstable.}
\end{cases}
\]
\item \textbf{Vertex equilibria.} For a vertex $z^\star=e_m$, the transverse modes have eigenvalues $+\kappa(p)$, hence
\[
\begin{cases}
J>0\ (\kappa(p)>0): & z^\star=e_m\ \text{is unstable;}\\[2pt]
J<0\ (\kappa(p)<0): & z^\star=e_m\ \text{is locally asymptotically stable (bad-mode collapse).}
\end{cases}
\]
\end{itemize}
\end{proposition}

\begin{corollary}[Stability summary for within--bad equilibria]
\label{cor:within-bad-stability-summary}
The within-bad stability types are the sign-reversed analogue of
Corollary~\ref{cor:within-good-stability-summary}:
\begin{table}[h]
    \centering
    \small
    \begin{tabular}{lcc}
    \toprule
    Equilibrium type & Condition on $J$ & Stability type \\
    \midrule
    Uniform $z_m=1/M$ & $J>0$ & Stable (bad mass diffuses) \\
    Uniform $z_m=1/M$ & $J<0$ & Unstable \\
    Vertex $z=e_m$ & $J>0$ & Unstable \\
    Vertex $z=e_m$ & $J<0$ & Stable (bad-mode collapse) \\
    \bottomrule
    \end{tabular}
    \caption{Stability of canonical equilibria for the within--bad ODE \eqref{eq:inner-ode-bad}.}
    \label{tab:within-bad-stability}
\end{table}
\end{corollary}

\begin{theorem}[Global limit of $z$ (bad-block counterpart of Theorem~\ref{thm:global})]
\label{thm:global-bad}
Consider \eqref{eq:inner-ode-bad} with an interior initialization $z(0)$ (all coordinates positive).
Then:
\begin{itemize}[leftmargin=1.5em]
\item If $J>0$, the flow is the \emph{reverse} of the good-block collision flow in internal time.
Consequently $z(t)$ converges to the uniform point on the full bad simplex:
\[
z(t)\ \longrightarrow\ \frac{1}{M}\mathbf 1,
\]
and the collision probability $t_2(t)=\|z(t)\|_2^2$ decreases monotonically to $1/M$.
\item If $J<0$, the direction reverses and $z(t)$ follows the \emph{forward} collision flow
(on $\Delta^{M-1}$). For generic initial conditions (no ties), $z(t)$ converges to the
vertex selected by the unique maximizer $m=\arg\max_m z_m(0)$ (winner-take-all among bad modes).
\item If some coordinates of $z(0)$ are exactly zero, the support is invariant and the same statements hold
with $M$ replaced by the support size (uniform-on-support for $J>0$, vertex-on-support for $J<0$).
\end{itemize}
\end{theorem}
\emph{Proof.} Immediate from Theorem~\ref{thm:global} by the correspondence
\eqref{eq:GD-bad} (time reversal) and $K\mapsto M$.
\qedhere
\begin{proposition}[Exponential approach in internal time]
\label{prop:bad-exp-rates}
The rate statements in Proposition~\ref{prop:exp-tau} transfer to $z$ with the same substitutions:
\begin{itemize}[leftmargin=1.5em]
\item For $J>0$ (uniform stable), linearization at $z^\star=\frac{1}{M}\mathbf 1$ yields
\[
\|z(\tau)-\tfrac{1}{M}\mathbf 1\|_2 \;=\; \Theta\!\big(e^{-\tau/M}\big)
\quad(\tau\to+\infty),
\]
and hence $t_2(\tau)-\frac{1}{M}=\Theta(e^{-2\tau/M})$.
\item For $J<0$ (vertices stable), after reversing internal time as in the sign discussion of
Section~\ref{sec:inner-good}, the non-winners decay as $e^{-|\tau|}$ toward the winning vertex,
exactly as in Proposition~\ref{prop:exp-tau} with $K\mapsto M$.
\end{itemize}
\end{proposition}

\begin{remark}[ scalar closed-form representation (sign-flipped analogue of \eqref{eq:y_j_intetm_I})]
\label{rem:z-scalar-rep}
The explicit ``one-scalar'' representation for $y(\tau)$ in Section~\ref{sec:inner-good}
also carries over to the bad block with the sign flipped in the denominators.
If $q:=z(0)\in\Delta^{M-1}$ and $I(\tau)$ is a strictly increasing scalar with $I(0)=0$, then
\[
z_m(\tau)
=
\frac{\dfrac{q_m}{1+I(\tau)\,q_m}}{\sum_{\ell=1}^M \dfrac{q_\ell}{1+I(\tau)\,q_\ell}},
\qquad I(\tau)\uparrow\infty \ \Rightarrow\ z(\tau)\to \tfrac{1}{M}\mathbf 1.
\]
This is the direct sign-reversal analogue of \eqref{eq:y_j_intetm_I}; we omit the derivation.
\end{remark}

\begin{remark}[Effect on the total bad-mass drift in the multi-bad model]
\label{rem:z-impact-on-p}
In the multi-bad setting the total bad mass $p(t)$ couples to \emph{both} collision terms:
in internal time $\tau$,
\[
\frac{dp}{d\tau}
=-\,p(1-p)\big(\|y(\tau)\|_2^2+\|z(\tau)\|_2^2\big),
\qquad
\frac{d}{d\tau}\log\!\frac{p}{1-p}
=-(\|y\|_2^2+\|z\|_2^2).
\]
Thus $\|z\|_2^2\in[1/M,1]$ enters only as a bounded multiplicative factor in the decay of $\logit p$,
recovering the aggregated-bad model when $M=1$ (where $\|z\|_2^2\equiv 1$).
\end{remark}


\newpage
\section{Shahshahani geometry and the within–good flow}
\label{subsec:shah-y-flow}

Let \(y=(y_1,\dots,y_K)\in\Delta^{K-1}\) denote the within–good composition and define
\begin{equation}
\label{eq:y-flow}
\dot y
\;=\;
\kappa(p)\,\Big(y\odot y-\|y\|_2^2\,y\Big),
\qquad
\kappa(p)\;:=\;\frac{J}{\sigma(p)}\,p(1-p),
\end{equation}
where \(J\) is the judge–separation, \(\sigma(p)>0\) is the group–normalization scale, and \(p\in[0,1]\) is the bad–mass.
Set \(s_2:=\|y\|_2^2=\sum_i y_i^2\) and \(s_3:=\sum_i y_i^3\).

\paragraph{Shahshahani (Fisher) metric on the simplex.}
On the interior of the simplex \(\Delta^{K-1}\), the Shahshahani inner product is
\[
\langle u,v\rangle_y
\;=\;\sum_{i=1}^K \frac{u_i v_i}{y_i}
\quad \text{on the tangent space } T_y\Delta^{K-1}=\Big\{v\in\mathbb{R}^K:\;\sum_i v_i=0\Big\}.
\]
For a smooth potential \(\phi:\Delta^{K-1}\to\mathbb{R}\), the associated \emph{natural gradient} (steepest ascent in this metric) has the replicator form
\begin{equation}
\label{eq:replicator-gradient-identity}
\operatorname{grad}_{\mathrm{Shah}}\phi(y)
\;=\;
y\odot\Big(\nabla \phi(y)-\langle \nabla \phi(y),y\rangle\,\mathbf{1}\Big).
\end{equation}
Indeed, with \(w:=y\odot(\nabla\phi-c\,\mathbf{1})\) and \(c=\langle\nabla\phi,y\rangle\), we have
\(\sum_i w_i=0\) and, for any \(v\in T_y\Delta^{K-1}\),
\(\langle w,v\rangle_y=\sum_i(\nabla\phi_i-c)v_i=\langle\nabla\phi,v\rangle\), which characterizes the Riemannian gradient.

\begin{proposition}[Shahshahani gradient representation]
\label{prop:shah-gradient-y}
Let \(\Phi(y):=\tfrac12\|y\|_2^2=\tfrac12\sum_i y_i^2\).
Then \eqref{eq:y-flow} is the Shahshahani gradient flow of \(\Phi\), scaled by \(\kappa(p)\):
\[
\dot y \;=\; \kappa(p)\,\operatorname{grad}_{\mathrm{Shah}}\Phi(y).
\]
Equivalently, in coordinates,
\(
\dot y_i \;=\; \kappa(p)\,y_i\big(y_i-\|y\|_2^2\big).
\)
\end{proposition}

\begin{proof}
We have \(\nabla\Phi(y)=y\) and \(\langle \nabla\Phi(y),y\rangle=\|y\|_2^2\).
Applying \eqref{eq:replicator-gradient-identity} gives
\(\operatorname{grad}_{\mathrm{Shah}}\Phi(y)=y\odot(y-\|y\|_2^2\mathbf{1})\),
hence the claim.
\end{proof}

\paragraph{Interpretation (Herfindahl ascent in Fisher units).}
The potential \(\Phi(y)=\tfrac12\|y\|_2^2\) is the (half) Herfindahl–Hirschman concentration index.
Thus \eqref{eq:y-flow} is the \emph{steepest way, in Fisher/Shahshahani geometry}, to increase concentration when \(\kappa(p)>0\)
(and to decrease it when \(\kappa(p)<0\)).
The scalar \(\kappa(p)=\tfrac{J}{\sigma(p)}p(1-p)\) gates the time–scale: the inner reshuffling freezes at \(p\in\{0,1\}\) and is fastest near \(p=\tfrac12\);
its sign flips with \(J\).

\begin{corollary}[Lyapunov monotonicity and a variance identity]
\label{cor:lyap}
Along any trajectory of \eqref{eq:y-flow},
\begin{equation}
\label{eq:phi-derivative}
\frac{d}{dt}\Phi\big(y(t)\big)
\;=\;
\kappa\big(p(t)\big)\,\big(s_3-s_2^2\big)
\;=\;
\kappa\big(p(t)\big)\,\Var_{i\sim y}\!\big(y_i\big)\;\ge 0
\quad\text{whenever }\kappa\ge 0,
\end{equation}
with equality iff \(y\) is uniform on its support (\(y_i\in\{0,1/m\}\) on some subset of size \(m\)).
Equivalently,
\(\frac{d}{dt}\|y\|_2^2=2\,\kappa(p)\,\Var_{i\sim y}(y_i)\).
\end{corollary}

\begin{proof}
By the chain rule,
\(\tfrac{d}{dt}\Phi=\langle \nabla\Phi,\dot y\rangle=\kappa\sum_i y_i\big(y_i-\|y\|_2^2\big)y_i
=\kappa\,(s_3-s_2^2)\).
Since \(s_3-s_2^2=\sum_i y_i(y_i-s_2)^2\ge 0\) and vanishes exactly at support–uniform points, the claim follows.
\end{proof}

\begin{proposition}[Equilibria, support invariance, and stability]
\label{prop:eq-stability}
The rest points of \eqref{eq:y-flow} are exactly the barycenters of faces:
for any subset \(S\subseteq[K]\) of size \(m\),
\(y_i^\star=\tfrac1m\) for \(i\in S\) and \(y_i^\star=0\) otherwise.
Moreover:
\begin{enumerate}
\item[(i)] \emph{Support invariance.} If \(y_i(0)=0\), then \(y_i(t)\equiv 0\) for all \(t\) (since \(\dot y_i=\kappa\,y_i(\cdot)\)).
\item[(ii)] \emph{Stability for \(\kappa>0\).} The unique asymptotically stable equilibria are the vertices
(\(m=1\)); all higher–dimensional barycenters (\(m\ge 2\)) are saddles/unstable.
\item[(iii)] \emph{Stability for \(\kappa<0\).} The roles reverse: the full–uniform point
(\(m=K\)) is the unique asymptotically stable equilibrium; all others are unstable.
\end{enumerate}
\end{proposition}

\begin{proof}[Proof sketch]
At equilibrium, \(0=\dot y_i=\kappa\,y_i(y_i-s_2)\) implies:
either \(y_i=0\) or \(y_i=s_2\). If \(m\) coordinates are positive, then
\(1=\sum_i y_i= m\,s_2\Rightarrow y_i=s_2=\tfrac1m\) on the support.
For stability, note that \eqref{eq:y-flow} is a Shahshahani gradient system for the convex \(\Phi\),
scaled by \(\kappa\). When \(\kappa>0\) the flow ascends \(\Phi\) and converges to its maximizers on \(\Delta^{K-1}\),
which are precisely the vertices; when \(\kappa<0\) it descends to the unique minimizer, the full–uniform point.
Support invariance follows from the factor \(y_i\) in each coordinate.
\end{proof}

\section{Hyperparameters and Training Details}
\label{app:hparams}
Table~\ref{tab:train-config} summarizes the training and evaluation configuration used in all experiments.
We fine-tune a Qwen2.5-3B base model for $1410$ optimization steps (2 epochs) with a global batch size of $16$.
For rollout generation, we sample $8$ responses per prompt with temperature $1.0$ (top-$p=1.0$, top-$k=-1$) and truncate prompts/responses at a maximum length of $4000$ tokens each.
For actor optimization, we use GRPO with symmetric PPO clipping $(\varepsilon_{\mathrm{low}},\varepsilon_{\mathrm{high}})=(0.2,0.2)$ and Adam (learning rate $10^{-6}$, weight decay $0.1$, $(\beta_1,\beta_2)=(0.9,0.999)$), with gradient-norm clipping $1.0$ and a constant learning-rate schedule with $10$ warmup steps.
Evaluation is performed greedily (temperature $0.0$) with the same decoding truncation limits. In our experimental setup, we skipped the KL-regularization term by setting its corresponding coefficient to zero.


\begin{table}[t]
\centering
\caption{Training configuration.}
\label{tab:train-config}
\begin{tabular}{@{}l r@{}}
\toprule
\multicolumn{2}{l}{\textbf{Data Configuration}}\\
\midrule
Base Model & Qwen2.5-3B \\
Global Batch Size & 16 \\
Train Steps & 1410 \\
Total Epochs & 2 \\[3pt]

\multicolumn{2}{l}{\textbf{Rollout Inference}}\\
\midrule
Rollout Num per Prompt & 8 \\
Temperature & 1.0 \\
Top-p & 1.0 \\
Top-k & $-1$ \\
Max Prompt Length & 4000 \\
Max Response Length & 4000 \\[3pt]

\multicolumn{2}{l}{\textbf{Actor Training}}\\
\midrule
PPO Mini Batch Size & 32 \\
Advantage Estimation Type & GRPO \\
Clipping $\varepsilon_{\text{low}}$ & 0.2 \\
Clipping $\varepsilon_{\text{high}}$ & 0.2 \\
Optimizer & Adam \\
Learning Rate & $10^{-6}$ \\
Weight Decay & 0.1 \\
$(\beta_1,\beta_2)$ & $(0.9,\,0.999)$ \\
Gradient Norm Clipping & 1.0 \\
Learning Rate Scheduler & constant \\
Warmup Steps & 10 \\[3pt]
KL coefficient ($\beta$) & 0.0 \\[3pt]

\multicolumn{2}{l}{\textbf{Evaluation Setup}}\\
\midrule
Temperature & 0.0 \\
Top-p & 1.0 \\
Top-k & $-1$ \\
Max Generation Length & 4000 \\
\bottomrule
\end{tabular}
\end{table}


\section{Noise Injection Pseudocode}
\label{app:noise}

\begin{algorithm}[H]
\caption{Noisy Verifier Wrapper}
\label{alg:noisy-verifier}
\begin{algorithmic}[1]
\Require Oracle checker $\mathrm{Oracle}(\cdot)\in\{0,1\}$, target $(\TPR,\FPR)$
\Function{NoisyCheck}{$\text{program}$}
  \State $z \gets \mathrm{Oracle}(\text{program})$ \Comment{ground truth}
  \If{$z=1$}
     \State $r \gets \mathrm{Bernoulli}(\TPR)$
  \Else
     \State $r \gets \mathrm{Bernoulli}(\FPR)$
  \EndIf
  \State \Return $r$
\EndFunction
\end{algorithmic}
\end{algorithm}

\section{Data Sample}


\begin{tcolorbox}[
  colback=gray!1,
  colframe=gray!60!black,
  title={Example Coding Problem: \texttt{kMarsh}},
  fonttitle=\bfseries\ttfamily,
  breakable,
  enhanced,
  sharp corners=south,
  boxrule=0.8pt,
  left=6pt, right=6pt, top=6pt, bottom=6pt
]

\textbf{Problem Statement.}  
Solve the following coding problem using the programming language \texttt{Python}:

\medskip
Mr.\ K has a rectangular plot of land which may contain marshes where fenceposts cannot be set.  
He wants you to find the \emph{perimeter of the largest rectangular fence} that can be built on this land.

\medskip
For example, in the following $m\times n = 4\times 4$ grid, 
$\boldsymbol{x}$ marks a marsh and $\boldsymbol{.}$ marks good land:
\begin{verbatim}
....
..x.
..x.
x...
\end{verbatim}

If we number the rows and columns starting with $\mathrm{1}$, 
there are two main areas that can be fenced: $(1,1)-(3,2)$ and $(1,2)-(4,4)$.  
The longest perimeter is $10$.

\medskip
\textbf{Function Description.}  
Complete the function \texttt{kMarsh} in the editor below.  
It should print either an integer or the word \texttt{impossible}.

\begin{description}[style=unboxed, leftmargin=1.2cm]
\item[\texttt{kMarsh(grid)}:]  
\begin{itemize}[leftmargin=*, itemsep=1pt]
  \item \textbf{Input:} an array of strings that represent the grid.  
  \item \textbf{Output:} an integer representing the largest perimeter, or the string \texttt{impossible}.
\end{itemize}
\end{description}

\textbf{Input Format.}  
\begin{itemize}[leftmargin=*, itemsep=1pt]
  \item The first line contains two space-separated integers $m$ and $n$, the grid rows and columns.  
  \item Each of the next $m$ lines contains $n$ characters describing the land:  
  `'x'` (ASCII 120) if it is a marsh, and `'.'` (ASCII 46) otherwise.
\end{itemize}

\textbf{Constraints.}  
\[
2 \le m,n \le 500
\]

\textbf{Output Format.}  
Print a single integer—the largest perimeter—or \texttt{impossible} if no rectangular fence can be built.

\medskip
\textbf{Sample Input 0}
\begin{verbatim}
4 5
.....
.x.x.
.....
.....
\end{verbatim}

\textbf{Sample Output 0}
\begin{verbatim}
14
\end{verbatim}

\textbf{Explanation 0.}  
The fence can be built around the entire field.  
\[
\text{Perimeter} = 2(4-1) + 2(5-1) = 14.
\]

\medskip
\textbf{Sample Input 1}
\begin{verbatim}
2 2
.x
x.
\end{verbatim}

\textbf{Sample Output 1}
\begin{verbatim}
impossible
\end{verbatim}

\textbf{Explanation 1.}  
We need a minimum of four corner points to form a fence, hence it is impossible.

\medskip
\textbf{Sample Input 2}
\begin{verbatim}
2 5
.....
xxxx.
\end{verbatim}

\textbf{Sample Output 2}
\begin{verbatim}
impossible
\end{verbatim}

\textbf{Explanation 2.}  
The lower row prevents forming a valid rectangle.

\medskip
The input is provided via \texttt{stdin}, and the solution should print its result to \texttt{stdout}.

\medskip
\textbf{Task.}  
Now, solve the problem and return the code.

\end{tcolorbox}

\begin{tcolorbox}[
  colback=gray!1,
  colframe=gray!60!black,
  title={Test Cases for \texttt{kMarsh}},
  fonttitle=\bfseries\ttfamily,
  breakable,
  enhanced,
  sharp corners=south,
  boxrule=0.8pt,
  left=6pt, right=6pt, top=6pt, bottom=6pt
]

\textbf{Test Case List (JSON-like format):}
\begin{verbatim}
[
  {
    'fn_name': None,
    'input': '4 5\n.....\n.x.x.\n.....\n.....\n',
    'output': '14\n',
    'type': 'stdin_stdout'
  },
  {
    'fn_name': None,
    'input': '2 2\n.x\nx.\n',
    'output': 'impossible\n',
    'type': 'stdin_stdout'
  },
  {
    'fn_name': None,
    'input': '2 5\n.....\nxxxx.\n',
    'output': 'impossible\n',
    'type': 'stdin_stdout'
  }
]
\end{verbatim}

\end{tcolorbox}

\end{document}